\documentclass{article}

\usepackage[final]{neurips_2021}

\usepackage[utf8]{inputenc} %
\usepackage[T1]{fontenc}    %
\usepackage{hyperref}       %
\usepackage{url}            %
\usepackage{booktabs}       %
\usepackage{amsfonts}       %
\usepackage{nicefrac}       %
\usepackage{microtype}      %
\usepackage[usenames,dvipsnames]{xcolor}         %

\usepackage{amssymb}
\usepackage{amsmath}

\usepackage{empheq}
\usepackage[most]{tcolorbox}
\newtcbox{\mybox}[1][]{%
    nobeforeafter, math upper, tcbox raise base,
    enhanced, colframe=black!100,
    colback=blue!0, boxrule=1.5pt,
    #1}

\usepackage{algorithm}
\usepackage{algpseudocode}

\usepackage[nameinlink]{cleveref}  %

\setcitestyle{
    braces={round}
}
\definecolor{mydarkblue}{rgb}{0,0.08,0.45}
\hypersetup{
  colorlinks   = true, %
  allcolors    = mydarkblue,
}

\usepackage[alpha]{mdpn}
\renewcommand{\Rfun}{r}

\renewcommand{\MDP}{(\sset, \aset, \Tfun, \Rfun, \initS, \D)}

\newcommand\GP{\mathcal{GP}}

\newcommand\rbr[1]{\left( #1 \right)}
\newcommand\sbr[1]{\left[ #1 \right]}

\newcommand{\eup}{\mathrm e}

\newcommand{\N}{\mathbb{N}}

\newcommand{\R}{\mathbb{R}}

\usepackage{bbm} %

\newcommand\inv[1]{{#1}^{-1}}  %

\newcommand{\argmin}{\mathop{\mathrm{argmin}}}
\newcommand{\argmax}{\mathop{\mathrm{argmax}}}

\newcommand{\Var}{{\mathop{\mathrm{Var}}}}  %
\newcommand{\Cov}{{\mathop{\mathrm{Cov}}}}  %

\newcount\colveccount
\newcommand*\colvec[1]{
        \global\colveccount#1
        \begin{pmatrix}
        \colvecnext
}
\def\colvecnext#1{
        #1
        \global\advance\colveccount-1
        \ifnum\colveccount>0
                \\
                \expandafter\colvecnext
        \else
                \end{pmatrix}
        \fi
}

\newcommand{\vtheta}{\mbox{\boldmath $\theta$}}

\newcommand{\vmu}{\mbox{\boldmath $\mu$}}

\newcommand{\vphi}{\mbox{\boldmath $\phi$}}

\newcommand{\vomega}{\mbox{\boldmath $\omega$}}

\newcommand{\va}{\mathbf a}

\newcommand{\vc}{\mathbf c}

\newcommand{\vr}{\mathbf r}

\newcommand{\vw}{\mathbf w}
\newcommand{\vx}{\mathbf x}

\newcommand{\vz}{\mathbf z}

\newcommand{\cD}{\mathcal{D}}

\newcommand{\cL}{\mathcal{L}}

\newcommand{\cN}{\mathcal{N}}

\newcommand{\cX}{\mathcal{X}}

\newcommand{\cZ}{\mathcal{Z}}

\newcommand{\bbE}{\mathbb{E}}

\newcommand{\design}{\lambda}
\newcommand{\simplex}{\Sigma}

\def\calzhat{\widehat{\mathcal{Z}}}

\renewcommand{\simplex}{
  \mathchoice
    {\includegraphics[height=1.4ex, trim= 0pt 30pt 0pt 0pt]{simplex.pdf}} %
    {\includegraphics[height=1.4ex, trim= 0pt 25pt 0pt 0pt]{simplex.pdf}} %
    {\includegraphics[height=1.1ex, trim=0pt 30pt 0pt 0pt]{simplex.pdf}} %
    {\includegraphics[height=.8ex, trim=0pt 30pt 0pt 0pt]{simplex.pdf}} %
}

\newcommand{\CandPol}{\Pi_c}
\newcommand{\CandStates}{\mathcal{Q}_c}
\newcommand{\Dataset}{\mathcal{D}}
\newcommand{\ExpRet}[1]{G({#1})}
\newcommand{\StateVisit}[1]{\mathbf{f}^{#1}}
\newcommand{\StateVisitDirection}[2]{\mathbf{v}_{#1,#2}}
\renewcommand{\simplex}{\Delta}
\newcommand{\InformationGain}[2]{I ( #1, #2 )}
\newcommand{\ConditionalInformationGain}[3]{I ( #1; #2 | #3 )}
\newcommand{\Entropy}[1]{H ( #1 )}
\newcommand{\ConditionalEntropy}[2]{H ( #1 | #2 )}
\newcommand{\KL}[2]{D_{\text{KL}} ( #1 \| #2 )}
\newcommand{\VectorProd}[2]{\langle #1, #2 \rangle}
\newcommand{\ExpRetBelief}[1]{\hat{G}({#1})}
\newcommand{\RfunBelief}{\hat{r}}

\newcommand{\MeanRfun}{\bar{\Rfun}}
\newcommand{\MeanOptPol}{\bar{\p}^*}

\newcommand{\Query}{q}
\newcommand{\Evidence}{y}
\newcommand{\EvidenceBelief}{\hat{\Evidence}}

\newcommand{\linfac}{c}
\newcommand{\linfactors}{C}
\newcommand{\linstates}{S}

\newcommand{\Segment}{\sigma}

\newcommand{\legendEI}{\includegraphics[height=1.4ex]{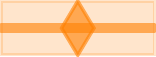}}
\newcommand{\legendIGGP}{\includegraphics[height=1.4ex]{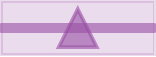}}
\newcommand{\legendUniform}{\includegraphics[height=1.4ex]{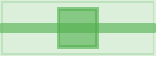}}
\newcommand{\legendIDRL}{\includegraphics[height=1.4ex]{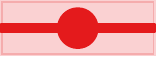}}
\newcommand{\legendMAXREG}{\includegraphics[height=1.4ex]{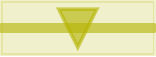}}
\newcommand{\legendEPD}{\includegraphics[height=1.4ex]{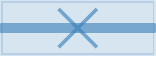}}
\newcommand{\legendIDRLABLATION}{\includegraphics[height=1.4ex]{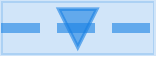}}
\newcommand{\legendEXPERT}{\includegraphics[height=1.4ex]{figures/legend_EPD.pdf}}

\newcommand{\result}[2]{$#1 \pm #2$}
\newcommand{\hresult}[2]{{\boldmath $#1 \pm #2$ }}

\usepackage{xspace}
\newcommand{\AlgoNameLong}{Information Directed Reward Learning\xspace}
\newcommand{\AlgoNameShort}{\texttt{IDRL}\xspace}

\usepackage{amsthm}
\usepackage{thmtools}
\usepackage{thm-restate}

\usepackage{booktabs}
\usepackage{multirow, makecell}
\usepackage[para]{threeparttable}
\usepackage{adjustbox}
\newcommand{\specialcell}[2][c]{%
  \begin{tabular}[#1]{@{}c@{}}#2\end{tabular}}

\usepackage{tikz}
\usetikzlibrary{matrix}
\usetikzlibrary{calc,fit,tikzmark}
\usetikzlibrary{positioning,angles,quotes}
\tikzstyle{every picture}+=[remember picture]
\usetikzlibrary{arrows,automata,calc}
\usepackage{pgfplots}
\pgfplotsset{compat=1.16}

\usepackage{wrapfig}
\usepackage{subcaption}

\theoremstyle{definition}

\crefname{observation}{observation}{observations}
\Crefname{observation}{Observation}{Observations}
\declaretheorem[name=Definition,numberwithin=section]{definition}
\crefname{definition}{definition}{definitions}
\Crefname{definition}{Definition}{Definitions}

\crefname{proposition}{proposition}{propositions}
\Crefname{proposition}{Proposition}{Propositions}

\crefname{equation}{equation}{equations}
\Crefname{equation}{Equation}{Equations}

\newcommand{\paragraphsmall}[1]{\vspace{-1em}\paragraph{#1}}
\usepackage{pifont}
\newcommand{\cmark}{{\color{OliveGreen}\ding{51}}}
\newcommand{\xmark}{{\color{BrickRed}\ding{55}}}

\newcommand{\github}{\url{https://github.com/david-lindner/idrl}}

\title{Information Directed Reward Learning \\ for Reinforcement Learning}

\author{%
  David Lindner \\
  Department of Computer Science \\
  ETH Zurich \\
  \texttt{david.lindner@inf.ethz.ch} \\
  \And
  Matteo Turchetta \\
  Department of Computer Science \\
  ETH Zurich \\
  \texttt{matteo.turchetta@inf.ethz.ch} \\
  \And
  Sebastian Tschiatschek \\
  Department of Computer Science \\
  University of Vienna \\
  \texttt{sebastian.tschiatschek@univie.ac.at} \\
    \And
  Kamil Ciosek\thanks{Work done while at Microsoft Research Cambridge.} \\
  Spotify \\
  \texttt{kamilc@spotify.com} \\
    \And
  Andreas Krause \\
  Department of Computer Science \\
  ETH Zurich \\
  \texttt{krausea@ethz.ch} \\
}

\begin{document}

\maketitle

\begin{abstract}
For many reinforcement learning (RL) applications, specifying a reward is difficult.
This paper considers an RL setting where the agent obtains information about the reward only by querying an expert that can, for example,
evaluate individual states or provide binary preferences over trajectories.
From such expensive feedback, we aim to learn a model of the reward that allows standard RL algorithms to achieve high expected returns \emph{with as few expert queries as possible}.
To this end, we propose \emph{\AlgoNameLong} (\AlgoNameShort), which uses a Bayesian model of the reward and selects queries that maximize the information gain about the difference in return between plausibly optimal policies.
In contrast to prior active reward learning methods designed for specific types of queries, \AlgoNameShort naturally accommodates different query types. Moreover, it achieves similar or better performance with significantly fewer queries by shifting the focus from reducing the reward approximation error to improving the policy induced by the reward model. We support our findings with extensive evaluations  in multiple environments and with different query types.
\end{abstract}

\section{Introduction}\label{sec:introduction}

\emph{Reinforcement learning} \citep[RL;][]{sutton2018reinforcement} casts the problem of learning to perform complex tasks by interacting with an environment as an optimization problem where the learning agent aims to maximize its expected cumulative reward.
Despite the remarkable successes of RL \citep[e.g.,][]{mnihHumanlevelControlDeep2015,silverMasteringGameGo2016}, specifying reward functions that capture complex tasks is still an open problem. A promising approach is to \emph{learn} a reward function from human feedback \citep[e.g.,][]{christiano2017deep}. However, since human feedback is expensive, \emph{active reward learning} aims to minimize the number of queries. Prior work often focuses on approximating the reward function uniformly well. However, this may not be aligned with the original goal of RL: \emph{finding an optimal policy}, as \Cref{fig:gridworld} shows. Moreover, prior work is often tailored to specific types of queries, such as comparisons of two trajectories \citep[e.g.,][]{dorsa2017active} or numerical evaluations of trajectories \citep[e.g.,][]{daniel2015active}, limiting its applicability.

\begin{figure}[t]\centering
\newcommand{\iconCherry}{\ensuremath{%
  \mathchoice{\includegraphics[height=2ex]{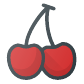}}
    {\includegraphics[height=2ex]{figures/gridworld/icons/Cherry.pdf}}
    {\includegraphics[height=1.5ex]{figures/gridworld/icons/Cherry.pdf}}
    {\includegraphics[height=1ex]{figures/gridworld/icons/Cherry.pdf}}
}}
\newcommand{\iconApple}{\ensuremath{%
  \mathchoice{\includegraphics[height=2ex]{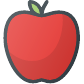}}
    {\includegraphics[height=2ex]{figures/gridworld/icons/Apple.pdf}}
    {\includegraphics[height=1.5ex]{figures/gridworld/icons/Apple.pdf}}
    {\includegraphics[height=1ex]{figures/gridworld/icons/Apple.pdf}}
}}
\newcommand{\iconCorn}{\ensuremath{%
  \mathchoice{\includegraphics[height=2ex]{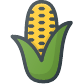}}
    {\includegraphics[height=2ex]{figures/gridworld/icons/Corn.pdf}}
    {\includegraphics[height=1.5ex]{figures/gridworld/icons/Corn.pdf}}
    {\includegraphics[height=1ex]{figures/gridworld/icons/Corn.pdf}}
}}
\newcommand{\iconPear}{\ensuremath{%
  \mathchoice{\includegraphics[height=2ex]{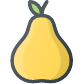}}
    {\includegraphics[height=2ex]{figures/gridworld/icons/Pear.pdf}}
    {\includegraphics[height=1.5ex]{figures/gridworld/icons/Pear.pdf}}
    {\includegraphics[height=1ex]{figures/gridworld/icons/Pear.pdf}}
}}

\begin{minipage}[t]{0.32\linewidth}\vspace{0.4cm}
    \includegraphics[width=\linewidth]{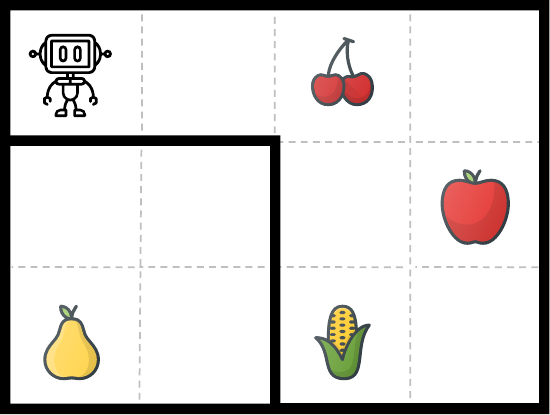}\\
    $T = 4$
\end{minipage}\hfill
\begin{minipage}[t]{0.67\linewidth}\vspace{0pt}
    \begin{minipage}[t]{0.4\linewidth}\vspace{0pt}\hspace{0pt}
        \begin{tabular}{>{\centering\arraybackslash} m{1.9cm} >{\centering\arraybackslash} m{1.9cm}}
            $\p_1$ & $\p_2$ \\
             \includegraphics[width=\linewidth]{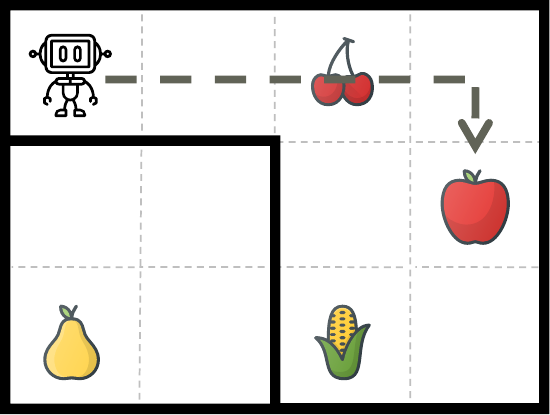} & \hspace*{-1em}\includegraphics[width=\linewidth]{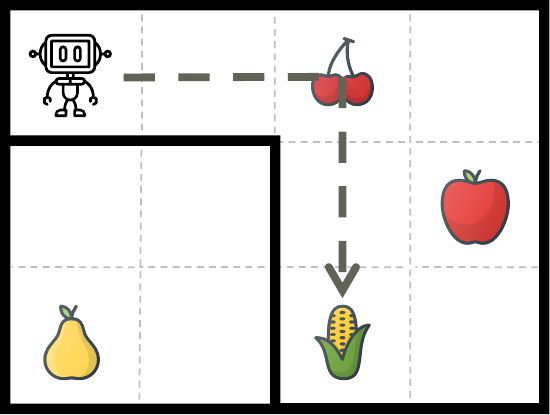}
        \end{tabular}
    \end{minipage}\hspace{0.4cm}
    \begin{minipage}[t]{0.59\linewidth}
    \begin{align*}
        &\ExpRetBelief{\p_1} = \RfunBelief({\protect\iconCherry}) + \RfunBelief({\protect\iconApple}) \\
        &\ExpRetBelief{\p_2} = \RfunBelief({\protect\iconCherry}) + \RfunBelief({\protect\iconCorn}) \\
        &\ExpRetBelief{\p_1} - \ExpRetBelief{\p_2} = \RfunBelief({\protect\iconApple}) - \RfunBelief({\protect\iconCorn})
    \end{align*}
    \end{minipage}
    
    \hspace*{0.08cm}
    \begin{tcolorbox}[width=0.95\linewidth, height=1.65cm, nobeforeafter,
                     colframe=black!100, colback=blue!0, boxrule=1.5pt]
    \vspace*{-0.2cm}
    \begin{align*}
        \argmax_{\Query \in \{{\protect\iconCherry},  {\protect\iconApple}, {\protect\iconCorn}, {\protect\iconPear}\}} \underbrace{\InformationGain{\ExpRetBelief{\p_1} - \ExpRetBelief{\p_2}}{(\Query, \EvidenceBelief)}}_{\text{\AlgoNameShort objective}}  = \{{\protect\iconApple}, {\protect\iconCorn}\}%
    \end{align*}
    \end{tcolorbox}
\end{minipage}
\caption{\small The robot wants to collect food for a human. It can only move $4$ timesteps in the gridworld, cannot pass through the black walls, and collecting more food is always better. The robot does not know the human's preferences, but it can ask for food ratings. Common active learning methods aim to learn the reward uniformly well, and would query all items similarly often. In contrast, \AlgoNameShort considers only the two plausibly optimal policies $\p_1$ and $\p_2$. Since both policies collect the cherry, and do not collect the pear, the robot only needs to learn about the apple and the corn. \AlgoNameShort can solve the task with 2 queries instead of 4.}
\label{fig:gridworld}
\end{figure}

\looseness -1
\paragraphsmall{Contributions.}
We propose \emph{\AlgoNameLong} (\AlgoNameShort), a general \emph{active reward learning} approach for learning a model of the reward function from expensive feedback with the goal of \emph{finding a good policy} rather than uniformly reducing the model's error. \AlgoNameShort can use arbitrary Bayesian reward models and arbitrary types of queries (\Cref{sec:idrl_general}), making it more general than existing methods. We describe an exact and efficient implementation of \AlgoNameShort using Gaussian process (GP) reward models (\Cref{sec:idrl_gp}) and different types of queries, and an approximation of \AlgoNameShort that uses a deep neural network reward model and a state-of-the-art policy gradient algorithm to learn from comparison queries (\Cref{sec:idrl_deep_rl}). We evaluate \AlgoNameShort extensively in simulated environments (\Cref{sec:experiments}), including a driving task and high-dimensional continuous control tasks in the MuJoCo simulator, and show that both implementations significantly outperform prior methods.

\section{Related work}

\paragraph{Reward Learning for RL.}
Several works aim to directly learn policies rather than reward functions from expert feedback in the form of numerical evaluations \citep{knox2009interactively,macglashan2017interactive} or comparisons \citep{regan2009regret,furnkranz2012preference}, and some work also explores active query selection \citep{akrour2012april,wilson2012bayesian}. However, learning policies directly from feedback has several downsides: it is difficult to combine different types of feedback, and policies tend to generalize poorly between environments.
\citet{ng2000algorithms} argue that reward functions are a more robust representation of desired behavior than policies. \emph{Inverse reinforcement learning} (IRL) aims to learn a reward model from expert demonstrations \citep{abbeel2004apprenticeship}. Reward models can also be learned from comparisons of two or more different behaviors \citep{wirth2017survey}, or other kinds of feedback \citep{jeon2020reward}. While reward models can successfully learn hard-to-specify control and game-playing tasks \citep{christiano2017deep,ibarz2018reward}, most work uses simple heuristics to select queries to make. In contrast, \emph{active reward learning} aims to select the most informative queries in a principled way.

\paragraphsmall{Active reward learning.}
For linear reward functions, \citet{dorsa2017active} ask the expert to compare trajectories synthesized to maximize the volume removed from a hypothesis space. \citet{biyik2019asking} argue that maximizing information gain leads to better sample efficiency and queries that are easier to answer than volume removal. \citet{biyik2020active} generalize maximizing information gain to non-linear reward functions using a GP model. We also use an information gain objective to select queries; however, our approach focuses on finding an optimal policy instead of uniformly reducing the error of the reward model. With a similar motivation, \citet{wilde2020active} aim to capture how informative a query is for distinguishing policies. However, their method is limited to comparisons between potentially optimal policies. \citet{daniel2015active} also introduce an acquisition function to measure how informative a query is for learning a good policy. However, their setting is restricted to observing the cumulative reward of a trajectory, and their acquisition function is computationally expensive. \Cref{tab:prior_work} gives an overview of how our method compares to this prior work.

\paragraphsmall{Bayesian optimization.}
\emph{Bayesian optimization} (BO) aims to maximize an expensive-to-evaluate function by learning a Bayesian model of the function and selecting informative queries \citep{mockus1978application}. We face a related but significantly harder problem: we aim to find an optimal policy in RL but only \emph{indirectly} obtain information about the value of policies.
Our problem has striking connections to variants of the \emph{multi-armed bandit problem} \citep{bubeck2012regret}, in particular \emph{partial monitoring problems} \citep{rustichini1999minimizing} and \emph{transductive linear bandits} \citep{fiez2019sequential}. We explore this connection in detail in \Cref{app:connection_TLB}.

\begin{table}[t]
\newcommand{\extension}{\xmark\tnote{1}}
\centering
\begin{adjustbox}{max width=\linewidth}\begin{threeparttable}
    \begin{tabular}{ccccccc}
       \toprule
        & \specialcell{Non-linear \\ Rewards} & \specialcell{Single-state \\ Queries} & \specialcell{Trajectory \\ Queries} & \specialcell{Numerical \\ Queries} & \specialcell{Comparison \\ Queries} & \specialcell{Considers \\ Env. Dynamics} \\
        \midrule
        \citet{dorsa2017active} & \xmark & \xmark & \cmark & \xmark & \cmark & \xmark \\
        \citet{biyik2019asking} & \xmark & \extension & \cmark & \extension & \cmark & \xmark \\
        \citet{biyik2020active} & \cmark & \extension & \cmark & \extension & \cmark & \xmark \\
        \citet{daniel2015active} & \cmark & \extension & \cmark & \cmark & \extension & \cmark  \\
        \citet{wilde2020active} & \xmark & \xmark & \cmark & \xmark & \cmark & \cmark \\
        \AlgoNameShort (ours) & \cmark & \cmark & \cmark & \cmark & \cmark & \cmark \\
       \bottomrule
    \end{tabular}
    \begin{tablenotes}\footnotesize
        \item[1] The original authors do not consider this setting, but we provide an extension to their method in \Cref{sec:experiments} and \Cref{app:baselines_details}.
    \end{tablenotes}\vspace{1em}
\end{threeparttable}\end{adjustbox}
\caption{\small In contrast to most prior work on active reward learning for RL, \AlgoNameShort can handle non-linear reward functions and different query types, in particular numerical evaluations and comparisons of individual states and (partial) trajectories. Further, \AlgoNameShort takes the environment dynamics into account to achieve better sample efficiency (cf. \Cref{fig:gridworld}).}\label{tab:prior_work}
\end{table}

\section{Background and problem setting}\label{sec:problem}

\paragraph{Markov decision process.}
Markov decision processes \citep[MDPs;][]{puterman2014markov} model sequential decision-making problems in dynamical systems. An MDP $\MDP$ consists of a state space $\sset$, an action space $\aset$, a transition function $\Tfun$, a reward function $\Rfun$, an initial state distribution $\initS$, and a discount factor $\Ddef$. In an MDP, the agent starts in state $s_0 \sim \initS(s)$ and, when taking action $a_t$, transitions from state $s_t$ to state $s_{t+1}$ with probability $\Tfun(s_{t+1} | s_t, a_t)$.
The agent affects the environment through actions determined by a policy $\p(a_t | s_t)$, indicating the probability of taking action $a_t$ in state $s_t$. The agent's goal is to find a policy $\p$ that maximizes the \textit{expected discounted return} $\ExpRet{\p} = \bbE_{\Tfun, \p, \initS} \sbr{\sum_{t=0}^{\infty} \D^t \Rfun_t}$, where $\Rfun_t$ is the reward obtained at time $t$.

\paragraphsmall{Information gain.} 
Intuitively, the information gain between two random variables measures the amount of information that can be obtained about one of them by observing the other. Formally, for two random variables $X$ and $Y$ with marginal distributions $p_X$, $p_Y$ and joint distribution $p_{(X,Y)}$, the \emph{information gain} (or mutual information) is $\InformationGain{X}{Y} = \KL{p_{(X,Y)}}{p_X \cdot p_Y}$, where $\KL{\cdot}{\cdot}$ is the KL-divergence. Given a third random variable $Z$, conditional information gain is defined as $\ConditionalInformationGain{X}{Y}{Z=z} = \KL{p_{(X,Y)|Z=z}}{p_{X|Z=z} \cdot p_{Y|Z=z}}$.

\paragraphsmall{Problem setting.}
We focus on MDPs where the reward function is not readily available. Instead, the agent can query an expert for information about the reward. In iteration $i$, the agent makes a query $\Query_i$ to the expert, and receives a response $\Evidence_i$. For example, $\Query_i$ could ask the expert to compare two trajectories or judge a single trajectory, and $\Evidence_i$ could indicate which of the two trajectories is better or provide the return of a single trajectory. We assume that the agent can interact with the environment cheaply, but queries to the expert are \emph{expensive}, and hence the agent has to find a policy $\p$ that \emph{maximizes the expected return} $\ExpRet{\p}$  using \emph{as few queries as possible}.

\paragraph{Our reward learning approach.} We approach this problem by learning a model of the reward function, i.e., a model that predicts the reward of a given state,%
\footnote{Our approach is also applicable to reward functions that depend on state-action pairs or transitions. We focus on state-dependent reward functions for simplicity of exposition.}
and computing a policy that maximizes the return induced by the model.
Importantly, we want to learn a reward model such that the induced optimal policy achieves a high return under the true reward function.
Note that any RL algorithm can be used to find the policy. Hence, the problem reduces to selecting a model for the reward function and deciding which queries to make. Our key insight is that queries that help most to find a good policy might differ from those that uniformly reduce the model's uncertainty.

\section{The \AlgoNameLong acquisition function}\label{sec:idrl_general}

This section introduces \AlgoNameLong (\AlgoNameShort) for a general Bayesian model of the reward and discusses how to select queries $\Query_i$, making no assumptions on their form nor on the responses $\Evidence_i$.

\paragraph{Reward model.}
To select informative queries, we need to quantify uncertainty; hence, we use a Bayesian model of the reward function. From a Bayesian perspective, it is important to distinguish between the agent's belief about a quantity and its ``actual'' value that is unknown to us. We denote the belief about the reward in state $s$ with $\RfunBelief(s)$ and its actual value with $\Rfun(s)$.

\paragraph{Query selection.} To select informative queries, we have to consider that the responses might only give \emph{indirect} information about the set of optimal policies. For example, assume the agent can ask the expert to quantify the reward of individual states. These rewards  provide information about the expected return of a policy but may yield no information about the set of optimal policies. For example, if a state is visited similarly often by every plausibly optimal policy, knowing its reward does not help decide between the policies (e.g., the cherry in \Cref{fig:gridworld}). Therefore, any approach that only aims to reduce the uncertainty of the reward model may waste expensive queries that do not help find an optimal policy.

\begin{algorithm}[t]
\caption{\emph{\AlgoNameLong} (\AlgoNameShort). The algorithm requires a set of candidate queries $\CandStates$, a Bayesian model of the reward function, and an RL algorithm that returns a policy given a reward function. $\ExpRetBelief{\p}$ is the belief about the expected return of policy $\p$, induced by the reward model $P(\RfunBelief | \Dataset)$, and $\RfunBelief$ is the belief about the reward function. }
\label{algorithm}
\begin{algorithmic}[1]
  \State $\Dataset \gets \{\}$; $\CandPol \gets$ initialize candidate policies;
  initialize reward model with prior distribution $P(\RfunBelief)$
  \While{not converged}
     \State Select a query:
     \State \hspace{\algorithmicindent}$\p_1, \p_2 \in \argmax_{\p, \p' \in \CandPol} \ConditionalEntropy{\ExpRetBelief{\p} - \ExpRetBelief{\p'}}{\Dataset}$
     \State \hspace{\algorithmicindent}$\Query^* \in \argmax_{\Query \in \CandStates} \ConditionalInformationGain{\ExpRetBelief{\p_1} - \ExpRetBelief{\p_2}}{(\Query, \EvidenceBelief)}{\Dataset}$
     \State Make query and update reward model:
     \State \hspace{\algorithmicindent}$\Evidence^* \gets$ Response to query $\Query^*$
     \State \hspace{\algorithmicindent}$P(\RfunBelief | \Dataset \cup \{(\Query^*, \Evidence^*)\} ) \propto P(\Evidence^* | \RfunBelief, \Dataset, \Query^*) P(\RfunBelief | \Dataset)$ \Comment{Update belief about the reward}
     \State \hspace{\algorithmicindent}$\Dataset \gets \Dataset \cup \{(\Query^*, \Evidence^*)\}$ \Comment{Add observation to dataset}
     \State Optionally update candidate policies $\CandPol$
  \EndWhile
  \State $\MeanRfun \gets$ mean estimate of the reward model;
  $\MeanOptPol \gets \mathtt{RL}(\MeanRfun)$
  \State \textbf{return} $\MeanOptPol$
\end{algorithmic}
\end{algorithm}

Intuitively, we want to instead select queries that \emph{help identify the optimal policy}. In the language of information theory, we want to maximize the \emph{information gain} of a query about \emph{the identity of the optimal policy}. More formally, if $\Dataset = \{(\Query_1, \Evidence_1), \dots, (\Query_t, \Evidence_t)\}$ is a dataset of past queries and responses, let us denote with $P(\hat{\p}^* | \Dataset)$ the agent's belief about the optimal policy, induced by our belief about the reward function $\RfunBelief$. Also, let $\CandStates$ be a set of candidate queries the agent can make. Then, one way to formalize this intuition is to select queries
$
\Query^* \in \argmax_{\Query \in \CandStates} I(\hat{\p}^*; (\Query, \EvidenceBelief) | \Dataset)
$,
where $\EvidenceBelief$ is the agent's belief about the response it will get to query $\Query$, and $I$ denotes the information gain. Unfortunately, this objective has two undesirable properties. First, the agent has to keep track of a distribution over all possible policies to compute it, which is intractable in general. Second, reducing uncertainty about the optimal policy only matters as long as there are significant differences in the return of plausibly optimal policies. For example, if the agent identifies a set of plausibly optimal policies with similar returns, we care less about identifying exactly which policy is optimal, compared to when such policies have very different returns.

To address the first challenge, we obtain a finite set of \emph{candidate policies} $\CandPol$ that are \emph{plausibly optimal} according to our Bayesian reward model. To address the second challenge, we select the most informative query for distinguishing policies in terms of their value.

Let us first discuss how to select queries, assuming a set of plausibly optimal policies $\CandPol$ to be available. We can exploit the fact that the belief about the reward function $\RfunBelief$ induces a belief about the expected return of policy $\p \in \CandPol$, denoted as $\ExpRetBelief{\p}$.
Concretely, the expected return of a policy can be computed as the scalar product $\ExpRet{\p} = \VectorProd{\StateVisit{\p}}{\vr}$,
where $\StateVisit{\p}$ is a vector of the (discounted) expected state-visitation frequencies of policy $\p$ and $\vr$ is a vector of rewards of the corresponding states. We can estimate $\StateVisit{\p}$ from trajectories sampled using policy $\p$, and then determine $\ExpRetBelief{\p}$ from $\RfunBelief$.

Given $\ExpRetBelief{\p}$, \AlgoNameShort proceeds in two steps. It first selects two policies that maximize the model's uncertainty about the difference in their expected returns: 
\begin{align}
\p_1, \p_2 \in \argmax_{\p, \p' \in \CandPol} \ConditionalEntropy{\ExpRetBelief{\p} - \ExpRetBelief{\p'}}{\Dataset},
\label{eq:dig_policy_selection}
\end{align}
where $H$ is the entropy of the belief conditioned on past queries. To gather information about the distinction between $\p_1$ and $\p_2$, \AlgoNameShort then selects queries that maximize the information gain about the difference in expected return between $\p_1$ and $\p_2$:
\begin{align}
\Query^* \in \argmax_{\Query \in \CandStates} \ConditionalInformationGain{\ExpRetBelief{\p_1} - \ExpRetBelief{\p_2}}{(\Query, \EvidenceBelief)}{\Dataset}.
\label{eq:dig_state_selection}
\end{align}
Note that this is not the same as jointly maximizing the information gain about $\ExpRetBelief{\p_1}$ and $\ExpRetBelief{\p_2}$. \Cref{eq:dig_state_selection} prefers queries that help to distinguish $\p_1$ and $\p_2$ over queries that help to determine the exact value of $\ExpRetBelief{\p_1}$ and $\ExpRetBelief{\p_2}$.
Assuming a set of optimal policies is contained in $\CandPol$, reducing the uncertainty about the difference in returns within $\CandPol$ will help to identify an optimal policy quickly. In particular, if there is no remaining uncertainty about the differences in return, we can clearly identify an optimal policy.

Let us now discuss how to obtain a set of candidate policies $\CandPol$. For \AlgoNameShort to select informative queries, $\CandPol$ has to reflect the agent's current belief about optimal policies. \AlgoNameShort uses Thompson sampling \citep[TS,][]{thompson1933likelihood} as a flexible way to create $\CandPol$. We implement TS by repeatedly sampling a reward function from the posterior reward model and finding an approximately optimal policy for this sampled reward function. This approximates sampling from the posterior distribution over optimal policies. 
Since TS is demanding, in our experiments, we investigate two effective alternatives to alleviate its computational burden: (1) we update $\CandPol$ in regular intervals, rather than at every step, and (2) we start from the candidates computed in previous steps rather than starting the policy optimization from scratch.

\Cref{algorithm} shows the full \AlgoNameShort algorithm. In each iteration, \AlgoNameShort identifies two plausibly optimal policies with high uncertainty about their difference in return and then aims to reduce this uncertainty. We can stop the algorithm after a fixed number of queries, or by checking a convergence criterion, and return a policy $\MeanOptPol$ that is optimized for the current reward model.

Note that \AlgoNameShort is agnostic to how the candidate queries $\CandStates$ are generated. Different applications might require different approaches to generating $\CandStates$. In our experiments, for example, we consider: using all possible queries in small environments, choosing states or trajectories to query from rollouts of the currently optimal policy $\MeanOptPol$, selecting queries from rollouts of the candidate policies, and selecting queries from trajectories of a pre-defined explorations policy. Importantly, all of these, and others, are compatible with \AlgoNameShort.

\section{An exact and efficient implementation of \AlgoNameShort for GP reward models}\label{sec:idrl_gp}

Here, we describe one concrete implementation of \AlgoNameShort using a Gaussian process \citep[GP,][]{rasmussen2006gaussian} reward model and linear query types. These choices allow us to compute equations \labelcref{eq:dig_policy_selection} and \labelcref{eq:dig_state_selection} exactly and efficiently.

\paragraph{Reward model.}
We model the reward function as a GP with (w.l.o.g.) a zero-mean prior distribution $\RfunBelief(s) \sim \GP(0,k(s,s'))$ using a kernel $k$ which measures the similarity of states.

\paragraph{Query selection.}
We first show how to compute equations \labelcref{eq:dig_policy_selection} and \labelcref{eq:dig_state_selection} if the posterior belief about the reward function is Gaussian. Then, we discuss a family of practically relevant query types that satisfy this assumption. We provide proofs for all results in \Cref{app:proofs}.

\begin{restatable}{proposition}{IDRLGP}\label{IDRLGP}
If $\RfunBelief(s) | \Dataset$ is a GP, then $P(\ExpRetBelief{\p} - \ExpRetBelief{\p'} | \Dataset)$ is Gaussian and:
\begin{align*}
    \argmax\limits_{\p, \p' \in \CandPol} \ConditionalEntropy{\ExpRetBelief{\p} - \ExpRetBelief{\p'}}{\Dataset} &= \argmax\limits_{\p, \p' \in \CandPol} \Var[\ExpRetBelief{\p} - \ExpRetBelief{\p'} | \Dataset] \\
    \argmax_{\Query \in \CandStates} I(\ExpRetBelief{\p_1} - \ExpRetBelief{\p_2}; (\Query, \EvidenceBelief) | \Dataset)
&= \argmin_{\Query \in \CandStates} \Var[\ExpRetBelief{\p_1} - \ExpRetBelief{\p_2} | \Dataset \cup \{(\Query, \EvidenceBelief)\}]
\end{align*}
\end{restatable}\vspace{-0.3em}

We can compute both variances analytically, enabling exact implementation of equations \eqref{eq:dig_policy_selection} and \eqref{eq:dig_state_selection}.

\paragraph{Query types.} To apply this result, we need $\RfunBelief(s) | \Dataset$ to be a GP, which is not the case for general observations $(\Query_i, \Evidence_i)$. If the queries are individual states, i.e., $\Query_i = s_i$ and $\Evidence_i = r(s_i)$, the problem is standard GP regression, and $\RfunBelief(s) | \Dataset$ is a GP \citep{rasmussen2006gaussian}. More generally, a similar statement holds if the observations are \emph{linear} combinations of rewards.

\begin{restatable}{definition}{DefLinearObservations}\label{DefLinearObservations}
We call $\Query = (\linstates, \linfactors)$ a \emph{linear reward query}, if it consists of states $\linstates = \{s_{1}, \dots, s_{N}\}$ and linear weights $\linfactors = \{\linfac_{1}, \dots, \linfac_{N}\}$, and the response to query $\Query$ is a linear combination of rewards $\Evidence = \sum_{j=1}^N \linfac_{j} \Rfun(s_{j}) + \varepsilon$, with Gaussian noise $\varepsilon \sim \cN(0, \sigma_n^2)$.
\end{restatable}\vspace{-0.5em}

\begin{restatable}{proposition}{GPLinearObservations}\label{GPLinearObservations}
Let $\Query$ be a linear reward query. If the prior belief about the reward $\RfunBelief(s)$ is a GP, then the posterior belief about the reward $\RfunBelief(s) | (\Query,\Evidence)$ is also a GP.
\end{restatable}\vspace{-0.5em}

Linear reward queries result in a particularly efficient implementation of \AlgoNameShort. Of course, \AlgoNameShort with a GP model could be extended to non-linear observations using approximate inference. However, it turns out that many commonly used query types can be modeled as linear reward queries, including the return of trajectories or comparisons of trajectories (see \Cref{app:linear_reward_queries}).

\section{A scalable Deep RL approximation of \AlgoNameShort}\label{sec:idrl_deep_rl}

GP models provide a convenient way to implement \AlgoNameShort exactly. But, can \AlgoNameShort also be used if we can not model the reward function as a GP? Moreover, can we scale it to large environments on the scale of typical Deep RL applications?

To address these questions, we propose a second implementation of \AlgoNameShort using a deep neural network (DNN) reward model. To scale \AlgoNameShort to large Deep RL scenarios, we integrate it into a policy optimization algorithm, similar to \citet{christiano2017deep}.\footnote{We provide a detailed comparison between our setup and \citet{christiano2017deep} in \Cref{app:comparison_to_christiano}.} In our experiments, we focus on comparison queries, but it is straightforward to extend the algorithm to other query types.

\paragraphsmall{Reward model.} To model the reward function, we use adaptive basis function regression with DNNs, similar to \citet{snoek2015scalable}. Concretely, we train a DNN from comparisons of short clips of the agents behavior using the Bradley-Terry model and $\ell_2$-regularization. We then treat the learned representation as a basis function and the final layer of the DNN as a maximum \emph{a posteriori} (MAP) estimate of the parameters of a Bayesian logistic regression model. Finally, we approximate the full posterior using a Laplace approximation.

\paragraphsmall{Query selection.} Because of the Laplace approximation, the posterior distribution of $\RfunBelief(s) | \Dataset$ is Gaussian, and we can compute equations \labelcref{eq:dig_policy_selection} and \labelcref{eq:dig_state_selection} the same way we did for a GP reward model.

\paragraphsmall{Query types.} Similar to \citet{christiano2017deep}, we consider queries $\Query_i = (\Segment_i^1, \Segment_i^2)$ that compare two segments of trajectories $\Segment_i^1$ and $\Segment_i^2$, where the user responds with their preference $\Evidence_i \in \{-1, 1\}$.

\paragraphsmall{Candidate policies.} In large environments, it is infeasible to train new policies from scratch during the Thompson sampling step. To avoid this, we maintain a fixed set of policies that we update regularly, instead of training new policies from scratch whenever we receive new samples.

\paragraphsmall{Candidate queries.} We generate candidate queries by rolling out the current policy optimized for the mean estimate of the reward model, as well as the candidate policies and uniformly sampling pairs of segments from the resulting trajectories.

\paragraphsmall{Full algorithm.} We use a policy gradient algorithm to train a policy for the current reward model, and the candidate policies. Similar to \citet{christiano2017deep}, the agent queries comparisons following a fixed schedule in which the number of samples is proportional to $\frac{1}{T}$, where $T$ is the number of policy training steps, i.e., we provide more samples early during training and less later on. For more details, including full pseudocode for the Deep RL algorithm, see \Cref{app:comparison_to_christiano}.

\section{Experiments}\label{sec:experiments}

We empirically test \AlgoNameShort in several environments, ranging from gridworlds to complex continuous control tasks, and for several different query types, including numerical evaluations and comparisons of trajectories. Our evaluation covers most scenarios existing in the literature and shows that \AlgoNameShort attains comparable or superior performance to methods designed for specific scenarios.

In all experiments, the agent's queries are answered with simulated feedback based on an underlying true reward function unknown to the agent. We usually evaluate the \emph{regret} of a policy $\p$ trained using the reward model, i.e., $\ExpRet{\p^*} - \ExpRet{\p}$ for an optimal policy $\p^*$. If we do not know $\p^*$, we approximate it with a policy trained on the true reward function.

We first validate that GP-based \AlgoNameShort improves sample efficiency in simple gridworld environments for numerical and comparison queries (\Cref{sec:experiments_toy_envs}). Next, we consider the most common setup in the literature, that is,  learning from comparisons of trajectories, and compare GP-based \AlgoNameShort against alternative approaches in a driving simulator, proposed in prior work (\Cref{sec:experiments_comparison}). Then, we study another natural feedback type: ratings of clips of the agent's behavior. In this setting, we demonstrate how GP-based \AlgoNameShort can be scaled up to bigger environments in the MuJoCo simulator (\Cref{sec:experiments_gp_mujoco}). Finally, we further demonstrate scalability by considering the Deep RL implementation of \AlgoNameShort to learn standard MuJoCo tasks from comparisons of clips of trajectories, similar to \citet{christiano2017deep} (\Cref{sec:experiments_deep_rl_mujoco}).

For each environment, we choose our setup to be close to prior work to promote a fair comparison. This leads to some design choices, such as the RL solver or the query types, to differ between environments. As a side effect, this highlights \AlgoNameShort's generality.
\Cref{app:implementation_details_gp,app:implementation_details_deep_rl} describe the experimental setup in more detail, and we provide code to reproduce all experiments.\footnote{\github}

\subsection{Baselines}\label{sec:experiments_baselines}
We consider five baselines:
(\textit{i}) \emph{Uniform sampling} selects queries from $\CandStates$ with equal probability.
(\textit{ii}) \emph{Information gain on the reward} (IGR) selects queries that maximize information gain about the reward $\ConditionalInformationGain{(\Query, \EvidenceBelief)}{\RfunBelief}{\Dataset}$. For a GP model, this is equivalent to maximizing $\Var[\EvidenceBelief | \Dataset, \Query]$. \citet{biyik2019asking} use IGR to learn rewards from comparisons of trajectories; however, it can be extended to other query types.
(\textit{iii}) \emph{Expected improvement on the reward} (EIR) maximizes the improvement in the value of a query compared to the best observation so far, in \emph{expectation}, and is a common acquisition function in BO \citep{mockus1978application}. EIR can not be applied to comparison queries.
(\textit{iv}) \emph{Expected policy divergence} (EPD) is an active reward learning method introduced by \citet{daniel2015active}, which makes queries that maximally change the current policy. Since EPD updates the policy for each potential observation, it is prohibitively expensive for large $\CandStates$.
While EPD was introduced to query the return of trajectories, we extend it to other query types (cf. \Cref{app:baselines_details}).
(\textit{v}) \emph{Maximum regret} (MR) is an acquisition function proposed by \citet{wilde2020active}. It assumes access to a set of candidate reward functions and corresponding optimal policies. MR compares policies that perform well according to one reward function but poorly according to a different one. It can only be used with comparisons of full trajectories.
We also tested \emph{expected volume removal} \citep[EVR,][]{dorsa2017active} for comparison queries; however, we found it to get stuck often, which confirms the findings of \citet{biyik2019asking}. Note that IGR and EIR reduce uncertainty uniformly over the state space, while EPD and MR consider the environment dynamics.

\subsection{Can \AlgoNameShort improve sample efficiency by considering the environment dynamics?}\label{sec:experiments_toy_envs}

\begin{wrapfigure}{R}{0.45\textwidth}\centering\vspace*{-1em}
\begin{tabular}{>{\centering\arraybackslash}m{0.49\linewidth} @{\hskip3pt} >{\centering\arraybackslash}m{0.49\linewidth}}
\hspace*{2em}{\scriptsize Reward of States} & \hspace*{-0.1em} {\scriptsize Comparisons of States} \\
\includegraphics[height=9.7em]{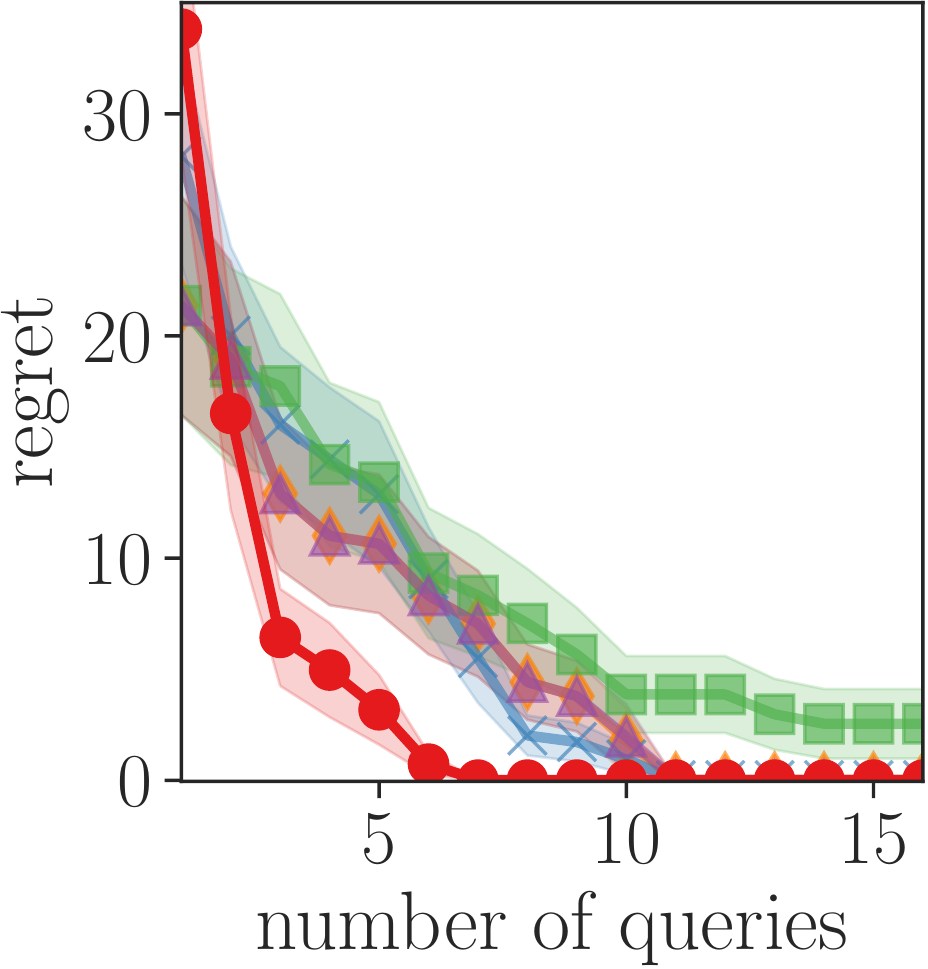} &
\includegraphics[height=9.7em]{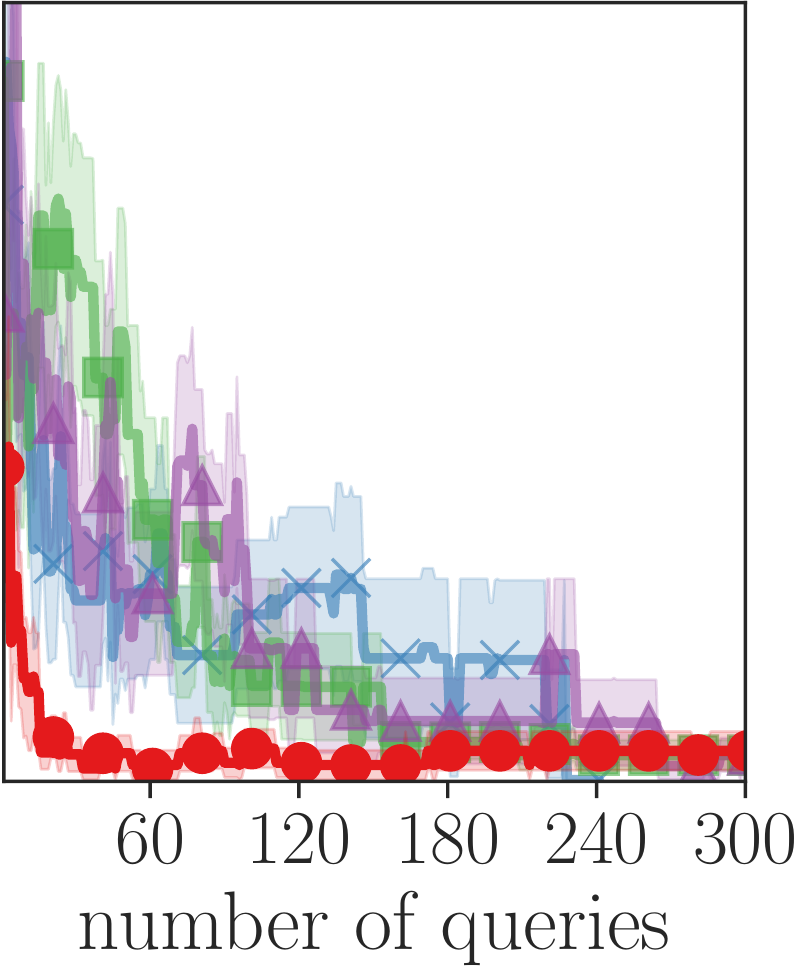}
\end{tabular}
\caption{\small Regret of a policy trained in randomly generated gridworlds as a function of the number of queries. The queries are either asking about the reward of individual states, or comparisons between two states. The plot compares IDRL (\legendIDRL) to EPD (\legendEPD), IGR (\legendIGGP) and uniform sampling (\legendUniform). For numerical queries, we also compare to EIR (\legendEI), which is not applicable to comparisons. MR is not applicable to single-state queries. The plots show mean and standard error over 30 random seeds.}\label{fig:results_gridworlds}
\vspace*{-1em}
\end{wrapfigure}

We first validate our hypothesis that \AlgoNameShort improves sample efficiency in small toy environments. Here, we highlight experiments in a set of \emph{Gridworlds} similar to \Cref{fig:gridworld}. \Cref{app:additional_results} presents two additional toy environments that isolate specific reasons why \AlgoNameShort outperforms the baselines.

\paragraph{Setup.}
We consider $10 \times 10$ \emph{Gridworlds} with randomly placed walls and objects with different rewards. The agent has to find the object with the largest reward.
We consider queries about the reward of individual states, i.e., $\Query_i = s_i \in \sset$ and $\Evidence_i = \Rfun(s_i)$, and comparison queries with $\Query_i = (s_{i1}, s_{i2})$ and $\Evidence_i \in \{-1, 1\}$. The candidate queries $\CandStates$ either consist of all states or all pairs of states. We use GP-based \AlgoNameShort, with a kernel that encodes which objects are the same, and which are different.
All experiments run for less than 1 hour on a single CPU.

\paragraph{Results.}
\Cref{fig:results_gridworlds} shows the regret of a policy trained on the reward model after different numbers of queries. \AlgoNameShort finds better policies than the baselines with a limited number of queries because it focuses on regions of the state space relevant for finding the optimal policy. As shown in \Cref{fig:gridworld}, this improves sample efficiency over methods that uniformly reduce uncertainty, such as IGR and EIR. \AlgoNameShort also outperforms EPD, because EPD's goal of selecting queries that maximally change the current policy is also misaligned with the goal of finding an optimal policy. We investigate EPD's specific failure modes in \Cref{app:additional_results}.

\subsection{Can \AlgoNameShort learn from comparisons of trajectories using a GP reward model?}\label{sec:experiments_comparison}

Most prior work studies reward learning from comparisons of trajectories. To evaluate \AlgoNameShort in this setting, we consider the 2-dimensional, continuous \emph{Driver} environment by \citet{dorsa2017active}.

\paragraph{Setup.}
In \emph{Driver}, the agent controls a car on a highway with another car driving on a fixed trajectory (cf. \Cref{fig:result_driver}). For each experiment, we randomly sample an underlying (linear) reward function to describe the desired driving behavior. We use code by \citet{dorsa2017active} to simulate and solve the environment, but we adapt it to our setting. In contrast to \citet{dorsa2017active}, we do not synthesize queries. Instead, we sample a fixed set of $200$ reward functions from a Gaussian prior distribution. We then optimize a policy for each of these reward functions, and, similarly to \citet{wilde2020active}, consider all pairs of policies as potential queries. 
Moreover, we assume a linear observation model (see \Cref{app:linear_reward_queries}), whereas  \citet{dorsa2017active} and \citet{wilde2020active} choose different non-linear observation models.
Each experiment runs for less than 24 hours on a single CPU.

\paragraph{Results.}
\Cref{fig:result_driver} shows the regret curves for the learned policy and the cosine similarity for the learned reward function weights. \AlgoNameShort outperforms the baselines and finds a better policy with fewer queries. However, the difference to pure information gain is small in this simple environment.

\begin{figure}[t]
\begin{adjustbox}{max width=\linewidth}\begin{tabular}{llllll}
    \textbf{Legend:} & & & & \\
    \legendUniform & Uniform Sampling &
    \legendEI & Expected Improvement (EI) &
    \legendIGGP & Information Gain on Reward (IGR) \\
    \legendEPD & Expected Policy Divergence (EPD) &
    \legendMAXREG & Maximum Regret &
    \legendIDRL & \AlgoNameShort (ours)
\end{tabular}\end{adjustbox}

\begin{minipage}{0.6\textwidth}
    \begin{subfigure}{\linewidth}
        \hspace*{1em}\includegraphics[width=0.22\linewidth]{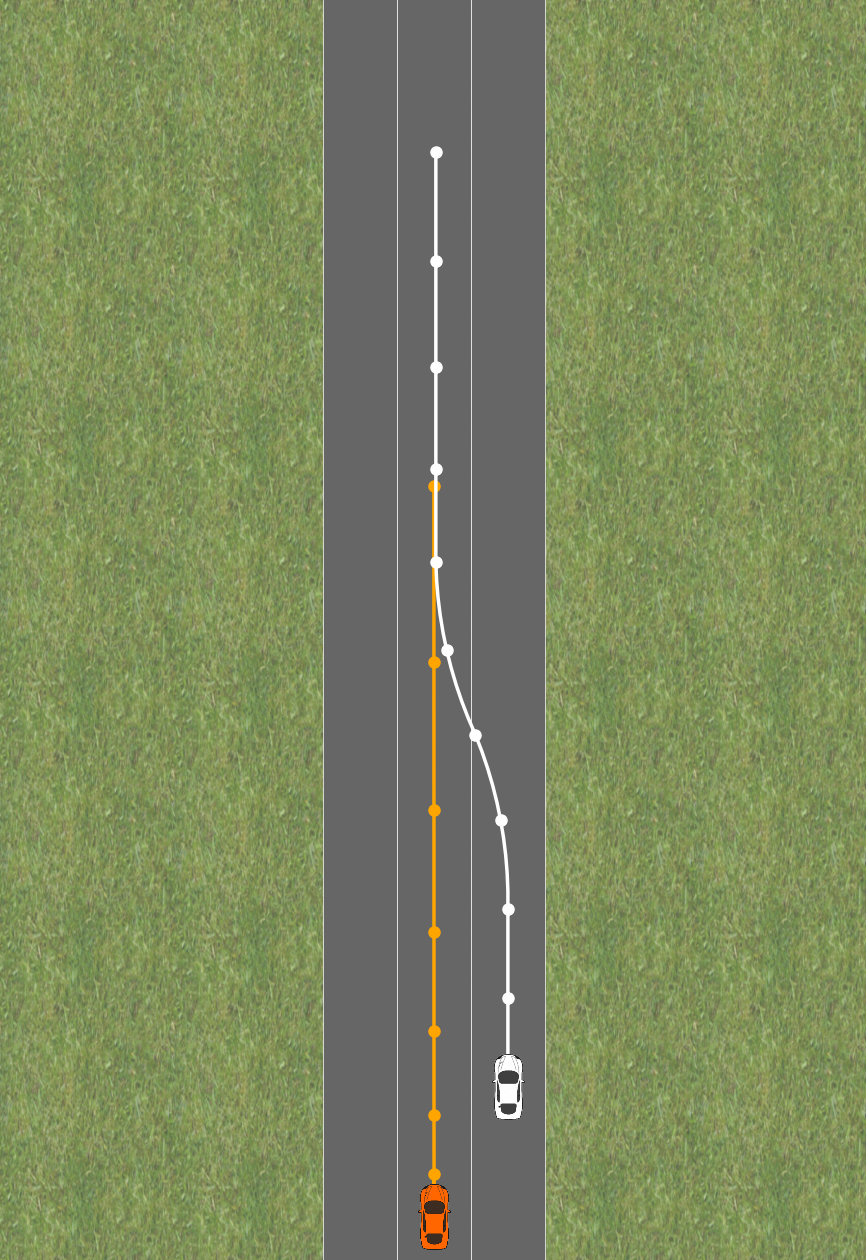}
        \hspace*{1em}\includegraphics[width=0.34\linewidth]{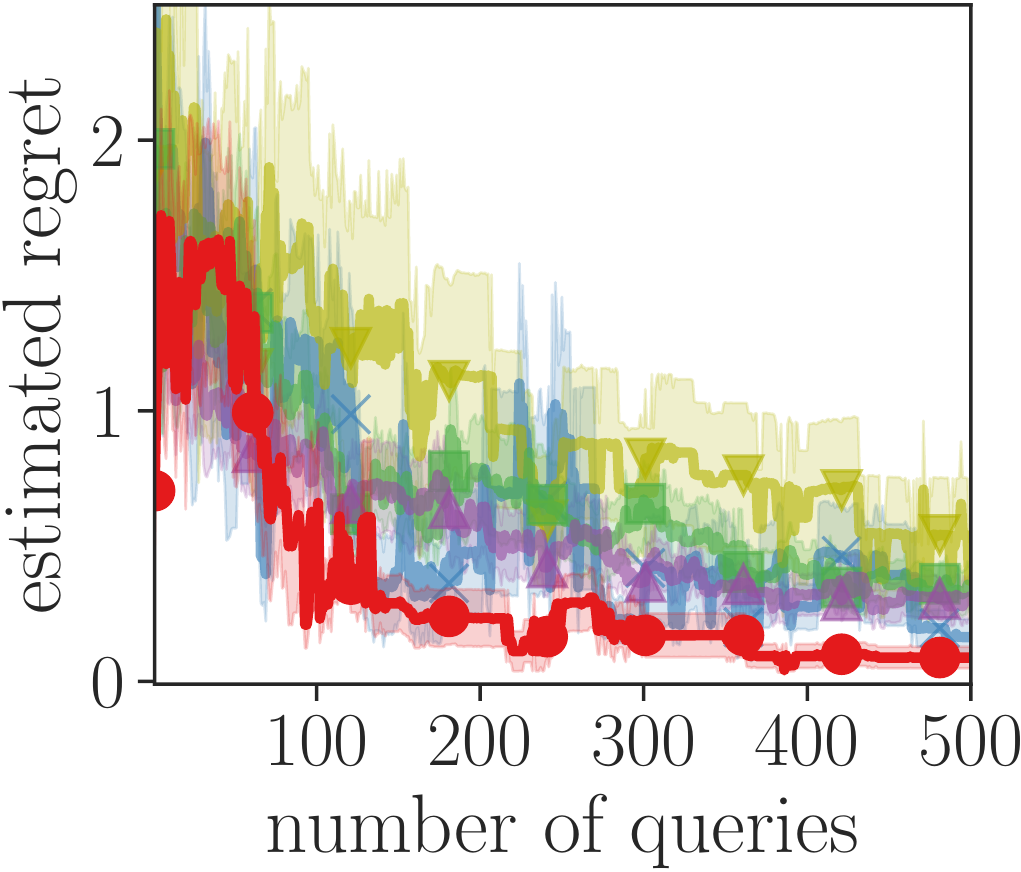}
        \includegraphics[width=0.34\linewidth]{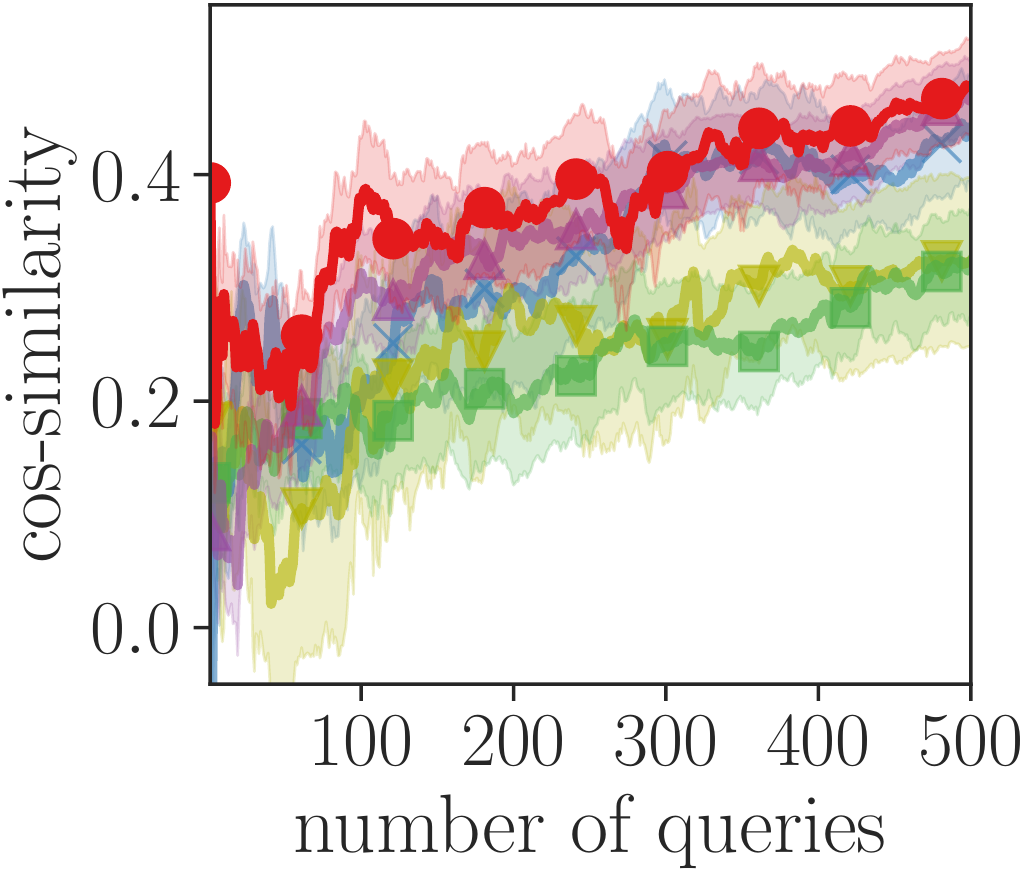}
        \caption{\emph{Driver} (comparison of trajectories)}\label{fig:result_driver}
    \end{subfigure} \\
    \begin{subfigure}{\linewidth}
        \raisebox{2.5em}{\includegraphics[width=0.3\linewidth]{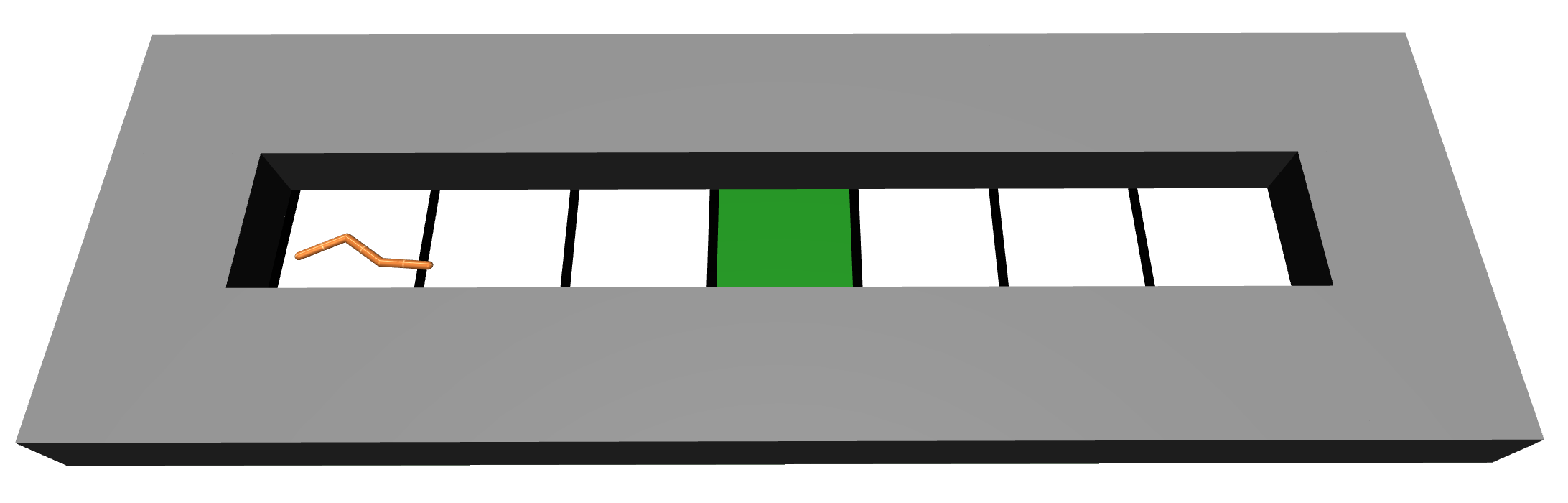}}
        \includegraphics[width=0.33\linewidth]{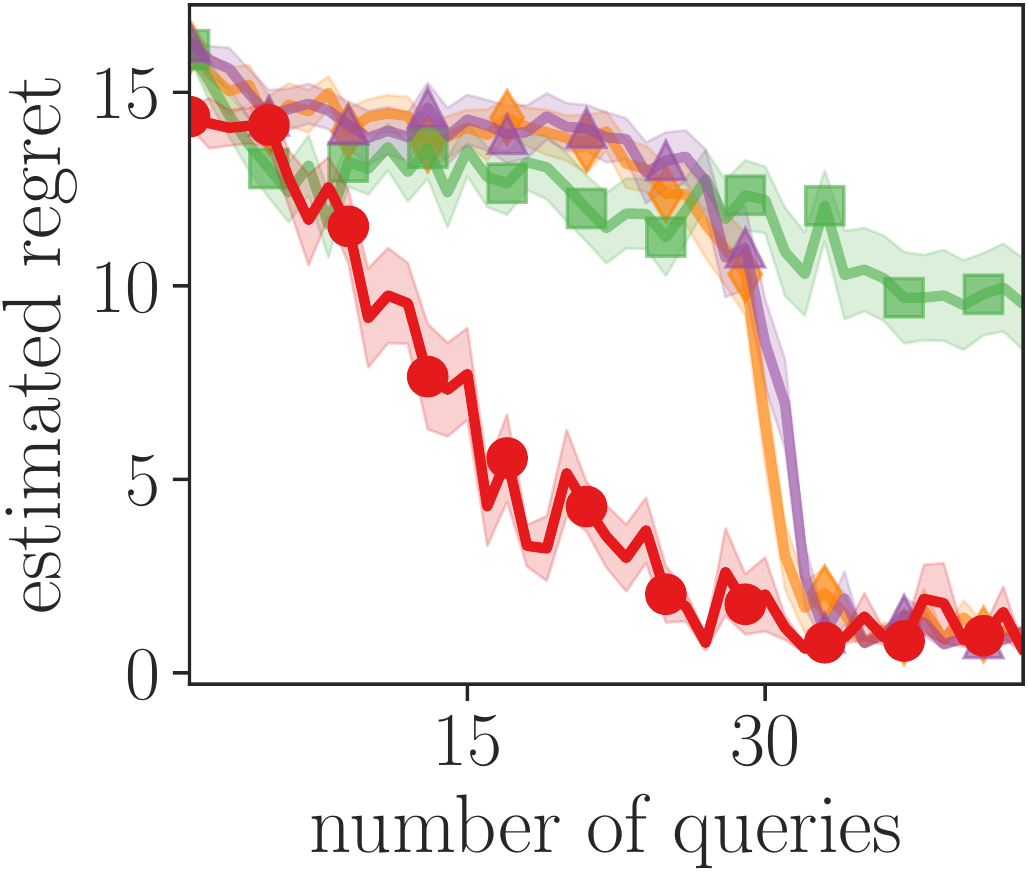}
        \includegraphics[width=0.35\linewidth]{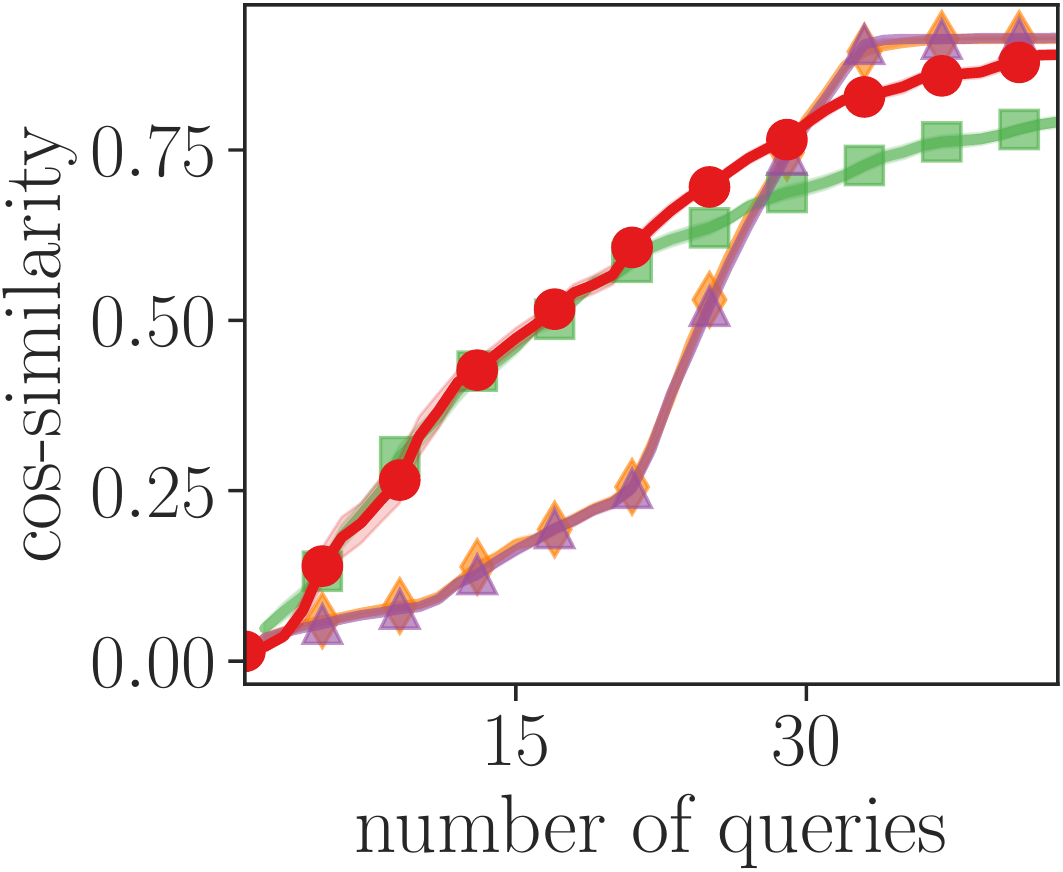}
        \caption{\emph{Swimmer-Corridor} (evaluation of trajectory clips)}\label{fig:result_mujoco}
    \end{subfigure}
\end{minipage}
\hspace*{0.5em}
\begin{minipage}{0.38\textwidth}\vspace*{1em}
\caption{\small Results in the (\subref{fig:result_driver}) \emph{Driver} and (\subref{fig:result_mujoco}) \emph{Swimmer-Corridor} environments (shown on the left).
We show the regret of a policy trained on the reward model compared to a policy trained on the true reward function as a function of the number of queries (middle plot).
We  report the cosine similarity between the learned and the true reward function (right plot). The plots show mean and standard error over 30 random seeds. \AlgoNameShort finds significantly better policies, while not necessarily learning an overall more accurate model of the reward function. A similar plot for \emph{Ant-Corridor} can be found in \Cref{app:additional_results}.}\label{fig:return_plots_mujoco}
\end{minipage}
\end{figure}

\subsection{Can \AlgoNameShort with a GP model be scaled to bigger environments?}\label{sec:experiments_gp_mujoco}

\begin{wraptable}{R}{0.4\textwidth}\centering\vspace*{-1em}
\begin{adjustbox}{max width=\linewidth}
    \begin{tabular}{ccc}
      \toprule
        & \specialcell{\emph{Swimmer-}\\\emph{Corridor}} & \specialcell{\emph{Ant-}\\\emph{Corridor}} \\
        \midrule
        Uniform Sampling & \result{11.8}{0.9} & \result{15}{1} \\
        IGR & \result{13.3}{0.6} & \result{17.2}{0.5} \\
        EIR & \result{12.0}{0.8} & \result{17.5}{0.8} \\
        \AlgoNameShort ($20$ updates) & \hresult{2.4}{0.8} & \hresult{2.2}{0.8} \\
        \AlgoNameShort ($4$ updates) & \result{2.8}{0.7} & \result{5}{1} \\
        \AlgoNameShort ($2$ updates) & \result{5.1}{0.6} & \result{8}{1} \\
        \AlgoNameShort ($1$ update) & \result{7.6}{0.8} & \result{12}{1} \\
      \bottomrule
    \end{tabular}
\end{adjustbox}
\caption{\small Results comparing \AlgoNameShort for different update frequencies of the candidate policies in the \emph{Corridor} environments. The table shows the estimated regret of a policy trained using $20$ queries about the reward function.}\label{tab:results_batch_size}
\vspace*{-1em}
\end{wraptable}

To demonstrate that GP-based \AlgoNameShort scales to larger environments, we use the  MuJoCo simulator \citep{todorov2012mujoco}, which provides challenging environments commonly used as benchmarks for RL.
However, its standard locomotion tasks are very easy to learn for a GP model because the reward is directly proportional to the agent's velocity in x-direction. Instead, we propose a task where the reward function is harder to learn.

\paragraph{Setup.}\looseness=-1
In our \emph{Corridor} environments (\Cref{fig:result_mujoco}), a robot (\emph{Swimmer}, or \emph{Ant}) has to move forward and stop at a goal position.
The simulated expert rates trajectory clips according to a reward function that is proportional to the velocity in the direction of the goal. This reward function is linear in a set of features of the state, as described in \Cref{app:mujoco_corridor}.
We use \emph{augmented random search} \citep{mania2018simple} as RL algorithm. For the \emph{Swimmer-Corridor} we learn a linear policy, and for the \emph{Ant-Corridor} we learn a hierarchical policy on top of pre-trained policies moving in four different directions.
To generate candidate queries, we use a fixed, noisy exploration policy that moves along the whole corridor. Unfortunately, EPD is too expensive to evaluate in this environment and MR is not suited to this kind of queries.

\paragraph{Results.}
\Cref{fig:result_mujoco} shows that \AlgoNameShort needs significantly fewer queries to find a good policy than any of the baselines. \AlgoNameShort adapts its queries to the policies that the current reward model induces: it initially samples clips in which the robot moves close to its starting position and shifts its focus to other regions as the reward model improves, and the learned policy starts to move.
In contrast, the baselines make queries in the whole reachable space similarly often, and, therefore, waste queries in regions that are not directly relevant for improving the policy.

\looseness -1 The computationally most expensive part of this implementation of \AlgoNameShort is updating the candidate policies in each iteration. Updating them less often reduces the computational cost at the expense of potentially reducing the sample efficiency. \Cref{tab:results_batch_size} studies this trade-off and shows that \AlgoNameShort outperforms the baselines even when the policies are updated only once at the beginning of training.
In this extreme case, we reduce IDRL's runtime from about $40$ hours to about $20$ hours in \emph{Swimmer-Corridor}, and from about $40$ hours to about $10$ hours in \emph{Ant-Corridor}. This shows the benefits are larger when solving the RL problem is more expensive. Nonetheless, the baseline algorithms are still faster, and run for only $2-3$ hours. This is because they do not require the additional inference steps necessary to optimize \Cref{eq:dig_state_selection}.  These results indicate that \AlgoNameShort using full Thompson sampling to generate candidate policies can trade-off computational cost and sample efficiency, which allows it to be applied to large environments.

\subsection{Can \AlgoNameShort be scaled to a Deep RL setting?}
\label{sec:experiments_deep_rl_mujoco}

\begin{wrapfigure}{R}{0.45\textwidth}\centering\vspace*{-2em}
\includegraphics[width=0.9\linewidth]{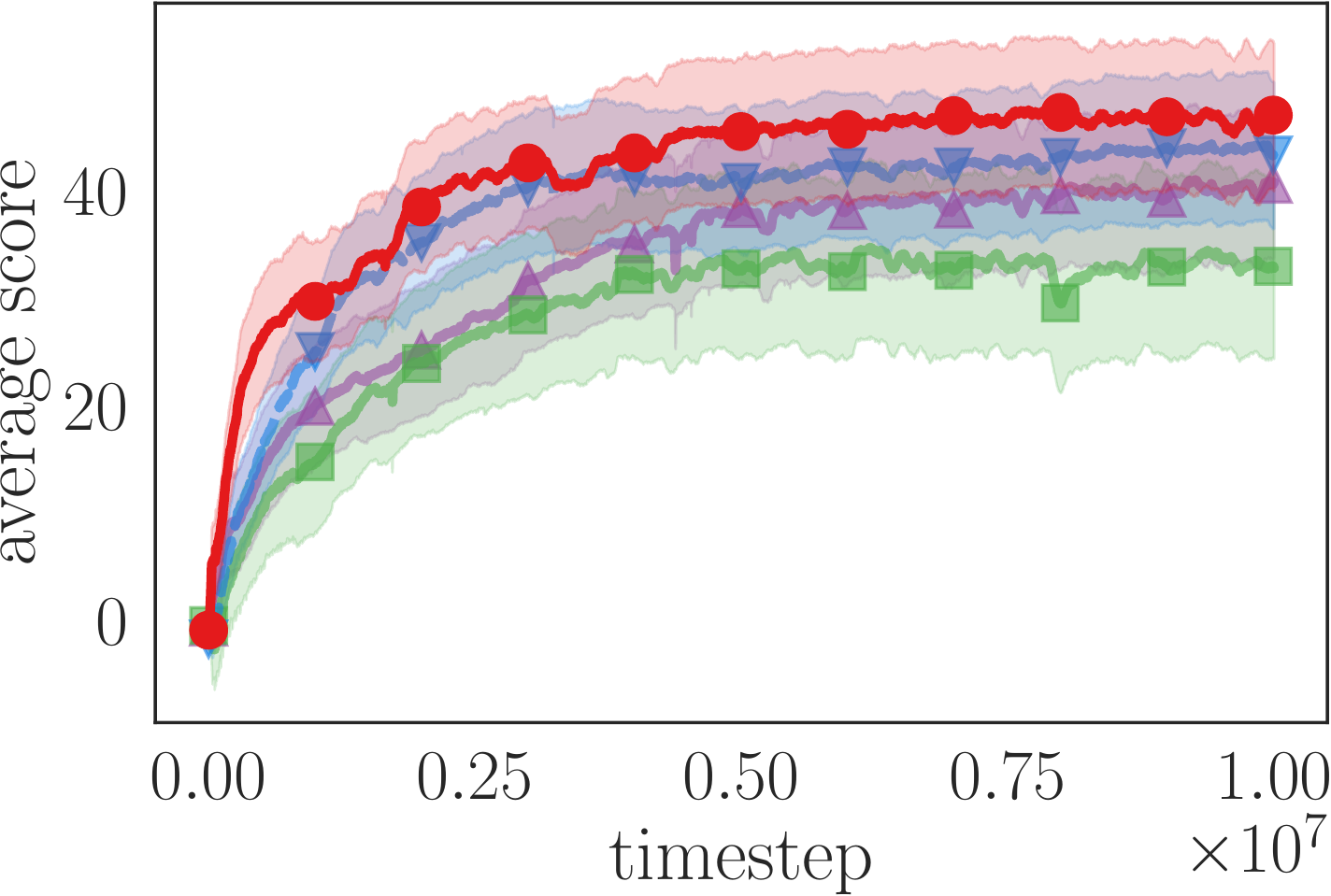}
\caption{\small Normalized score of policies learned from $1400$ (synthetic) comparisons of clips of the agent's behavior, averaged over all MuJoCo environments (higher is better). We show the mean and standard error of the score averaged over $5$ random seeds per environment. The plot compares IDRL (\legendIDRL) to IGR (\legendIGGP) and uniform sampling (\legendUniform), as well as an ablation of \AlgoNameShort that does not use the candidate policies to generate additional candidate queries (\legendIDRLABLATION). EPD is too expensive and MR is not suited to this kind of queries.}\label{fig:results_deep_rl}
\vspace*{-1em}
\end{wrapfigure}

Finally, we consider the Deep RL implementation of \AlgoNameShort from \Cref{sec:idrl_deep_rl}, using the Soft Actor-Critic algorithm \citep[SAC;][]{haarnoja2018soft}. We test it on standard MuJoCo locomotion tasks, which are harder to learn with a DNN than with a GP model because the former encodes less prior information.

\paragraph{Setup.} We consider a suite of standard tasks in MuJoCo implemented in OpenAI Gym \citep{openaiGym}: \emph{HalfCheetah-v3}, \emph{Walker2d-v3}, \emph{Hopper-v3}, \emph{Ant-v3}, \emph{Swimmer-v3}, \emph{InvertedPendulum-v2}, \emph{InvertedDoublePendulum-v2}, \emph{Reacher-v2}. Similar to \citet{christiano2017deep}, we modify some environments to remove the termination conditions. Our environments differ slightly from \citeauthor{christiano2017deep}, for the details see \Cref{app:comparison_to_christiano}. Our evaluation metric is a normalized score, averaged over all environments. A score of $0$ corresponds to a random policy and a score of $100$ is the performance of a policy trained on the true reward function. We provide results for the individual environments in \Cref{app:additional_results}. Since \AlgoNameShort tracks the candidate policies, it generates the candidate queries rolling out the currently optimal policy \textit{and} the candidate policies. However, the baselines do not have access to the candidate policies, and therefore consider a smaller set of potential queries. For a fair comparison, we perform an ablation where \AlgoNameShort does not consider the candidate policies to generate candidate queries. Since \AlgoNameShort maintains $3$ (additional) candidate policies, it is roughly $4$ times slower (about $80$ hours on a single GPU) than the baselines (about $20$ hours on a single GPU). 

\paragraph{Results.}
\Cref{fig:results_deep_rl} shows that \AlgoNameShort on average learns good policies significantly faster than the baselines. The individual results in each environment (in \Cref{app:additional_results}) are more nuanced. \AlgoNameShort clearly outperforms the baselines in some environments (e.g., \emph{Hopper-v3}), performs comparable in other environments (e.g., \emph{Walker2d-v3}), and performs worse than uniform sampling in a few environments (e.g., \emph{HalfCheetah-v3}). Also while mostly using the candidate policy rollouts improves the performance of IDRL, this is not always the case (e.g., in \emph{Swimmer-v3} the ablation performs better). This indicates that much of the variance might be caused by which queries are considered, which could be improved by using other exploration strategies than the candidate policies to generate candidate queries. Crucially, these experiments demonstrate that \AlgoNameShort is scalable to high-dimensional, complex tasks, while still improving sample efficiency over existing methods for such tasks.

\section{Conclusion}\label{sec:conclusion}

We studied the problem of actively learning reward function models using as few expert queries as possible. We introduced \emph{\AlgoNameLong} (\AlgoNameShort), a novel information-theoretic algorithm that focuses on learning a good policy rather than attaining a low approximation error of the reward and that, differently from most prior methods, works with multiple types of feedback. We show it needs significantly fewer queries than prior methods and that it scales to complex environments.

\paragraph{Limitations and future work.} The main practical limitation of $\AlgoNameShort$ is its computational cost. We demonstrated how to scale $\AlgoNameShort$ to complex environments, increasing the runtime by only a constant factor. While \AlgoNameShort is still more demanding than most existing algorithms, it is preferable in situations where better sample efficiency is more important than low computational cost.

Our problem setup also has some conceptual limitations. We assume that interactions with the environment are cheap, which is not the case in many applications. Future work could aim to achieve low sample complexity in terms of environment interactions as well as reward queries. Moreover, we assume that the goal of RL is to learn a good policy in a single environment, which does not consider the problem of generalizing to other environment. In fact, being designed to learn a good policy in a single environment, \AlgoNameShort might not be best for learning a reward model that generalizes well. To address this, future versions of IDRL could aim to learn a reward model that leads to good policies over a \emph{distribution} of environments instead of a single environment.

Overall, we consider \AlgoNameShort an addition to the set of existing active reward learning algorithms rather than a replacement of existing methods.

\paragraph{Broader impact.} \AlgoNameShort improves the sample efficiency of learning reward models, which is a step towards making RL a viable solution for real-world problems. RL systems can be used in various ways, and they could cause risks from malicious actors \citep{brundage2018malicious}. However, overall, learning reward models is likely to help in making RL more robust and safe \citep{leike2018scalable}.

Improving the sample efficiency of learning reward models, is crucial for making RL more useful. By addressing this problem, \AlgoNameShort takes a step towards making RL a more viable solution for real-world problems.

\section*{Acknowledgements}
This research was supported through the Microsoft Swiss Joint Research Center.
We thank Johannes Kirschner and Jonas Rothfuss for valuable feedback on an earlier version of this paper, and Nils Wilde for valuable comments about the Maximum Regret approach to reward learning.

\newcommand{\ICML}{Proceedings of International Conference on Machine Learning (ICML)}
\newcommand{\RSS}{Proceedings of Robotics: Science and Systems (RSS)}
\newcommand{\NeurIPS}{Advances in Neural Information Processing Systems}
\newcommand{\IJCAI}{Proceedings of International Joint Conferences on Artificial Intelligence}
\newcommand{\ICLR}{International Conference on Learning Representations (ICLR)}
\newcommand{\CoRL}{Conference on Robot Learning (CoRL)}
\newcommand{\UAI}{Uncertainty in Artificial Intelligence}

\bibliography{references}
\bibliographystyle{plainnat}

\clearpage
\appendix

\section{Proofs of propositions}\label{app:proofs}

In this section, we provide proofs of all results mentioned in the main paper. The results generally follow from well-known facts about Gaussian distributions and information theory.

\begin{restatable}{proposition}{ReturnDiffGaussian}\label{ReturnDiffGaussian}
If $\RfunBelief(s)$ is a GP, the difference in expected return between two fixed policies $\p, \p'$ follows a Gaussian distribution $\cN(\mu, \sigma^2)$ with
\begin{align*}
\mu &= \bbE[\ExpRetBelief{\p} - \ExpRetBelief{\p'}] = \VectorProd{\StateVisitDirection{\p}{\p'}}{\mu_{\hat{\vr}}}\\
\sigma^2 &= \Var[\ExpRetBelief{\p} - \ExpRetBelief{\p'}] = \StateVisitDirection{\p}{\p'}^T \cdot \Sigma_{\hat{\vr}} \cdot \StateVisitDirection{\p}{\p'}
\end{align*}
where $\StateVisitDirection{\p}{\p'} = \StateVisit{\p} - \StateVisit{\p'}$ is the difference between expected state-visitation frequencies of $\p$ and $\p'$ respectively, and $\mu_{\hat{\vr}}$ and $\Sigma_{\hat{\vr}}$ are the mean and covariance of the joint Gaussian distribution of the reward of all states $\p$ or $\p'$ visit.
\end{restatable}

\begin{proof}
If a random variable $X$ is Gaussian distributed $X \sim \cN(\vmu, \Sigma)$, $\vmu \in \R^n, \Sigma \in \R^{n \times n}$, then for $\va \in \R^n$, $\va^T X$ is also Gaussian distributed $\va^T X \sim \cN(\VectorProd{\va}{\vmu}, \va^T \Sigma \va)$ \citep[][Theorem 14.2]{wasserman2004all}.

We can directly apply this fact to $\hat{\vr} \sim \cN(\mu_{\hat{\vr}}, \Sigma_{\hat{\vr}})$ and  $\ExpRetBelief{\p} - \ExpRetBelief{\p'} = \VectorProd{\StateVisitDirection{\p}{\p'}}{\hat{\vr}}$, resulting in
\[
\ExpRetBelief{\p} - \ExpRetBelief{\p'} \sim \cN(\VectorProd{\StateVisitDirection{\p}{\p'}}{\mu_{\hat{\vr}}}, \StateVisitDirection{\p}{\p'}^T \Sigma_{\hat{\vr}} \StateVisitDirection{\p}{\p'}).
\]
\end{proof}

\IDRLGP*
\begin{proof}
If a random variable $X$ is Gaussian distributed $X \sim \cN(\mu, \sigma^2)$, then the entropy $\Entropy{X}$ is given by \citep[][Theorem 8.4.1]{cover2006elements}
\begin{align}
\Entropy{X} = \frac12 \log (2\pi\eup \sigma^2).
\end{align}
\Cref{ReturnDiffGaussian} shows that the conditional distribution of $\ExpRetBelief{\p} - \ExpRetBelief{\p'} | \Dataset$ is Gaussian, which implies both statements.

For the first statement, observe that the entropy of $\ExpRetBelief{\p} - \ExpRetBelief{\p'} | \Dataset$ is
\begin{align}
\Entropy{\ExpRetBelief{\p} - \ExpRetBelief{\p'} | \Dataset} = \frac12 \log (2\pi\eup \Var[\ExpRetBelief{\p} - \ExpRetBelief{\p'} | \Dataset]),
\end{align}
and that two policies that maximize the variance on the r.h.s. also maximize the entropy, because the logarithm is a monotonic function.

\newcommand{\PolicyDiff}{\hat{\Delta}_{\p_1,\p_2}}
\newcommand{\DatasetNew}{\Dataset \cup \{ (\Query, \EvidenceBelief) \}}
To see the second statement, let $\PolicyDiff = \ExpRetBelief{\p_1} - \ExpRetBelief{\p_2}$. Then
\begin{align*}
&\argmax_{\Query \in \CandStates} I(\PolicyDiff; (\Query, \EvidenceBelief) | \Dataset) \\
= &\argmax_{\Query \in \CandStates} \rbr{\ConditionalEntropy{\PolicyDiff}{\Dataset} - \ConditionalEntropy{\PolicyDiff}{\DatasetNew}} \\
= &\argmin_{\Query \in \CandStates} \ConditionalEntropy{\PolicyDiff}{\DatasetNew} \\
= &\argmin_{\Query \in \CandStates} \frac12 \log \rbr{2\pi\mathrm{e} \Var[\PolicyDiff | \DatasetNew]} \\
= &\argmin_{\Query \in \CandStates} \Var[\PolicyDiff | \DatasetNew].
\end{align*}
Here we wrote the information gain in terms of conditional entropies \citep[][Theorem 2.4.1]{cover2006elements}, and used that only one of the terms depends on $\Query$. This turns the maximization of information gain into a minimization of a conditional entropy.
As before, we can further simplify this to minimizing conditional variance by using the entropy of a Gaussian and the fact that the logarithm is a monotonic function.
\end{proof}

\GPLinearObservations*
\begin{proof}

Let $\Query = (\linstates, \linfactors)$ be a linear reward query, i.e., $\linstates = \{s_{1}, \dots, s_{N}\} \subseteq \sset$ is a set of states and $\linfactors = \{\linfac_{1}, \dots, \linfac_{N}\}$ a set of linear weights, and $\Evidence = \sum_{j=1}^N \linfac_{j} \Rfun(s_{j})$ the corresponding observation.

Let $S^* = \{s_1^*, \dots, s_n^*\} \subseteq \sset$ be a set of states for which we want to compute the posterior belief. We show that
\[
P(\RfunBelief(s_1^*), \dots, \RfunBelief(s_n^*) | (\Query, \Evidence)) \sim \cN(\vmu^*_{\Query}, \Sigma^*_{\Query})
\]
for some $\vmu^*_{\Query}$ and $\Sigma^*_{\Query}$. Because this holds for any set of states $S^*$, it shows that the posterior reward model is a GP.

We define the following vector notation:
\begin{align*}
    \vc &= ( \linfac_{1}, \dots, \linfac_{N} )^T \in \R^{N} \\
    \hat{\vr} &= ( \RfunBelief(s_{1}), \dots, \RfunBelief(s_{N}) )^T \in \R^{N} \\
    \hat{\vr}^* &= ( \RfunBelief(s_{1}^*), \dots, \RfunBelief(s_{n}^*) )^T \in \R^{n}
\end{align*}
such that $\Evidence = \VectorProd{\vc}{\hat{\vr}}$.

The prior distribution of $\hat{\vr}$ is Gaussian, i.e.,
\begin{align*}
P(\hat{\vr} | \linstates) \sim \cN(\vmu, \Sigma),
\end{align*}
with mean $\vmu$ and covariance $\Sigma$.

Because $\EvidenceBelief$ is a linear function of $\hat{\vr}$ plus Gaussian noise, the prior distribution of $\EvidenceBelief$ is also Gaussian \citep[][Theorem 14.2]{wasserman2004all}:
\begin{align*}
P(\EvidenceBelief | \linstates, \linfactors) \sim \cN(\VectorProd{\vc}{\vmu}, \vc^T \Sigma \vc + \sigma_n^2 I).
\end{align*}

Further, $\hat{\vr}$ and $\hat{\vr}^*$ are jointly Gaussian distributed:
\begin{align*}
P\left (\left[ \begin{smallmatrix} \hat{\vr} \\ \hat{\vr}^* \end{smallmatrix} \right] | \linstates, S^* \right) \sim \cN \left( \left[ \begin{smallmatrix} \vmu \\ \vmu^* \end{smallmatrix} \right],  \left[ \begin{smallmatrix}
     \Sigma & \Sigma^*  \\
     (\Sigma^*)^T & \Sigma^{**}
\end{smallmatrix} \right] \right)
\end{align*}
where $\vmu^*$ is the mean of $\hat{\vr}^*$, and $\Sigma^{*} = \Cov[\hat{\vr},\hat{\vr}^*]$ and $\Sigma^{**} = \Cov[\hat{\vr}^*,\hat{\vr}^*]$ denote the components of the joint covariance matrix.

Hence, $\hat{\vr}^*$ and $\EvidenceBelief$ are also jointly Gaussian distributed:
\begin{align*}
P\left (\left[ \begin{smallmatrix} \hat{\vr}^* \\ \EvidenceBelief \end{smallmatrix} \right] | \linstates, \linfactors, S^* \right) \sim \cN \left( \left[ \begin{smallmatrix} \vmu^* \\ \VectorProd{\vc}{\vmu} \end{smallmatrix} \right],  \left[ \begin{smallmatrix}
     \Sigma^{**} & (\Sigma^*)^T \vc \\
     \vc^T \Sigma^* & \vc^T \Sigma \vc + \sigma_n^2 I
\end{smallmatrix} \right] \right)
\end{align*}
where we used the linearity of the covariance function to find the covariance matrix:
\begin{align*}
\Cov[\EvidenceBelief, \hat{\vr}^*] &= \Cov[\VectorProd{\vc}{\hat{\vr}}, \hat{\vr}^*] \\
&= \vc^T \Cov[\hat{\vr}, \hat{\vr}^*] = \vc^T \Sigma^* \\
\Cov[\hat{\vr}^*, \EvidenceBelief] &= (\Sigma^*)^T \vc
\end{align*}

Finally, we can use standard results on conditioning Gaussian distributions \citep[cf.][Chapter A.2]{rasmussen2006gaussian} to find that the conditional distribution is still Gaussian:
\begin{align*}
P(\hat{\vr}^* | (\Query, \Evidence)) = P(\hat{\vr}^* | \Evidence, \linstates, \linfactors, S^*) \sim \cN(\vmu^*_{\Query}, \Sigma^*_{\Query})
\end{align*}
with
\begin{align*}
\vmu^*_{\Query} &= \vmu^* + ((\Sigma^*)^T \vc) \inv{(\vc^T \Sigma \vc + \sigma_n^2 I)} (\Evidence - \VectorProd{\vc}{\vmu}) \\
\Sigma^*_{\Query} &= \Sigma^{**} - ((\Sigma^*)^T \vc) \inv{(\vc^T \Sigma \vc + \sigma_n^2 I)} (\vc^T \Sigma_*).
\end{align*}

When conditioning the distribution, we replaced our belief about the observation $\EvidenceBelief$ with its actual realization $\Evidence$.
\end{proof}

\section{Linear reward queries}\label{app:linear_reward_queries}

We consider linear reward queries for our implementation of \AlgoNameShort with GP models, which makes all computations analytically tractable (cf. \Cref{sec:idrl_gp}). Linear reward queries can be used to model many different observation types that are typical in practical settings. In this section, we recall the definition of linear reward queries, and then present a few particularly common types of linear reward queries, which include all query types used in our empirical evaluation of \AlgoNameShort with a GP model.

\DefLinearObservations*

\paragraphsmall{Single state rewards.} If $N=1$, a query consists of a single state $\Query_i = s_i \in \sset$, for which the expert provides a noisy reward $\Evidence = \Rfun(s_i) + \varepsilon$.

\paragraphsmall{Return of trajectories.}
For $N > 1$ and all $c_{i} = 1$, the agent observes the sum of rewards of multiple states. The set $\linstates$ could, e.g., contain the states in a trajectory or a sub-sequence of it. Then, the queries ask about the return, i.e., sum of rewards, of this sequence or states.

\paragraphsmall{Comparisons of states and trajectories.}
We can model a comparison of the reward in states $s_a$ and $s_b$ by defining $S = \{s_a, s_b\}$ and defining $\linfactors = \{1, -1\}$. Then the agent might observe $\Evidence = \Rfun(s_a) - \Rfun(s_b)$. In practice, comparison queries usually result in binary feedback, i.e., the expert states that either $s_a$ or $s_b$ is prefered. We can model this, e.g., with a Bernoulli distribution $P(\Evidence = 1) = (1+\Rfun(s_a)-\Rfun(s_b))/2$, if all rewards are between $0$ and $1$. The observations from this distribution have expectation $\Rfun(s_a) - \Rfun(s_b)$ and the noise model is subgaussian, which we can approximate with a Gaussian noise distribution \citep[cf.][]{kirschner2020information}. Hence, we can model such comparison queries as linear reward queries. We can model comparisons between two sets of states, e.g., between two trajectories, analogously. Other observation models for comparisons have been proposed in the literature, such as softmax \citep{dorsa2017active}, probit \citep{biyik2020active} or Bernoulli distributions with constant probability \citep{wilde2020active}. While we focus on linear observations, \AlgoNameShort could be extended to these alternatives by using approximate inference to update the reward model, similar to \citet{biyik2020active}.
\section{Connection to multi-armed bandits}\label{app:connection_TLB}

In the main paper, we motivated \AlgoNameShort from information-theoretic considerations. However, there are close connections to related algorithms in multi-armed bandits (MAB) that can serve as additional motivation.

Efficient exploration is extensively studied in MAB problems \citep{bubeck2012regret}. Recent work successfully uses decision criteria based on information gain in various MAB problems \citep{russo2014learning}. However, our setting is no standard MAB problem, because we do not directly observe the quantity we are optimizing for, i.e., the return of a policy.

In this section, we discuss two settings that are more closely related to our setting: the \emph{linear partial monitoring} problem and \emph{transductive linear bandits}.

\subsection{Linear partial monitoring}
Our setting is closely related to partial monitoring problems, which generalize the standard MAB to cases where the agent's observations provide only indirect information about the reward \citep{rustichini1999minimizing}.
For tabular MDPs, our setting can be interpreted as a linear partial monitoring problem. Let $\vr$ be a vector of all rewards in a tabular MDP. We consider observations that are a linear function of the rewards $\Rfun(s_i) = \VectorProd{\vc}{\vr}$, and the optimization target is also a
linear function of the reward vector $\ExpRet{\p} = \VectorProd{\StateVisit{\p}}{\vr}$. \citet{kirschner2020information} analyze linear partial monitoring problems and propose an information gain based criterion for selecting observations. One criterion they propose to measure information gain, called \emph{directed information gain}, is equivalent to our information gain criterion \citep[App. B.2]{kirschner2020information}. However, they consider cumulative regret minimization, and, therefore, their algorithm has to trade-off the information gain of an observation with its expected regret. In our setting, minimizing cumulative regret would correspond to maximizing $\sum_{t=1}^T \ExpRet{\bar{\p}_t}$, where $\bar{\p}_t$ is the policy that \AlgoNameShort returns if it is stopped after $t$ iterations. Instead, we just evaluate the final policy and aim to maximize $\ExpRet{\bar{\p}_T}$. Consequently, our algorithm directly uses directed information gain as a selection criterion.

\subsection{Transductive linear bandits}
Our setting is a \emph{pure exploration} problem \citep{bubeck2009pure}: we only evaluate the performance of the final policy after a fixed budget of queries and not the intermediary policies. Our problem is closely related to pure exploration in \emph{transductive linear bandits} that consider maximizing a linear reward function in a set $\cZ$ by making queries in a potentially different set $\cX$ \citep{fiez2019sequential}. In fact, for a tabular MDP, our problem is a special case of the transductive linear bandit setting. Moreover, we can understand \AlgoNameShort as an adaptive version of the RAGE algorithm introduced by \citet{fiez2019sequential}.

To see the connection between both settings, let us first define the transductive linear bandit problem.

\begin{definition}[\citeauthor{fiez2019sequential}, \citeyear{fiez2019sequential}]
A \emph{transductive linear bandit problem} is defined by two sets $\cX \subset \R^d$ and $\cZ \subset \R^d$, where the goal is to find $\argmax_{\vz\in \cZ} \VectorProd{\vz}{\vtheta^\ast}$ for some hidden parameter vector $\vtheta^\ast \in \R^d$. However, instead of observing this objective directly, the learning agent interacts with the bandit at each time-step by selecting an arm $\vx \in \cX$ to play, and then observing $\VectorProd{\vx}{\vtheta^{\ast}} + \eta$ where $\eta$ is independent, zero-mean, subgaussian noise. The agent's goal is to find the maximum in $\cZ$ by making as few queries in $\cX$ as possible.
\end{definition}

\begin{restatable}{proposition}{ObsTLBSpecialCase}
For finite state and action spaces, a fixed set of candidate policies $\CandPol$ and a set of linear reward queries $\CandStates$, our reward learning problem is a transductive linear bandit problem with $\cZ = \{\StateVisit{\p} | \p \in \CandPol\} \subset \R^{|\sset|}$ and $\cX = \{\vc_i | i \in \{1, \dots, | \CandStates |\}\} \subset \R^{|\sset|}$ a set of linear observations.
\end{restatable}

To see this, note that our goal is to maximize $\ExpRet{\p} = \VectorProd{\StateVisit{\p}}{\vr}$ and we query linear combinations of rewards in each round $\VectorProd{\vc_{i_t}}{\vr}$. Here $i_t$ is the index of the query that the agent selects at time $t$, and $\vc_{i_t}$ is a vector of linear weights that defines query $\Query_{i_t}$.

To understand the connection between \AlgoNameShort and the RAGE algorithm proposed by \citet{fiez2019sequential}, it is helpful to assume that the reward function is a linear function of some features of the state, and to use a linear kernel for the GP model, which is equivalent to Bayesian linear regression.

Let $\vphi: \sset \to \R^d$ be a feature function, and the true reward function $r(s) = \VectorProd{\vphi(s)}{\vtheta^*}$. Similarly, we can define a feature vector for each query $\Query \in \CandStates$ and overload the notation $\vphi(\Query) = \sum_{i=1}^N \linfactors_i \vphi(s_i)$. Also, we can write the expected return of a policy as $\ExpRet{\p} = \VectorProd{\StateVisit{\p}_{\phi}}{\vtheta^*}$ with $\StateVisit{\p}_{\phi} = (\StateVisit{\p})^T \Phi$ and $\Phi = (\vphi(s_1), \dots, \vphi(s_{|\sset|}))^T$.

To solve the transductive linear bandit problem, the RAGE algorithm proceeds in multiple rounds, in each of which it follows an allocation rule
\begin{align}\label{eq:rage_selection2}
\design_{t}^{\ast} &= \argmin_{\design \in \simplex_{\cX}}\max_{\p_1, \p_2 \in \calzhat_{t}} \| \StateVisit{\p_1}_{\phi} - \StateVisit{\p_2}_{\phi} \|_{\inv{A_\lambda}}^2
\end{align}
where $A_\lambda = \sum_{\Query_i \in \CandStates} \lambda_i \vphi(\Query_i) \vphi(\Query_i)^T$, and where $\simplex_{\cX}$ is the probability simplex over candidate queries, so this rule would select query $\Query_i$ at round $t$ with probability $(\design_{t}^{\ast})_i$.
Additionally, RAGE keeps track of a set of plausibly optimal arms $\calzhat_t$, i.e., plausibly optimal policies in our case. RAGE ensures that the suboptimality gap of arms in this set shrinks exponentially as the algorithm proceeds.

The next proposition provides an alternative notation for \AlgoNameShort that shows a formal similarity to RAGE.
\begin{restatable}{proposition}{ObsTLBConnection}\label{ObsTLBConnection}
Assume we estimate $\hat{\vtheta}$ with Bayesian linear regression with noise variance $\sigma^2$, and prior $\hat{\vtheta} \sim \cN(0, \inv{\alpha} I)$ after collecting data
\[
\Dataset = ((\vphi(\Query_{i_1}), \Evidence_{i_1}), \dots, (\vphi(\Query_{i_{t-1}}), \Evidence_{i_{t-1}})).
\]
Also, assume an infinitely wide prior $\inv{\alpha} \to \infty$.

We can then write the maximization in the first step of \AlgoNameShort as
\begin{align*}
\argmax_{\p, \p' \in \CandPol} \ConditionalEntropy{\ExpRetBelief{\p} - \ExpRetBelief{\p'}}{\Dataset}
= \argmax_{\p, \p' \in \CandPol} \| \StateVisit{\p}_{\phi} - \StateVisit{\p'}_{\phi} \|_{\inv{A_\Dataset}}^2
\end{align*}
where $A_\Dataset = \sum_{\Query_i \in \CandStates} N_i \vphi(\Query_i) \vphi(\Query_i)^T$.

Furthermore, for a given pair of policies, $\p_1$ and $\p_2$, we can write the maximization in the second step of \AlgoNameShort as
\begin{align*}
&\argmax_{\Query \in \CandStates} I(\ExpRetBelief{\p_1} - \ExpRetBelief{\p_2}; (\Query, \EvidenceBelief) | \Dataset) \\
= &\argmin_{\Query \in \CandStates} \| \StateVisit{\p_1}_{\phi} - \StateVisit{\p_2}_{\phi} \|_{\inv{A_{\Dataset,\Query}}}^2
\end{align*}
where $A_{\Dataset,\Query} = \vphi(\Query)\vphi(\Query)^T + \sum_{\Query_i \in \CandStates} N_i \vphi(\Query_i) \vphi(\Query_i)^T$ and $N_i$ is the number of times $\Query_i$ occurs in $\Dataset$.
\end{restatable}

\begin{proof}
In the Bayesian linear regression setting \citep[cf.][Chapter 3.3]{bishop2006pattern} with prior weight distribution $\vw \sim \cN(\mathbf{0}, \inv{\alpha})$, the posterior weight distribution is a Gaussian with covariance matrix
\begin{align*}
\Sigma_{\theta} &= \inv{\rbr{\alpha I + \sigma^{-2} \sum_{\Query \in \Dataset} \vphi(\Query) \vphi(\Query)^T }} \\
&= \inv{\rbr{\alpha I + \sigma^{-2} A_\Dataset }} \\
A_\Dataset &= \sum_{\Query \in \Dataset} \vphi(\Query) \vphi(\Query)^T \\
&= \sum_{\Query_i \in \CandStates} N_i \vphi(\Query_i) \vphi(\Query_i)^T
\end{align*}
For an infinitely wide prior ($\inv{\alpha} \to \infty$): $\Sigma_{\theta} \to \sigma^2 \inv{A_\Dataset}$.

Using the linear mapping from $\hat{\vtheta}$ to the expected return of a policy $\ExpRetBelief{\p}$, the posterior variance of the difference in return between two policies is
\begin{align} \label{eq:varG1G2_expansion}
\Var[\ExpRetBelief{\p_1} - \ExpRetBelief{\p_2} | \Dataset] = \sigma^2 (\StateVisitDirection{\p_1}{\p_2}^{\phi})^T \inv{A_\Dataset} \StateVisitDirection{\p_1}{\p_2}^{\phi}
\end{align}
where $\StateVisitDirection{\p_1}{\p_2}^{\phi} = \StateVisit{\p_1}_{\phi} - \StateVisit{\p_2}_{\phi}$

The first part of the statement follows using \Cref{IDRLGP}:
\begin{align*}
&\argmax_{\p, \p' \in \CandPol} \ConditionalEntropy{\ExpRetBelief{\p} - \ExpRetBelief{\p'}}{\Dataset} \\
= &\argmax_{\p, \p' \in \CandPol} \Var[\ExpRetBelief{\p} - \ExpRetBelief{\p'} | \Dataset] \tag{ \Cref{IDRLGP}} \\
= &\argmax_{\p, \p' \in \CandPol} \sigma^2 (\StateVisitDirection{\p}{\p'}^{\phi})^T \inv{A_\Dataset} \StateVisitDirection{\p}{\p'}^{\phi} \tag{ \Cref{eq:varG1G2_expansion}} \\
= &\argmax_{\p, \p' \in \CandPol} (\StateVisitDirection{\p}{\p'}^{\phi})^T \inv{A_\Dataset} \StateVisitDirection{\p}{\p'}^{\phi} \\
= &\argmax_{\p, \p' \in \CandPol} \| \StateVisitDirection{\p}{\p'}^{\phi} \|_{\inv{A_\Dataset}} \\
= &\argmax_{\p, \p' \in \CandPol} \| \StateVisit{\p}_{\phi} - \StateVisit{\p'}_{\phi} \|_{\inv{A_\Dataset}}
\end{align*}

After defining
\begin{align*}
A_{\Dataset,\Query} &= A_{\Dataset \cup \{(\Query, \Evidence)\}} \\
&= \vphi(\Query)\vphi(\Query)^T + \sum_{\Query_i \in \CandStates} N_i \vphi(\Query_i) \vphi(\Query_i)^T,
\end{align*}
the second part of the statement follows analogously to the first one after applying \Cref{IDRLGP}:
\begin{align*}
&\argmax_{\Query \in \CandStates} I(\ExpRetBelief{\p_1} - \ExpRetBelief{\p_2}; (\Query, \EvidenceBelief) | \Dataset) \\
= &\argmin_{\Query \in \CandStates} \Var[\ExpRetBelief{\p_1} - \ExpRetBelief{\p_2} | \Dataset \cup \{\Query, \Evidence\}] \tag{Prop. \ref{IDRLGP}} \\
= &\argmin_{\Query \in \CandStates} \sigma^2 (\StateVisitDirection{\p_1}{\p_2}^{\phi})^T \inv{A_{\Dataset,\Query}} \StateVisitDirection{\p_1}{\p_2}^{\phi} \tag{\Cref{eq:varG1G2_expansion}}\\
= &\argmin_{\Query \in \CandStates} (\StateVisitDirection{\p_1}{\p_2}^{\phi})^T \inv{A_{\Dataset,\Query}} \StateVisitDirection{\p_1}{\p_2}^{\phi} \\
= &\argmin_{\Query \in \CandStates} \| \StateVisitDirection{\p_1}{\p_2}^{\phi} \|_{\inv{A_{\Dataset,\Query}}} \\
= &\argmin_{\Query \in \CandStates} \| \StateVisit{\p_1}_{\phi} - \StateVisit{\p_2}_{\phi} \|_{\inv{A_{\Dataset,\Query}}}
\end{align*}
\end{proof}

Comparing this proposition with \cref{eq:rage_selection2} shows a formal similarity between both algorithms.  In particular, we can understand \AlgoNameShort as a version of \cref{eq:rage_selection2} that adapts to the data seen so far and selects the next observation that would minimize this objective. Instead of the matrix $A_\lambda$ that is induced by the allocation rule $\lambda$, \AlgoNameShort computes the variances using $A_\Dataset$, i.e., based on data observed in the past, and using $A_{\Dataset,\Query}$, i.e., evaluating the effect of an additional observation. Additionally, \AlgoNameShort performs two separate optimizations which one can consider as an approximation to the min-max problem in \cref{eq:rage_selection2}, which would be infeasible to evaluate in our setting. Similar to RAGE, \AlgoNameShort keeps track of a set of plausibly optimal policies. However, \AlgoNameShort uses Thompson sampling, while RAGE uses suboptimality gaps to build this set.

\section{Implementation details of \AlgoNameShort with GP reward models}
\label{app:implementation_details_gp}
\label{app:candidate_policies}
\label{app:baselines_details}

In this section we describe our implementation of \AlgoNameShort, the baselines we compare to, and our environments in more detail. For the choice of hyperparameters and additional implementation details we refer to the code of our experiments.

\subsection{Thompson sampling}

For some environments, a set of potentially optimal policies $\CandPol$ might be available. In other cases, we use Thompson sampling (TS) to generate $\CandPol$. Our TS approach is shown in \Cref{alg:ts}. We select $N$ policies by sampling reward functions from the posterior belief of the reward model and then finding optimal policies for them using some RL algorithm ($\mathtt{RL}$ in the pseudocode).

The set of candidate policies $\CandPol$ should be updated regularly during \AlgoNameShort to reflect the current posterior belief on optimal policies. In our experiments, we use $N=5$ policies, and update them in each iteration, if not stated differently in the text.

Depending on the environment, we use different RL algorithms. 
For \emph{Chain}, \emph{Junction}, and \emph{Gridworld} environments, we use an exact solver (using linear programming, or a lookup table of all deterministic policies).
For the \emph{Driver} environment, we use the L-BFGS-B solver provided by \citet{dorsa2017active}. However, we combine it with a lookup table of pre-computed policies to reduce noise, as suggested by \citet{wilde2020active}. Whenever the lookup table contains a policy that is better than the one returned by the solver, we use the policy from the table instead.
In the \emph{MuJoCo} environment, we use \emph{augmented random search} with linear policies \citep{mania2018simple} as $\mathtt{RL}$.

\begin{algorithm}[t]
\begin{algorithmic}
\State $\CandPol \gets \{\}$
\For{$i \in \{1, \dots, N\}$}
    \State sample reward function $\Rfun_i \sim P(\RfunBelief | \Dataset)$
    \State $\pi_i^* \gets \mathtt{RL}(\Rfun_i)$
    \State $\Pi_c \gets \Pi_c \cup \{\pi_i^*\}$
\EndFor
\State \textbf{return} $\Pi_c$
\end{algorithmic}
\caption{Thompson sampling for creating a set of candidate policies $\CandPol$.}\label{alg:ts}
\end{algorithm}

\subsection{Details on the baselines}

\begin{algorithm}[t]
\caption{Generic reward learning algorithm using an acquisition function $u(\Query, \Dataset)$. Our baselines use information gain and expected improvement for $u$. Uniform sampling samples $\Query^*$ uniformly from $\CandStates$ instead of \cref{line:baselines:query}.}
\label{alg:baselines}
\begin{algorithmic}[1]
  \State $\Dataset \gets \{\}$
  \State Initialize reward model with prior distribution $P(\RfunBelief)$
  \While{not converged}
     \State Select a query:
     \State \hspace{\algorithmicindent}$\Query^* \in \argmax_{\Query \in \CandStates} u(\Query, \Dataset)$ \label{line:baselines:query}
     \State Query $\Query^*$ and update reward model:
     \State \hspace{\algorithmicindent}$\Evidence^* \gets$ Response to query $\Query^*$
     \State \hspace{\algorithmicindent}$P(\RfunBelief | \Dataset \cup \{(\Query^*, \Evidence^*)\} ) \propto P(\Evidence^* | \RfunBelief, \Dataset, \Query^*) P(\RfunBelief | \Dataset)$
     \State \hspace{\algorithmicindent}$\Dataset \gets \Dataset \cup \{(\Query^*, \Evidence^*)\}$
  \EndWhile
  \State $\MeanRfun \gets$ mean estimate of the reward model
  \State $\MeanOptPol \gets \mathtt{RL}(\MeanRfun)$
  \State \textbf{return} $\MeanOptPol$
\end{algorithmic}
\end{algorithm}

In this section, we discuss the baselines in more detail. In \Cref{alg:baselines} we present pseudocode for the general reward learning algorithm that all of our baselines implement. They only differ in the choice of acquisition function in \cref{line:baselines:query}. In the following, we discuss the different choices.

\subsubsection{Uniform sampling}
The uniform sampling baseline runs \Cref{alg:baselines} with $\Query^*$ sampled uniformly from $\CandStates$ instead of \cref{line:baselines:query}.

\subsubsection{Information gain on the reward}
Another baseline uses information gain on the reward as acquisition function for \Cref{alg:baselines}, that is $u(\Query, \Dataset) = \ConditionalInformationGain{(\Query, \EvidenceBelief)}{\RfunBelief}{\Dataset}$. Note that we can write this information gain in terms of conditional entropies \citep[][Theorem 2.4.1]{cover2006elements}:
\begin{align}
    \ConditionalInformationGain{(\Query, \EvidenceBelief)}{\RfunBelief}{\Dataset}
    = \ConditionalEntropy{\EvidenceBelief}{\Dataset, \Query} - \ConditionalEntropy{\EvidenceBelief}{\RfunBelief, \Dataset, \Query}. \label{eq:ig_entropies}
\end{align}
If we assume that $q$ is a linear reward query, it is described by a set of states $\linstates = \{s_1, \dots, s_n\}$ and a set of linear weights $\linfactors = \{c_1, \dots, c_n\}$, such that $\EvidenceBelief = \sum_{i=1}^n \linfac_i \RfunBelief(s_i) + \varepsilon$. Then, the second term of \cref{eq:ig_entropies} is constant because $\RfunBelief$ contains all information about $\EvidenceBelief$. Further, the distribution $P(\EvidenceBelief | \cD, \Query)$ is Gaussian, and its entropy is \citep[][Theorem 8.4.1]{cover2006elements}:
\begin{align*}
    \ConditionalEntropy{\EvidenceBelief}{\Dataset, \Query}
    = \frac12 \log \rbr{2\pi\mathrm{e} \Var[\EvidenceBelief | \Dataset, \Query]}.
\end{align*}
Because the logarithm is a monotonic function, we can show, analogously to \Cref{IDRLGP}, that
\begin{align*}
    \argmax_{\Query\in\CandStates} \ConditionalInformationGain{(\Query, \EvidenceBelief)}{\RfunBelief}{\Dataset}
    = \argmax_{\Query\in\CandStates}  \Var[\EvidenceBelief | \Dataset, \Query].
\end{align*}
Hence, for a GP reward model with linear reward observations, using $u(\Query, \Dataset) = \ConditionalInformationGain{(\Query, \EvidenceBelief)}{\RfunBelief}{\Dataset}$ is equivalent to using $u(s, \Dataset) = \Var[\EvidenceBelief | \Dataset, \Query]$, which is what we do in practice.

\subsubsection{Expected improvement}

We define the \emph{expected improvement} (EI) acquisition function as:
\begin{align*}
    u(s, \Dataset) &= L \cdot \Phi (M) + \Var[\EvidenceBelief | \Dataset, \Query] \cdot \rho(M), \\
    M &= \frac{L}{\Var[\EvidenceBelief | \Dataset, \Query]}, \\
    L &= \bbE[\EvidenceBelief | \Dataset, \Query] - \Evidence_{\text{max}} - \xi,
\end{align*}
where $\rho$ and $\Phi$ are the probability density function, and the cumulative density function of the standard normal distribution, respectively. $\Evidence_{\text{max}}$ is the highest observation made so far. If $\Var[\EvidenceBelief | \Dataset, \Query] = 0$ we define $u(s, \Dataset) = 0$. $\xi$ is a hyperparameter, that we set to $0.001$.
EI quantifies how much higher a new observation is expected to be than the highest observation made so far. In our setting, EI is only applicable if observations are numerically, e.g., if the observations are rewards of individual states or trajectories. In particular, we cannot use EI if the queries are comparisons of states or trajectories.

\subsubsection{Expected policy divergence}

\begin{algorithm}[t]
\begin{algorithmic}
\State $\MeanRfun \gets$ mean estimate of the model $P(\RfunBelief | \Dataset)$
\State $\tilde{\p} \gets \mathtt{RL}(\MeanRfun)$
\State $\tilde{\Evidence} \gets \bbE[\EvidenceBelief | \Dataset, \Query] + \Var[\EvidenceBelief | \Dataset, \Query]$
\State $P(\RfunBelief | \Dataset \cup \{(\Query, \tilde{\Evidence})\} ) \propto P(\tilde{\Evidence} | \RfunBelief, \Dataset, \Query) P(\RfunBelief | \Dataset)$
\State $\MeanRfun^* \gets$ mean estimate of the model $P(\RfunBelief | \Dataset \cup \{(\Query, \tilde{\Evidence})\} )$
\State $\p^* \gets \mathtt{RL}(\bar{\Rfun}^*)$
\State $u(\Query, \Dataset) \gets d(\tilde{\p}, \p^*)$
\State \textbf{return} $u(\Query, \Dataset)$
\end{algorithmic}
\caption{The expected policy divergence (EPD) acquisition function, introduced by \citet{daniel2015active}, adapted to our setting. The algorithm computes EPD for a given query $\Query$.}\label{alg:epd}
\end{algorithm}

\citet{daniel2015active} introduce the expected policy divergence (EPD) acquisition function. EPD compares two policies: $\tilde{\p}$ is trained from a reward model conditioned on the current dataset $\Dataset$, and $\p^*$ estimates a policy trained from a reward model  conditioned on $\Dataset \cup \{(\Query, \Evidence)\}$. EPD aims to quantify the effect of making a query $\Query$ and observing response $\Evidence$ on the currently optimal policy. For each potential query it assumes an observation at an upper confidence bound  conditioned on the current model, and then selects observations that maximize some distance measure between policies $d(\tilde{\p}, \p^*)$. \Cref{alg:epd} shows how EPD is computed in our setting, which our implementation combines with \Cref{alg:baselines}. EPD requires solving an RL problem for each potential observation, which makes it computationally infeasible for bigger environments.

\citet{daniel2015active} introduce EPD using the KL divergence $\KL{\tilde{\p}}{\p^*}$ as distance measure $d(\tilde{\p}, \p^*)$.
However, in most of our experiments, the policies are deterministic, in which case the KL divergence is not well-defined. For tabular environments, we define $d$ to count the number of states in which the policies differ. For the \emph{Driver} environment we use an $\ell_2$-distance between the policy representations.

\subsubsection{Maximum regret}

\newcommand{\ExpRetUnder}[2]{\Ret_{#1}(#2)}

\citet{wilde2020active} introduce the \emph{Maximum Regret} (MR) acquisition function for selecting queries that compare two policies. They assume a set of candidate reward functions $\mathcal{R}_c = \{ \Rfun_1, \dots, \Rfun_n \}$ to be given and then consider a set of candidate policies $\CandPol = \{ \p_1, \dots, \p_n \}$, where each policy $\p_i$ is optimal for one of the reward functions $\Rfun_i$.

MR can be used to select comparison queries of the form $q=(\p_i, \p_j)$. In practice, we use MR for queries that compare trajectories sampled from $\p_i$ and $\p_j$.

MR aims to compare policies $\p_i$ and $\p_j$ that perform poorly when evaluated under each other's reward function $\Rfun_j$ and $\Rfun_i$ respectively. To formalize this in our notation, let us introduce the notation $\ExpRetUnder{\Rfun_i}{\p_j}$ to indicate the return of policy $\p_j$ evaluated using reward function $\Rfun_i$.  MR is defined as

\begin{align*}
u((\p_i, \p_j), \cD) = P(\Rfun_i | \cD) \cdot P(\Rfun_j | \cD) \cdot \rbr{R(\Rfun_i, \Rfun_j) + R(\Rfun_j, \Rfun_i)}
\end{align*}

where $R(\Rfun_i, \Rfun_j)$ is a measure of regret of a policy optimized for reward function $\Rfun_i$ when evaluated under reward function $\Rfun_j$. \citet{wilde2020active} use a regret measure based on a ratio of returns: $R(\Rfun_i, \Rfun_j) = 1 - \ExpRetUnder{\Rfun_j}{\p_i} / \ExpRet{\p_j}$. However, this measure is only meaningful if all rewards are positive, which is not the case in our experiments. Therefore, we instead use a regret measure based on differences of returns: $R(\Rfun_i, \Rfun_j) = \ExpRet{\p_j} - \ExpRetUnder{\Rfun_j}{\p_i}$.\footnote{This change was suggested by the authors of \citet{wilde2020active} in personal communication to deal with negative rewards.}

For computing the probabilities $P(\Rfun_i | \cD)$, \citet{wilde2020active} use a simple Bayesian model that assumes a uniform prior and a likelihood of making an observation of the form
\begin{align}
P(y=+1 | q=(\p_i, \p_j)) = \begin{cases}
p & \text{if }\ExpRet{\p_i} > \ExpRet{\p_j} \\
1-p & \text{else}
\end{cases} \label{eq:maxreg_model}
\end{align}
with $0.5 < p \leq 1$.
In the main paper we report results for MR using a GP reward model. We tested the simple Bayesian model in \cref{eq:maxreg_model} in preliminary experiments, and found it to result in comparable results to the GP model with observations simulated using a linear observation model.

In our experiments, MR performed worse than reported by \citet{wilde2020active}. This difference is likely explained by differences in the implementation of the reward model, the environment, or the acquisition function.
Unfortunately, \citet{wilde2020active} did not publicly release code for their implementation of MR. Therefore, we were unable to reproduce their exact setup and results, and could not investigate differences between our implementation and theirs in detail.

\subsection{Details on the environments}

This section provides more details on the environments we test IDRL with a GP model on. We start by introducing two additional environments: \emph{Chain} and \emph{Junction}, which are not presented in the main paper. Then, we discuss the \emph{Gridworld}, \emph{Driver}, and MuJoCo \emph{Corridor} experiments from the main paper. \Cref{app:additional_results} provides more detailed results for all environments.

\subsubsection{Chain}
\Cref{fig:chain_mdp} shows the \emph{Chain} environment. It has a discrete state space with $N$ states, and a discrete action space with 2 actions $a_r$ and $a_l$. In the first $M$ states of the chain both actions moves the agent right, whereas in the last $N-M$ states $a_l$ moves the agent left and $a_r$ moves the agent right. The dynamics are deterministic. The initial state distribution is uniform over the state space.

For the GP model of the reward, we choose a squared-exponential (SE) kernel
\[
k(s, s') = \sigma^2 \exp\rbr{- \frac{d(s, s')^2}{2 l^2}}
\]
with variance $\sigma=2$ and lengthscale $l=3$. The distance $d$ counts the number of states between $s$ and $s'$ on the chain.

The \emph{Chain} environment shows that, to select informative queries, it is important to consider in which states the agents actions change the transition probabilities. Queries about the first $M$ states are less informative, because in these states the agent cannot choose how to move.

\subsubsection{Junction}
\Cref{fig:junction_mdp} shows the \emph{Junction} environment. It has a discrete state space with $N+2M$ states, and a discrete action space with 2 actions $a_1$ and $a_2$. In the first $N$ states either action moves the agent right. From state $s_N$ action $a_1$ moves the agent to $s_{A_1}$ and action $a_2$ moves the agent to $s_{B_1}$. In either of the two paths the agent moves to one of the adjacent states with probability $0.5$, independent of the action it took. The reward of states $s_1, \dots, s_N$ is $0$, the reward of states $s_{B_1}, \dots, s_{B_M}$ is $0.8$, and the reward of state $s_{A_i}$ is
\[
\Rfun(s_{A_i}) = 1-\rbr{0.7 \cdot \frac{i}{M} - 1}^2
\]
This reward function ensures that the average reward in the upper chain is smaller than $0.8$ but the maximum reward is bigger than $0.8$. The initial state distribution is uniform over the state space.

For the GP model of the reward, we choose a SE kernel with variance $\sigma=2$ and lengthscale $l=3$. The distance $d$ measures the shortest path between $s$ and $s'$ on graph that defines the \emph{Junction} (disregarding the transition function).

The \emph{Junction} environment shows that it is not sufficient to select queries that are informative about the maximum of the reward function. Instead, it is important to consider the specifics of the environment to determine informative queries. In the \emph{Junction} environment, the agent has to find the path with the higher average reward instead of the one with the higher maximum reward.

\subsubsection{Gridworld}
The \emph{Gridworld} environment consists of a $10 \times 10$ grid in which $2$ objects of each of $10$ different types are placed, so $20$ objects in total. Each object type gives a reward uniformly sampled from $[-1, 1]$ when standing on it, while floor tiles give $0$ reward. Between each two cells, with probability $0.3$ there is a wall. The environment has a discrete states space with $100$ states and a discrete action space with $5$ actions: \emph{north}, \emph{east}, \emph{south}, \emph{west}, and \emph{stay}. The dynamics are deterministic. The initial position of the agent is randomly selected but fixed for one instance of the environment.

For the GP model of the reward we choose a kernel
\[
k(s, s') = \begin{cases}
1 &\parbox{15em}{if in $s$ and $s'$ the agent is standing on the same object type}\\
0 &\text{else}
\end{cases}
\]
so that the model learns a reward for each of the object types independently.

The \emph{Gridworld} environment shows that to select informative queries it is important to consider the reachable space in the environment. Queries of objects that are not reachable from the agent's initial position are not informative.

\subsubsection{Driver}
We implement the \emph{Driver} environment based on code provided by \citet{dorsa2017active} and \citet{biyik2019asking}. Here, we provide a brief description of the dynamics and features of the environment. For more details, refer to our implementation, or \citet{dorsa2017active}.

The \emph{Driver} environment uses point-mass dynamics with a continuous state and action space. The state $s = (x, y, \theta, v)$ consists of the agent's position $(x, y)$, its heading $\theta$, and its velocity $v$. The actions $a = (a_1, a_2)$ consist of a steering input and an acceleration. The environment dynamics are defined as
\begin{align*}
s_{t+1} &= (x_{t+1}, y_{t+1}, \theta_{t+1}, v_{t+1}) \\
&= (x_{t} + \Delta x, y_{t} + \Delta y, \theta_{t} + \Delta \theta, \mathrm{clip}(v_{t} + \Delta v, -1, 1))
\end{align*}\vspace{-2em}
\begin{align*}
(\Delta x, \Delta y, \Delta \theta, \Delta v) = (v\cos{\theta}, v\sin{\theta}, v a_1, a_2 - \alpha v)
\end{align*}
where $\alpha = 1$ is a friction parameter, and the velocity is clipped to $[-1, 1]$ at each timestep.

The environment contains a highway with three lanes. In addition to the agent, the environment contains a second car that changes from the right to the middle lane, moving on a predefined trajectory. The reward function is linear in a set of features
\[
f(s) = (f_1(s), f_2(s), f_3(s), f_4(s), 1)
\]
where $f_1(s) \propto \exp(d_1^2)$, $d_1$ is the distance to the closest lane center, $f_2(s) \propto (v-1)^2$, $f_3(s) \propto \sin(\theta)$, and $f_4(s) \propto \exp\rbr{-d_2^2 - c d_3^2}$ where $d_2$ and $d_3$ are the distance between the agent's car and the other car along the $x$ and $y$ directions respectively, and $c$ is a constant.
Note that these features correspond to the version of the environment that \citet{biyik2019asking} use and differ slightly from \citet{dorsa2017active}.
Reward functions for the environment are sampled from a Gaussian with zero mean and unit covariance. The first 4 features are normalized before the constant is appended.

The \emph{Driver} environment uses a fixed time horizon $T=50$, and policies are parameterized by $5$ actions that are each applied for $10$ time steps. For solving the environment we optimize over these policies using an L-BFGS-B solver as proposed by \citet{dorsa2017active}. We additionally use the set of candidate policies $\CandPol$ as a lookup table to reduce the variance of this solver. Whenever $\CandPol$ contains a policy that is better for a given reward function than the one returned by the solver, we choose this one instead. This was first proposed by \citet{wilde2020active}.

\subsubsection{MuJoCo Corridor}\label{app:mujoco_corridor}

Our MuJoCo \emph{Corridor} environments are based on code of the maze environments by \citet{duan2016benchmarking}. Our ``maze'' is a corridor of $13$ cells. The robot starts in the leftmost cell and one of the cells is a fixed goal cell. The true reward function, that is not directly available to the agent, rewards the agent proportional to its velocity in positive x-direction if the agent is before the goal, and rewards the agent proportionally to its velocity in negative x-direction if the agent is past the goal. This provides a reward function that is harder to learn than just moving in one direction. We encode this reward function as a linear function of a set of features
\[
f(s) = (v_x I_1, v_y I_1, I_1, \dots, v_x I_{13}, v_y I_{13}, I_{13})^T \in \R^{39}
\]
where $I_k$ are indicator features that are $1$ if the agent is in cell $k$ and $0$ otherwise, and $v_x$ and $v_y$ are the x- and y-velocity of the center of mass of the \emph{Swimmer}.

We use \emph{augmented random search} \citep{mania2018simple} with linear policies to solve the environment for a given reward function. The policy is linear in a seperate set of features than the reward function. The features for the policy are based on the standard features provided by the \emph{Swimmer} environment, extended by an indicator feature for the Swimmer being in each of the cells.

\begin{figure*}\centering
\begin{subfigure}{0.48\linewidth}
    \strut\vspace*{-\baselineskip}\newline\resizebox{0.99\linewidth}{!}{\definecolor{ETHa}{RGB}{31,64,122}

\begin{tikzpicture}
 \node (a) at (0,0)
   {
      \begin{tikzpicture}[->, >=stealth', shorten >=1pt, auto, node distance=2cm, semithick, state/.style={circle, draw,
      minimum size=1.2cm}]
         \node[state] (s1) {$s_1$};
         \node[state] (s2) [right of=s1] {$s_2$};
         \node[minimum size=1cm] (dots1) [right of=s2] {$\dots$};
         \node[state] (sM) [right of=dots1] {$s_M$};
         \node[minimum size=1cm] (dots2) [right of=sM] {$\dots$};
         \node[state] (sN-1) [right of=dots2] {$s_{N-1}$};
         \node[state] (sN) [right of=sN-1] {$s_N$};

         \path (s1) edge node {} (s2);
         \path (s2) edge node {} (dots1);
         \path (dots1) edge node {} (sM);
         \path (sM) edge[bend left] node {$a_r$} (dots2);
         \path (dots2) edge[bend left] node {$a_l$} (sM);
         \path (dots2) edge[bend left] node {$a_r$} (sN-1);
         \path (sN-1) edge[bend left] node {$a_l$} (dots2);
         \path (sN-1) edge[bend left] node {$a_r$} (sN);
         \path (sN) edge[bend left] node {$a_l$} (sN-1);
      \end{tikzpicture}
   };
\end{tikzpicture}}
    \caption{Chain}\label{fig:chain_mdp}
\end{subfigure}\hfill
\begin{subfigure}{0.48\linewidth}
    \strut\vspace*{-\baselineskip}\newline\resizebox{0.99\linewidth}{!}{\definecolor{ETHa}{RGB}{31,64,122}

\begin{tikzpicture}[->, >=stealth', shorten >=1pt, auto, node distance=2cm, semithick, state/.style={circle, draw,
minimum size=1.2cm}]
   \node[state] (s1) {$s_1$};
   \node[state] (s2) [right of=s1] {$s_2$};
   \node[minimum size=1cm] (dots1) [right of=s2] {$\dots$};
   \node[state] (sN) [right of=dots1] {$s_{N}$};

   \path (sN) ++(30:2.3) node[state] (sA1) {$s_{A_1}$};
   \node[state] (sA2) [right of=sA1] {$s_{A_2}$};
   \node[minimum size=1cm] (dotsA) [right of=sA2] {$\dots$};
   \node[state] (sAM) [right of=dotsA] {$s_{A_{M}}$};
   \node (empty1) [above of=sAM] {}; %

   \path (sN) ++(330:2.3) node[state] (sB1) {$s_{B_1}$};
   \node[state] (sB2) [right of=sB1] {$s_{B_2}$};
   \node[minimum size=1cm] (dotsB) [right of=sB2] {$\dots$};
   \node[state] (sBM) [right of=dotsB] {$s_{B_{M}}$};
   \node (empty2) [below of=sBM] {}; %

   \path (s1) edge node {} (s2);
   \path (s2) edge node {} (dots1);
   \path (dots1) edge node {} (sN);

   \path (sN) edge node {$a_1$} (sA1);
   \path[<->] (sA1) edge node {} (sA2);
   \path[<->] (sA2) edge node {} (dotsA);
   \path[<->] (dotsA) edge node {} (sAM);
   \path[<->] (sA1.south west) edge[transform canvas={yshift=-8mm}, thick, ETHa] node[label={[label distance=1mm, pos=0.6] above:random walk}] {} (sAM.south east);

   \path (sN) edge node {$a_2$} (sB1);
   \path[<->] (sB1) edge node {} (sB2);
   \path[<->] (sB2) edge node {} (dotsB);
   \path[<->] (dotsB) edge node {} (sBM);

   \draw[-] [
          thick,
          decoration={
              brace,
              mirror,
              amplitude=10pt,
              raise=0.5cm
          },
          decorate
      ] (s1.south west) -- (sN.south east) node [black,midway,yshift=-1.5cm] {$r = 0$};

   \draw[-] [
          thick,
          decoration={
              brace,
              mirror,
              amplitude=10pt,
              raise=0.5cm
          },
          decorate
      ] (sAM.north east) -- (sA1.north west) node [black,midway,yshift=1.6cm] {$\mathrm{avg}(r) < 0.8$, but $\max(r) > 0.8$};

   \draw[-] [
       thick,
       decoration={
           brace,
           mirror,
           amplitude=10pt,
           raise=0.5cm
       },
       decorate
   ] (sB1.south west) -- (sBM.south east) node [black,midway,yshift=-1.5cm] {$r = 0.8$};
\end{tikzpicture}}
    \caption{Junction}\label{fig:junction_mdp}
\end{subfigure}
\caption{Illustration of the \emph{Chain} and \emph{Junction} MDPs.
In the \emph{Chain} MDP (\subref{fig:chain_mdp}), the agent moves right with probability $1$ from state $s_1$ to $s_M$, and the agent can move deterministically left and right from state $s_M$ to state $s_N$ with $N > M$. The reward function is sampled from the GP prior with a square-exponential kernel. We choose $M=10$ and $N=20$.
In the \emph{Junction} MDP (\subref{fig:junction_mdp}), the agent moves right with probability $1$ from state $s_1$ to $s_{N}$. In $s_{N}$, the agent can take the upper or lower path, denoted with $A$ and $B$, respectively. Along both paths, the agent moves either left or right with probability $0.5$. In the lower path, the reward of all states is $0.8$. In the upper path, the highest reward is greater than $0.8$, but the average reward is less than $0.8$. We choose $N=15$ and $M=5$.
For both environments, the initial state distribution is uniform over the state space.
}
\end{figure*}
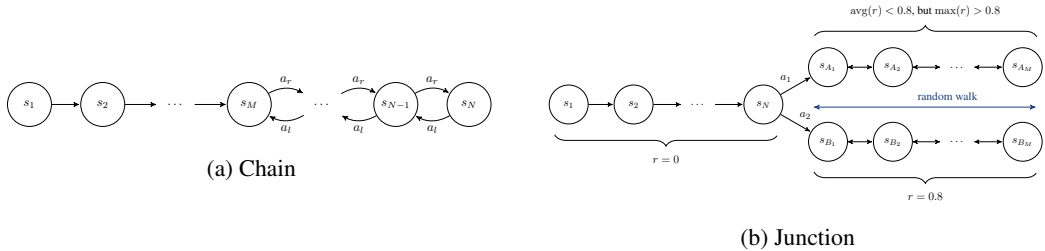

\begin{algorithm}[p]
\caption{\emph{\AlgoNameLong} (\AlgoNameShort) using DNNs and SAC to train policies. The algorithm maintains a currently optimal policy $\MeanOptPol$ as well as a set of candidate policies $\CandPol$. The reward function is represented with a neural network feature function parameterized by $\phi$, and a Bayesian linear model given by mean $\theta$ and covariance $H^{-1}$. The main training loop alternates between updating $\MeanOptPol$, querying new samples and updating the candidate policies. To select queries, the algorithm implements the IDRL objective, where \cref{line:idrl_deep_rl:policy_selection} corresponds to \cref{eq:dig_policy_selection} and \cref{line:idrl_deep_rl:query_selection} corresponds to \cref{eq:dig_state_selection}. \Cref{line:idrl_deep_rl:dnn_update} updates the feature representation and estimates the MAP of the posterior, and \cref{line:idrl_deep_rl:cov_update} approximates the posterior using a Laplace approximation. This pseudocode omits all hyperparameters that control how often new samples are queried, how many samples are queries, and how often and for how many steps all models are updated.
}
  \label{alg:idrl_deep_rl}
\begin{algorithmic}[1]
  \State Initialize policy $\MeanOptPol$
  \State Initialize candidate policies $\CandPol \gets \{ \p^c_1, \dots, \p^c_{n} \}$
  \State $\Dataset \gets$ initial dataset
  \State
  \State Initialize DNN reward model $\mu_{\RfunBelief}(o, a) = \theta^T f_\phi(o, a)$ with $\phi$ random weights, and $\theta = 0$
  \State Update DNN parameters, $\phi$ and $\theta$, via supervised learning on $\Dataset$
  \State $H \gets \left[ \nabla^2 \log~ p(\theta | \mathcal{D}, \phi) \right]\Bigr \rvert_{\theta}$
  \State
  \While{not done}
    \State Train policy $\MeanOptPol$ on reward $r_i(o, a) = \theta^T f_\phi(o, a)$ using SAC
    \State
    \If{update candidate policies}
        \State Sample $\vomega_1, \dots, \vomega_n$ from $\cN(\theta, H^{-1})$
        \For{$i \in \{1, \dots, n\}$}
            \State Train policy $\p^c_i$ on reward $r_i(o, a) = \vomega_i^T f_\phi(o, a)$ using SAC
            \State Estimate state visitation frequency $\StateVisit{\p^c_i}$ using Monte-Carlo rollouts
        \EndFor
    \EndIf
    \State
    \If{query samples}
        \State $\p_1, \p_2 \in \argmax_{\p, \p' \in \CandPol} (\StateVisit{\p} - \StateVisit{\p'})^T H^{-1} (\StateVisit{\p} - \StateVisit{\p'})$ \label{line:idrl_deep_rl:policy_selection}
        \State $\StateVisitDirection{\p_1}{\p_2} \gets \StateVisit{\p_1} - \StateVisit{\p_2}$
        \State Roll out policy $\MeanOptPol$ and collect a set of candidate queries $\CandStates$
        \For{$\Query_i \in \CandStates$}
            \State Compute Hessian with the expected response to $\Query_i$:
            \State \hspace*{1em} $H_{\Query_i} \gets \left[ \nabla^2 \log~ p(\theta | \mathcal{D} \cup \{(\Query_i, \EvidenceBelief_i)\}, \phi) \right]\Bigr \rvert_{\theta}$
            \State $u(\Query_i, \Dataset) \gets - \StateVisitDirection{\p_1}{\p_2}^T H_{\Query_i}^{-1} \StateVisitDirection{\p_1}{\p_2}$ \label{line:idrl_deep_rl:query_selection}
        \EndFor
        \State Sort queries by $u(\Query_i, \Dataset)$ and select $k$ queries $\{\Query_1, \dots, \Query_k\}$ with the largest values
        \State Make queries $\{\Query_1, \dots, \Query_k\}$ and observe $\{\Evidence_1, \dots, \Evidence_k\}$
        \State $\Dataset \gets \Dataset \cup \{((\Query_1, \Evidence_1), \dots, (\Query_k, \Evidence_k))\}$
        \State Update DNN parameters, $\phi$ and $\theta$, via supervised learning on $\Dataset$ \label{line:idrl_deep_rl:dnn_update}
        \State $H \gets \left[ \nabla^2 \log~ p(\theta | \mathcal{D}, \phi) \right]\Bigr \rvert_{\theta}$ \label{line:idrl_deep_rl:cov_update}
    \EndIf
  \EndWhile
  \State \textbf{return} $\MeanOptPol$
\end{algorithmic}
\end{algorithm}

\section{Implementation details of \AlgoNameShort with neural network reward models}
\label{app:implementation_details_deep_rl}

This section provides more details on our implementation of \AlgoNameShort using DNN reward models. \Cref{alg:idrl_deep_rl} shows pseudocode of the full algorithm.

\subsection{Training the reward model}

We represent the reward function as a function of observations $o$ and actions $a$, using a DNN model which we can write as $\mu_{\RfunBelief}(o, a) = \theta^T f_\phi(o, a)$. We conceptually separate the model into a feature representation $f_\phi(o, a)$ parameterized by weights $\phi$ and a linear function $\theta$. In practice $\theta$ is the last layer of the DNN and $\phi$ and $\theta$ are trained jointly.

We train the model on a dataset of pairwise comparisons of clips of trajectories. Let $\sigma = (s_1, a_1, \dots, s_L, a_L)$ denote a sequence of state-action pairs of lengths $L$, and let $r(\sigma) = \sum_{i=1}^L r(s_i, a_i)$ be the sum of rewards over the sequence. We make queries $\Query_i = (\sigma_{i1}, \sigma_{i2})$, $\Evidence_i \in \{-1, 1\}$ that compare two such sequences of state-action pairs. Similar to \citet{christiano2017deep}, we choose the Bradley-Terry observations model for the comparisons:
\[
p(\sigma_1 > \sigma_2 | \theta, \phi) = \frac{\exp(\hat{r}(\sigma_1))}{\exp(\hat{r}(\sigma_1)) + \exp(\hat{r}(\sigma_2))} = \frac{1}{1 + \exp(-(\hat{r}(\sigma_1)-\hat{r}(\sigma_2)))}
\]
which is equivalent to Logistic regression on $\hat{r}(\sigma_1)-\hat{r}(\sigma_2)$. We use gradient descent to minimize the negative log-likelihood of the data under this observation model
\[
\hat{\phi}, \hat{\theta} \in \argmin \cL(\theta, \phi)
\]
\[
\cL(\theta, \phi) = -\mathrm{log}~ p(\mathcal{D} | \phi, \theta) + \lambda (||\theta||^2 + ||\phi||^2)
\]
where we use $\ell_2$-regularization which corresponds to a Gaussian prior on the weights. Because the Bradley-Terry model is invariant to shifting the reward function, we use DNN layers without biases. Additionally, we normalize the output of the model when using it to train policies.

To compute the \AlgoNameShort objective, we need a Bayesian posterior. We fix the features $f_{\hat{\phi}}$, and perform Bayesian regression to approximate the posterior $p(\theta | \mathcal{D}, \hat{\phi})$. To this end, we consider $\hat{\theta}$ to be the mode of this posterior, and compute a Laplace approximation:
\[
p(\theta | \mathcal{D}, \hat{\phi}) \approx \mathcal{N}(\hat{\theta}, H^{-1})
\]
\[
H = \left[ \nabla^2 \mathrm{log}~ p(\theta | \mathcal{D}, \hat{\phi}) \right]\Bigr \rvert_{\theta = \hat{\theta}}
\]
where $H$ is the Hessian of the log-likelihood at point $\hat{\theta}$. The Laplace approximation is a very basic technique for approximate inference; however, it is convenient in our case because it approximates the posterior as a Gaussian distribution. This means we can compute the entropy and information gain of this distribution similarly easy as for a GP model.

\subsection{Hyperparameter choices}

\paragraph{Neural network model.} We use the same network architecture as \citet{christiano2017deep}: a two-layer neural network with $64$ hidden units each and leaky-ReLU activation functions ($\alpha=0.01$). For training we use $\ell_2$-regularization with $\lambda=0.5$.

\paragraphsmall{Policy training.} We use the \texttt{stable-baselines3} implementation of SAC \citep{stable-baselines3}, with default hyperparameters. For training the policy, we append a feature to the observations that measure the remaining time within an episode, $f_t = (t_{\max} - t) / t_{\max}$, where $t$ is the current time step in an episode and $t_{\max}$ is the episode length. Adding this feature tends to speed up training significantly in the MuJoCo environments. We do not add this feature for learning the reward function. Policies are trained for $10^7$ timesteps in total.

\paragraphsmall{Sampling rate.} We provide $25\%$ of samples to the reward model before starting to train the policy, and during training provide samples at a sampling rate proportional to $1/T$. Concretely, if $N_s$ samples are provided in $N_b$ batches over the course of training, the $i$-th batch will contain $\frac{N_s}{H_{N_b}} \cdot \frac{1}{T}$ samples, where $H_n$ is the $n$-th harmonic number.

\paragraphsmall{Candidate policies.} We maintain a set of $3$ candidate policies, that are each updated $10^7$ timesteps, as the main policy. The candidate policies are updated in regular intervals, which are controled by a hyperparameter $N_p$. Over the course of training, the candidate policies will be updated $N_p$ times using $10^7 / N_p$ timesteps each time.

\paragraphsmall{Hyperparameter tuning.} We only tuned two hyperparameters explicitly: $N_b$, the number of batches of training samples the model gets during training, and $N_p$ the number of times the candidate policies are updated during training. We selected all other hyperparameters after preliminary experiments and to be as similar as possible to \citet{christiano2017deep}. We first tuned $N_b$ using a random acquisition function and values in $\{10, 100, 1000, 10000\}$, and chose $N_b = 1000$ which gave the best performance evaluated over $5$ random seeds. We choose the same $N_b$ for all acquisition functions. Then, we tuned $N_p$ for the \AlgoNameShort acquisition function and values in $\{10, 100, 200, 400, 600, 800, 1000\}$. We chose $N_p = 100$ which lead to best performance evaluated over $5$ random seeds. All hyperparameters were only tuned on the HalfCheetah environment.

\subsection{Comparison of our Deep RL setup to Christiano et al.}\label{app:comparison_to_christiano}

In this section we point out differences in our Deep RL setup compared to \citet{christiano2017deep}. Some of the modifications are necessary for applying \AlgoNameShort. Other differences result from us not being able to reproduce the exact environments and hyperparameters because \citeauthor{christiano2017deep} do not provide code of their experiments.

\paragraphsmall{Reward model.} \citeauthor{christiano2017deep} model the reward function with an ensemble of DNNs. We learn a feature representation using a single DNN, and combine this with a Bayesian linear model which makes computing the \AlgoNameShort objective more straightforward.

\paragraphsmall{Policy learning.} We use SAC while \citeauthor{christiano2017deep} use TRPO for learning the policy. We chose SAC because it is significantly more sample efficient in MuJoCo environments.

\paragraphsmall{Sampling rate.} \citeauthor{christiano2017deep} provide $25\%$ of total samples to the model intially, and provide the rest of the samples at an adaptive sampling rate which they choose to be ``roughly proportional to $2 \cdot 10^6 / (T+2*10^6)$'' \citep[App. A.1]{christiano2017deep}, where $T$ is the number of environment interactions so far. Unfortunately, they do not provide enough information to exactly reproduce their sampling schedule. Instead, we simplify the schedule to be proportional to $1/T$.

\paragraphsmall{Clip length.} \citeauthor{christiano2017deep} query comparisons between clips that ``last 1.5 seconds, which varies from 15 to 60 timesteps depending on the task''\citep[App. A.1]{christiano2017deep}. Unfortunately, they do not specify the framerate used for each task, so we can not reproduce the exact clip lengths. Instead, we simply choose a length of 40 timesteps for each environment which is roughly in the middle of the range they provide.

\paragraphsmall{Observations. }From their paper it is unclear whether \citet{christiano2017deep} include the agents position in the observation in locomotion environments such as the HalfCheetah. Note, that including the observation makes the reward learning task much easier because the reward function is linear in the change of the agent's x-position. Therefore, we do not include the position in the observation which we use to predict the reward function.

\paragraphsmall{Penalties for termination.} Most of the standard MuJoCo environments have termination conditions, that, e.g., terminate an episode when the robot falls over. Such termination can leak information about the reward, i.e., longer episodes are better. Therefore, \citeauthor{christiano2017deep} replace ``these termination conditions by a penalty which encourages the parameters to remain in the range'' \citep[App. A]{christiano2017deep}. Unfortunately, they do not specify the exact penalties they use. We also remove the termination condition, but replace it with a bonus for ``being alive'' which is implemented in the version 3 environments of OpenAI Gym.

\section{Additional experimental results}\label{app:additional_results}

\subsection{Experiments in small environments}

In \Cref{fig:additional_results_toy}, we provide more detailed results of our experiments in small environments. This includes experiments in the \emph{Gridworld} environment, but also the \emph{Chain} and \emph{Junction} environments that were not presented in the main paper. The results confirm the conclusion we presented in the main paper: by focusing on queries that are informative about which policy is optimal, \AlgoNameShort is able to learn better policies with fewer queries than the baselines.

\begin{figure}[p]
{\centering
    \begin{adjustbox}{max width=\textwidth}
    \begin{tabular}{>{\centering\arraybackslash}m{2cm}|>{\centering\arraybackslash}m{7.5cm}>{\centering\arraybackslash}m{7.5cm}}
        Environment & \multicolumn{2}{c}{Query Type} \\
        & Reward of States & Comparison of States \\
        \midrule
        Chain &
        \includegraphics[width=0.49\linewidth]{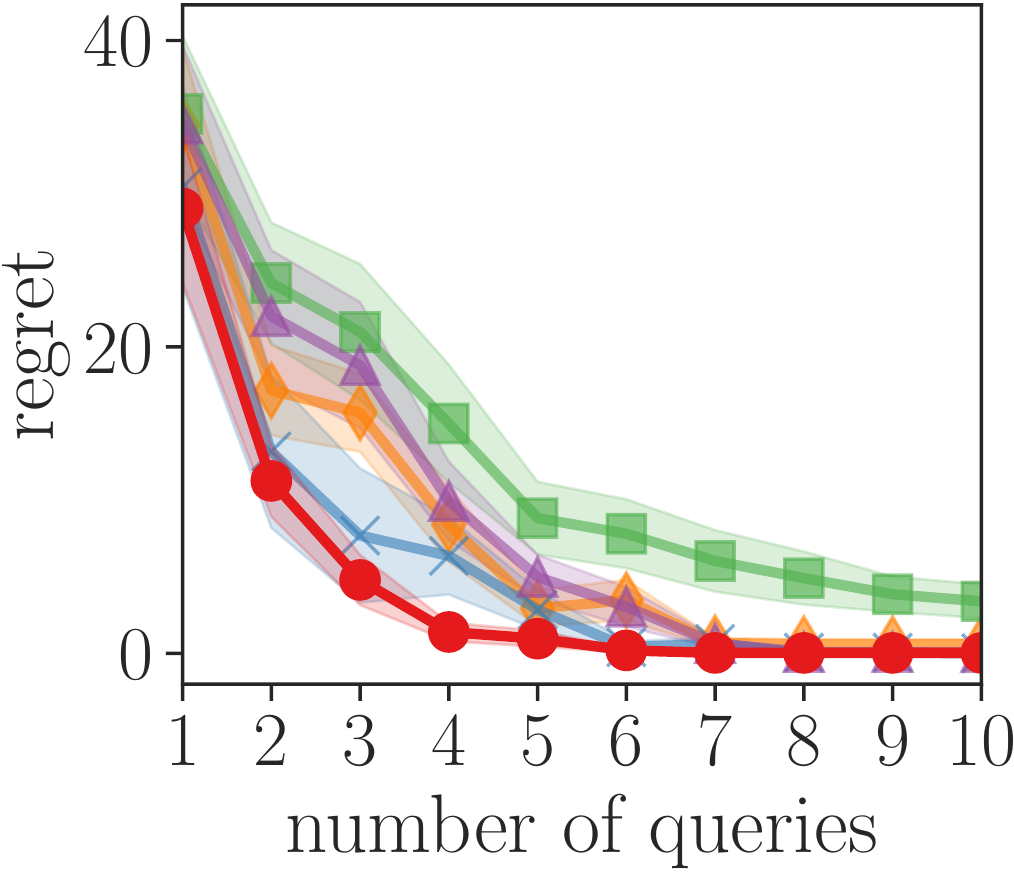}
        \includegraphics[width=0.49\linewidth]{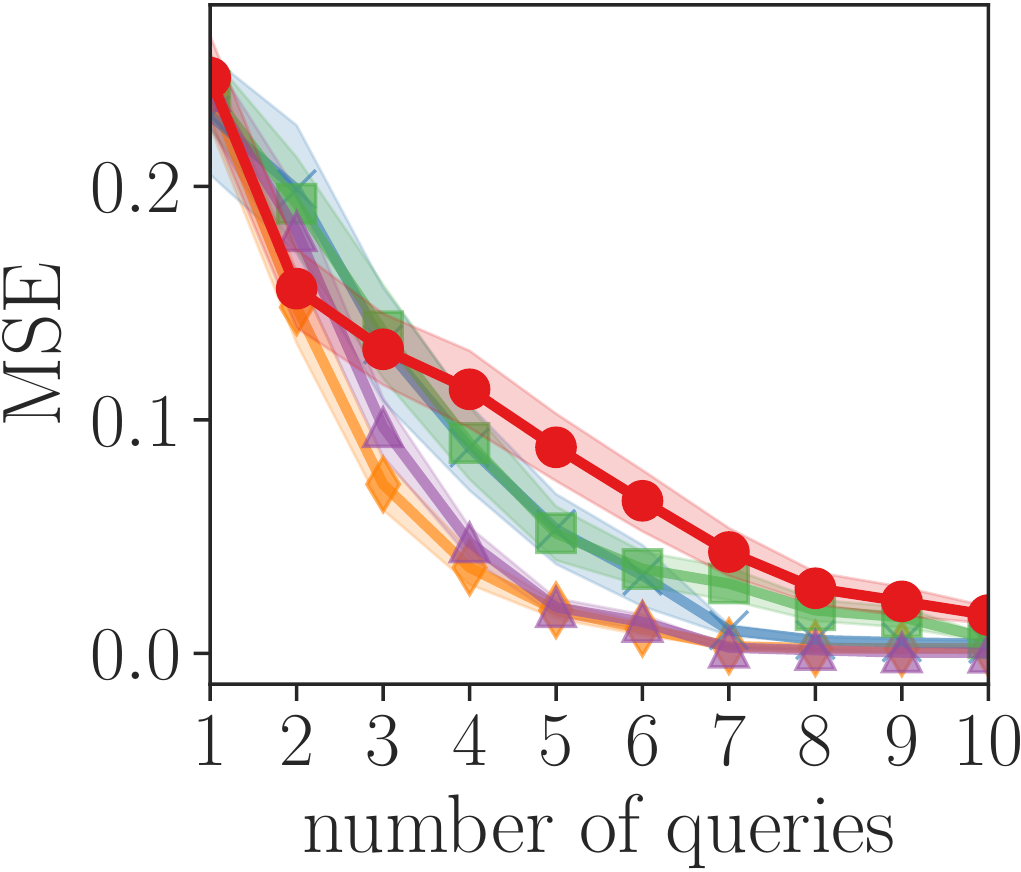} &
        \includegraphics[width=0.49\linewidth]{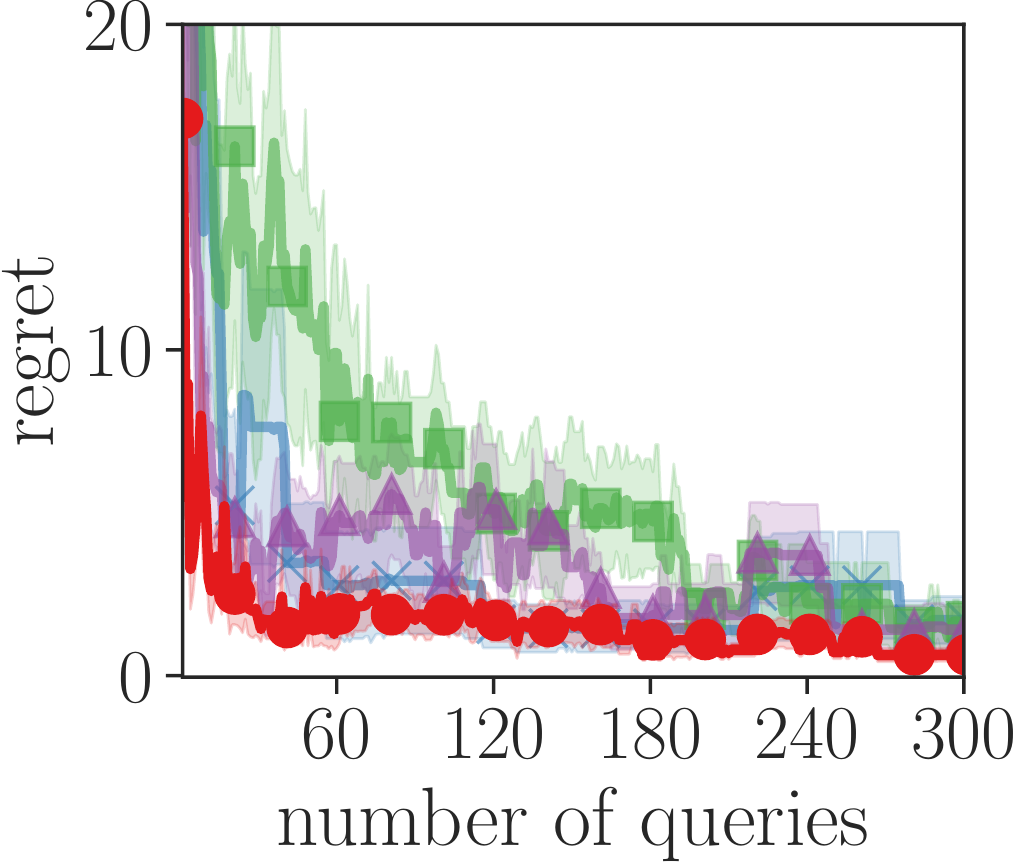}
        \includegraphics[width=0.49\linewidth]{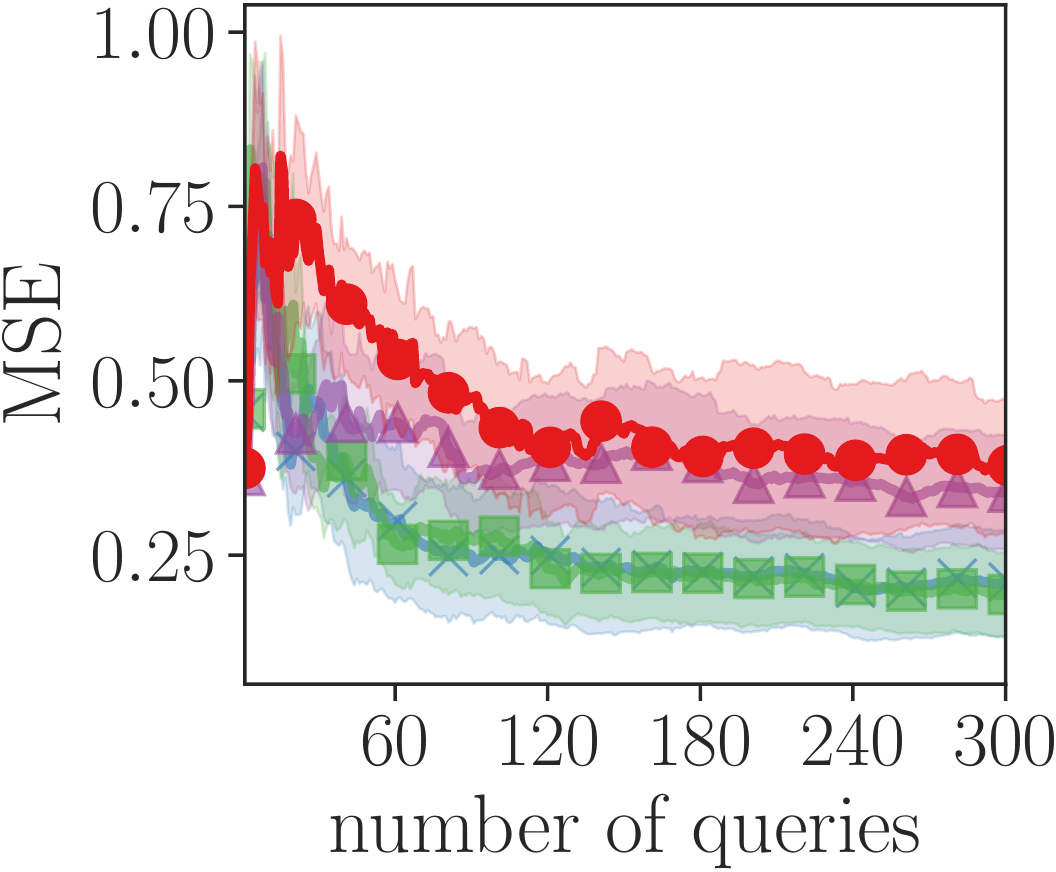} \\
        Junction &
        \includegraphics[width=0.49\linewidth]{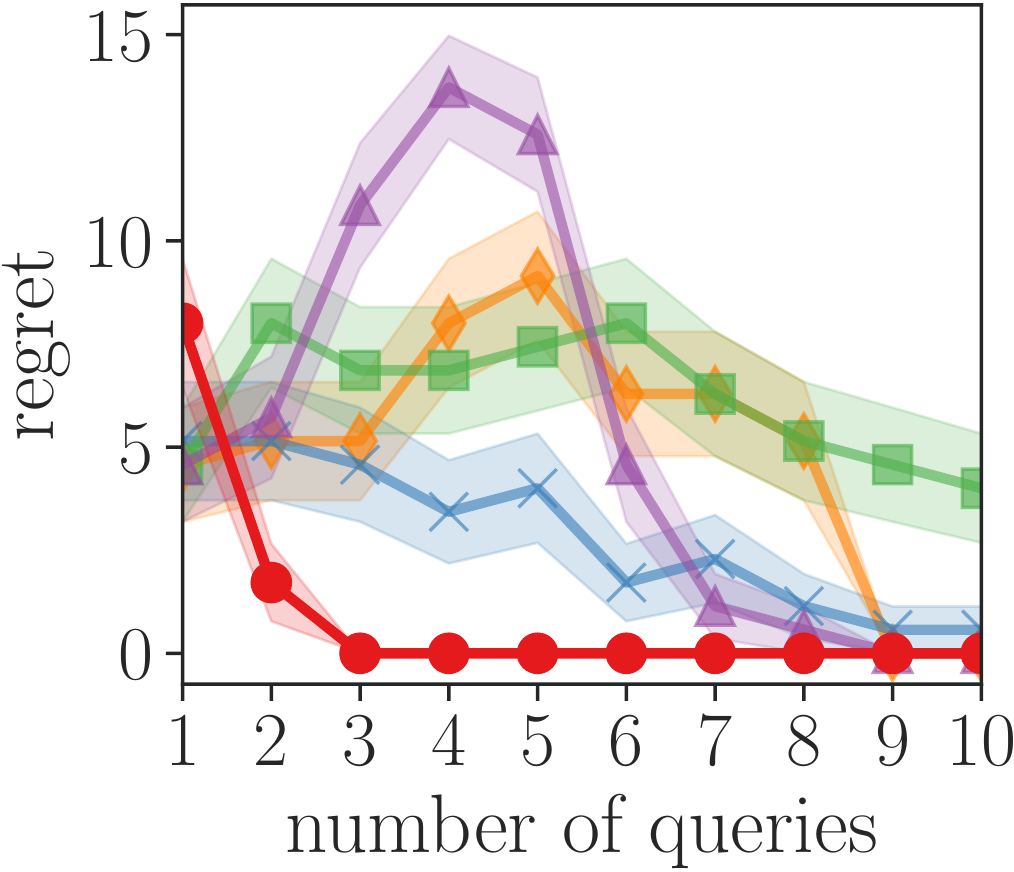}
        \includegraphics[width=0.49\linewidth]{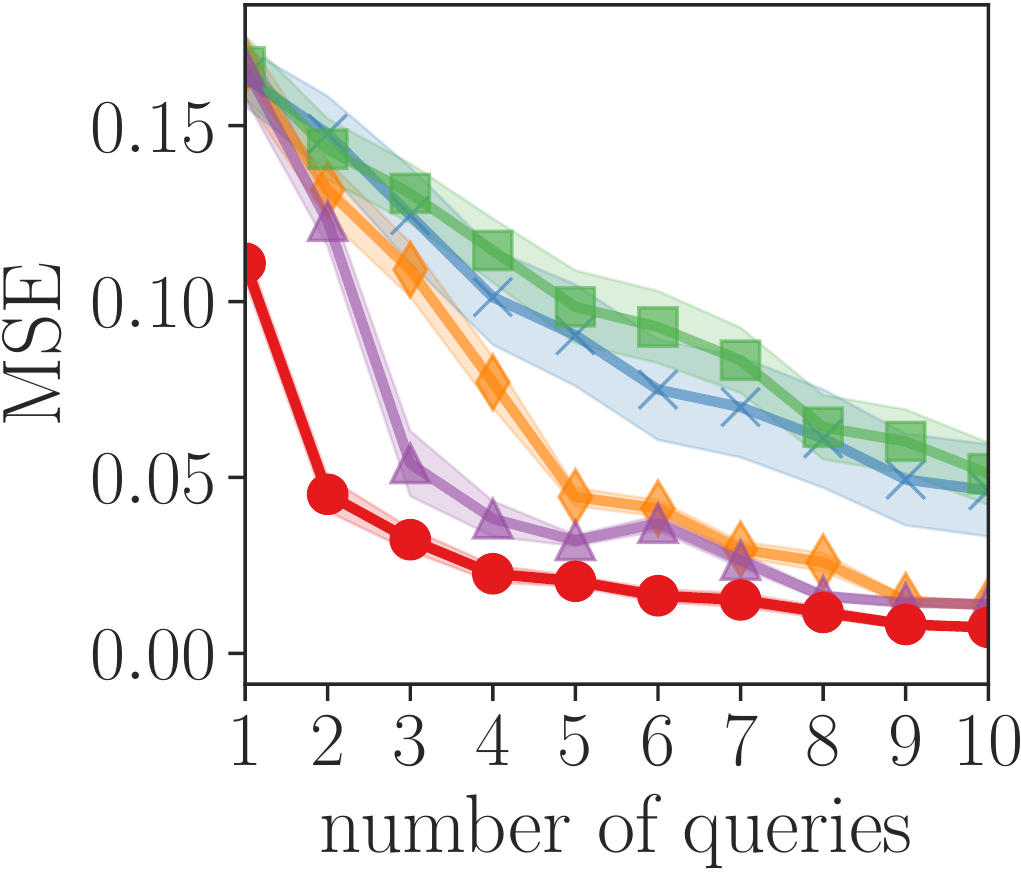} &
        \includegraphics[width=0.49\linewidth]{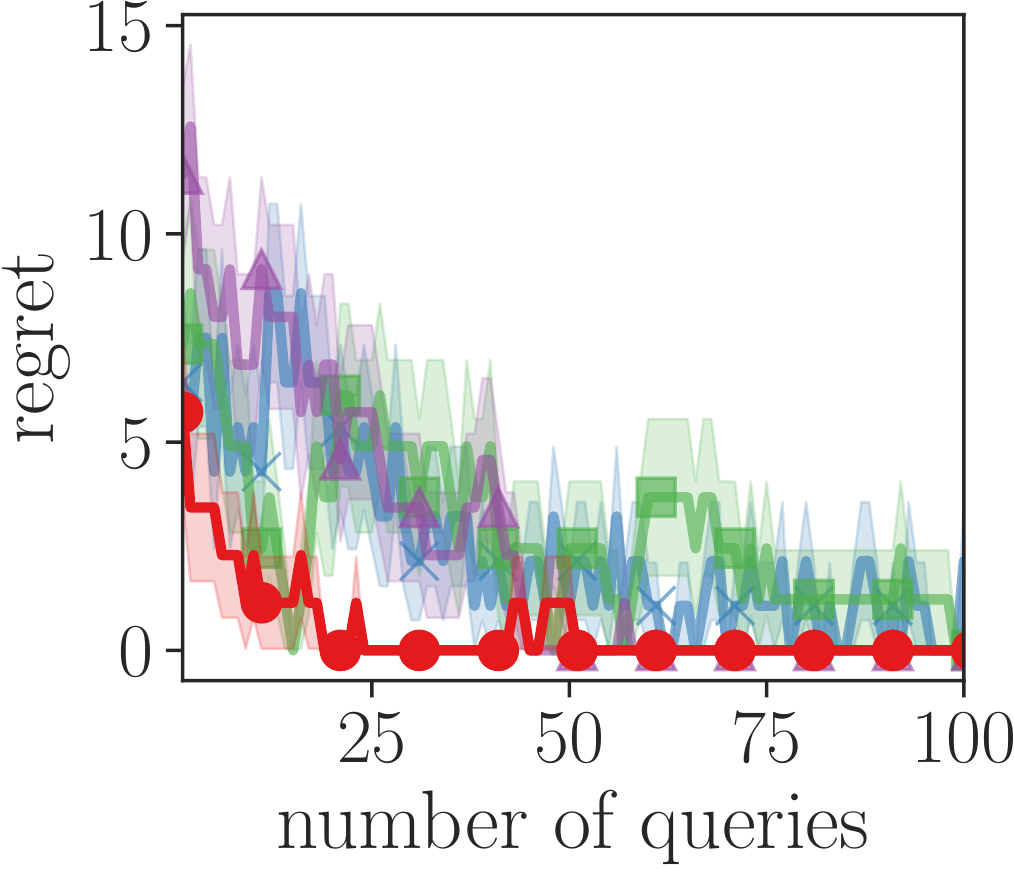}
        \includegraphics[width=0.49\linewidth]{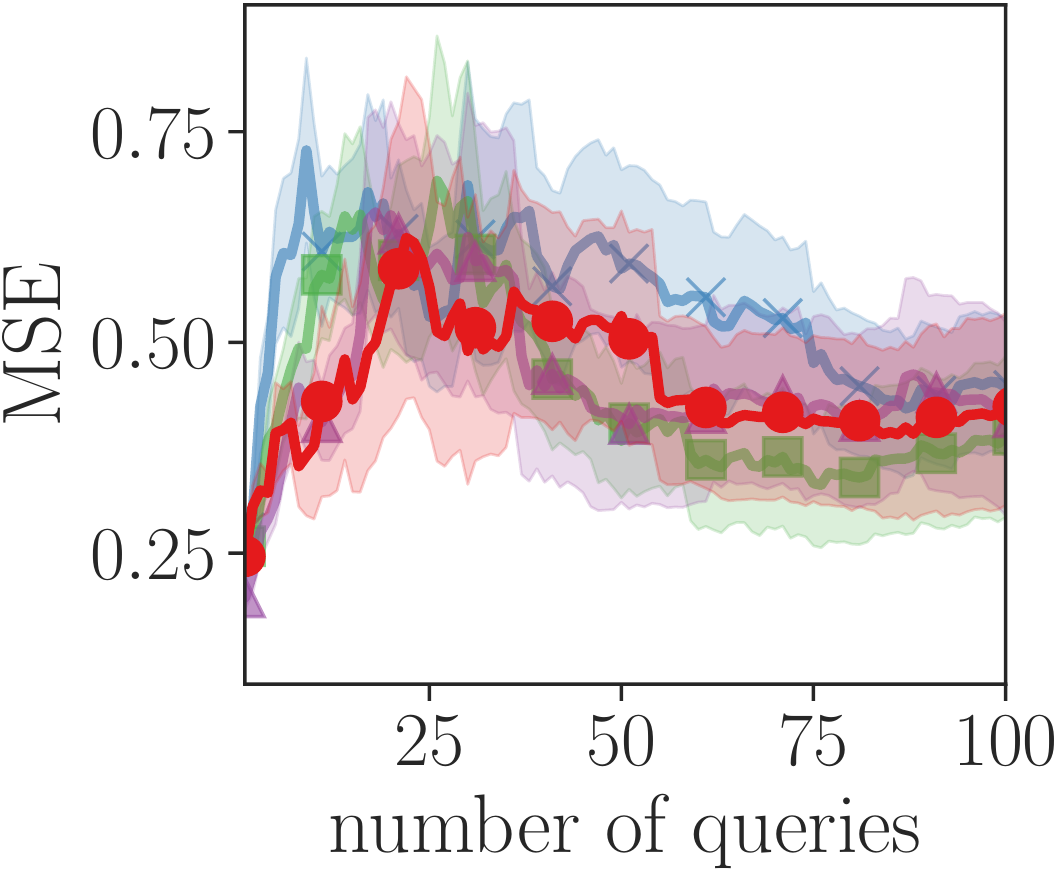} \\
        Gridworld &
        \includegraphics[width=0.49\linewidth]{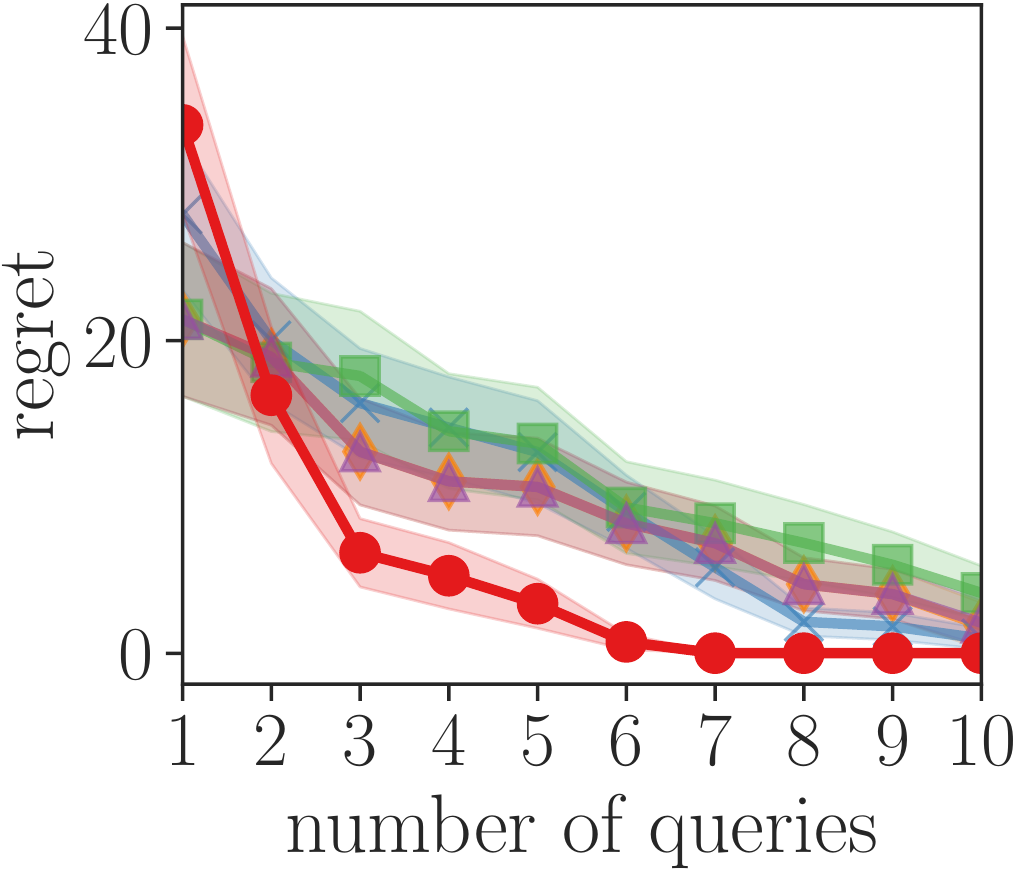}
        \includegraphics[width=0.49\linewidth]{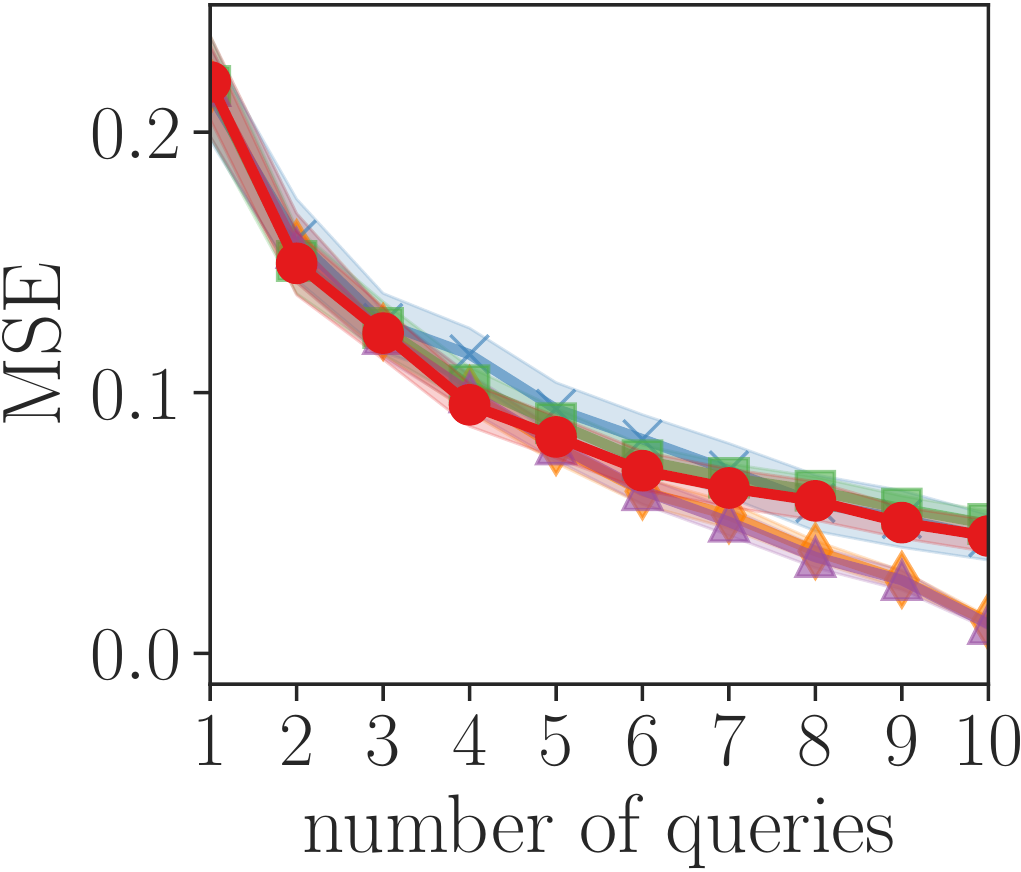} &
        \includegraphics[width=0.49\linewidth]{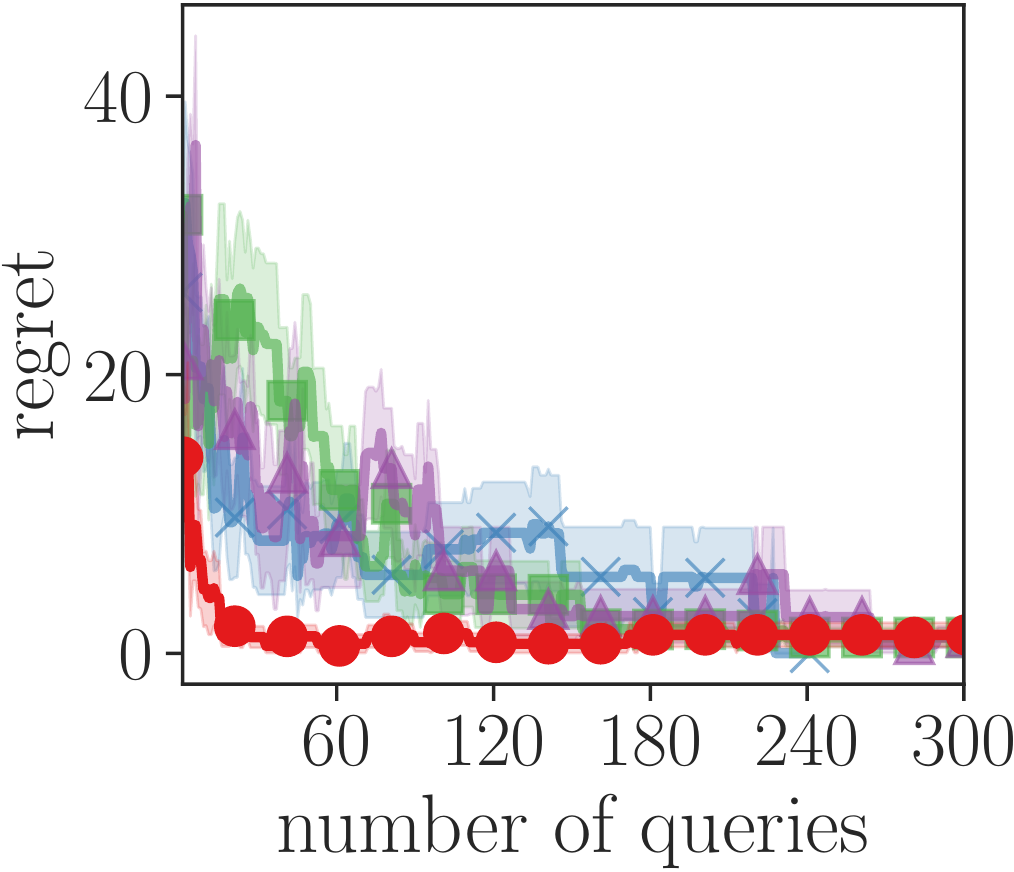}
        \includegraphics[width=0.49\linewidth]{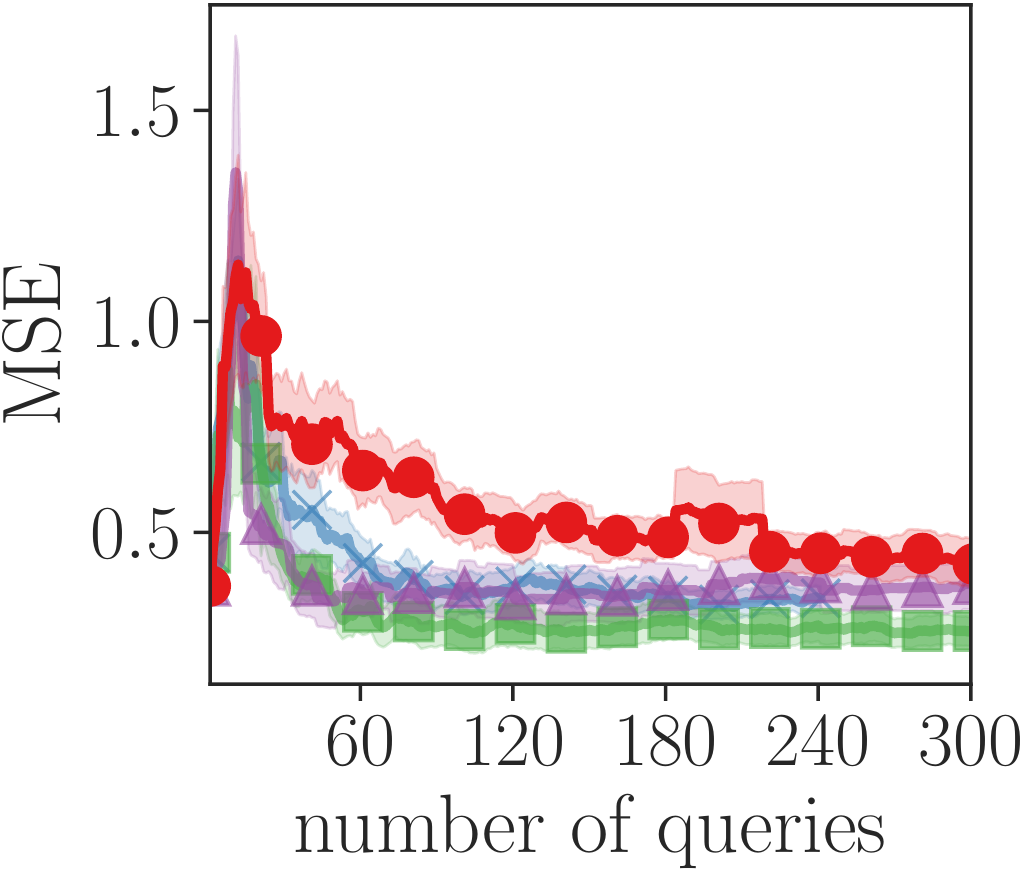}
    \end{tabular}
    \end{adjustbox}
}
{\small
    \hspace{0.1em}\textbf{Legend:} \vspace{1em} \\
    \begin{tabular}{llll}
    \legendUniform & Uniform Sampling &
    \legendEI & Expected Improvement (EI) \\
    \legendIGGP & Information Gain on the Reward (IGR) &
    \legendEPD & Expected Policy Divergence (EPD) \\
    \legendIDRL & \AlgoNameShort (ours) & &
    \end{tabular}
}
\caption{
    Learning curves for the \emph{Chain}, \emph{Junction} and \emph{Gridworld} environments for two query types: the reward of individual states and comparisons of states. For each setting, one plot shows the regret of a policy trained using the reward model and a second plot shows the mean squared error (MSE) of the reward model over the whole state space. The plots show mean and standard error across 30 random seeds. Across all environments, \AlgoNameShort learns a better policy with fewer queries. However, the MSE measured on the entire state space is usually worse, because \AlgoNameShort focuses on regions of the state space that are informative about the optimal policy.
}
\label{fig:additional_results_toy}
\end{figure}

\subsection{Ant Corridor}

\Cref{fig:additional_results_ant_corridor} shows results in the \emph{Ant-Corridor}, using the same experimental setup as the results in \Cref{fig:result_mujoco} in the main paper, which shows results in the \emph{Swimmer-Corridor}.

\begin{figure}[p]
\centering
\includegraphics[width=0.4\linewidth]{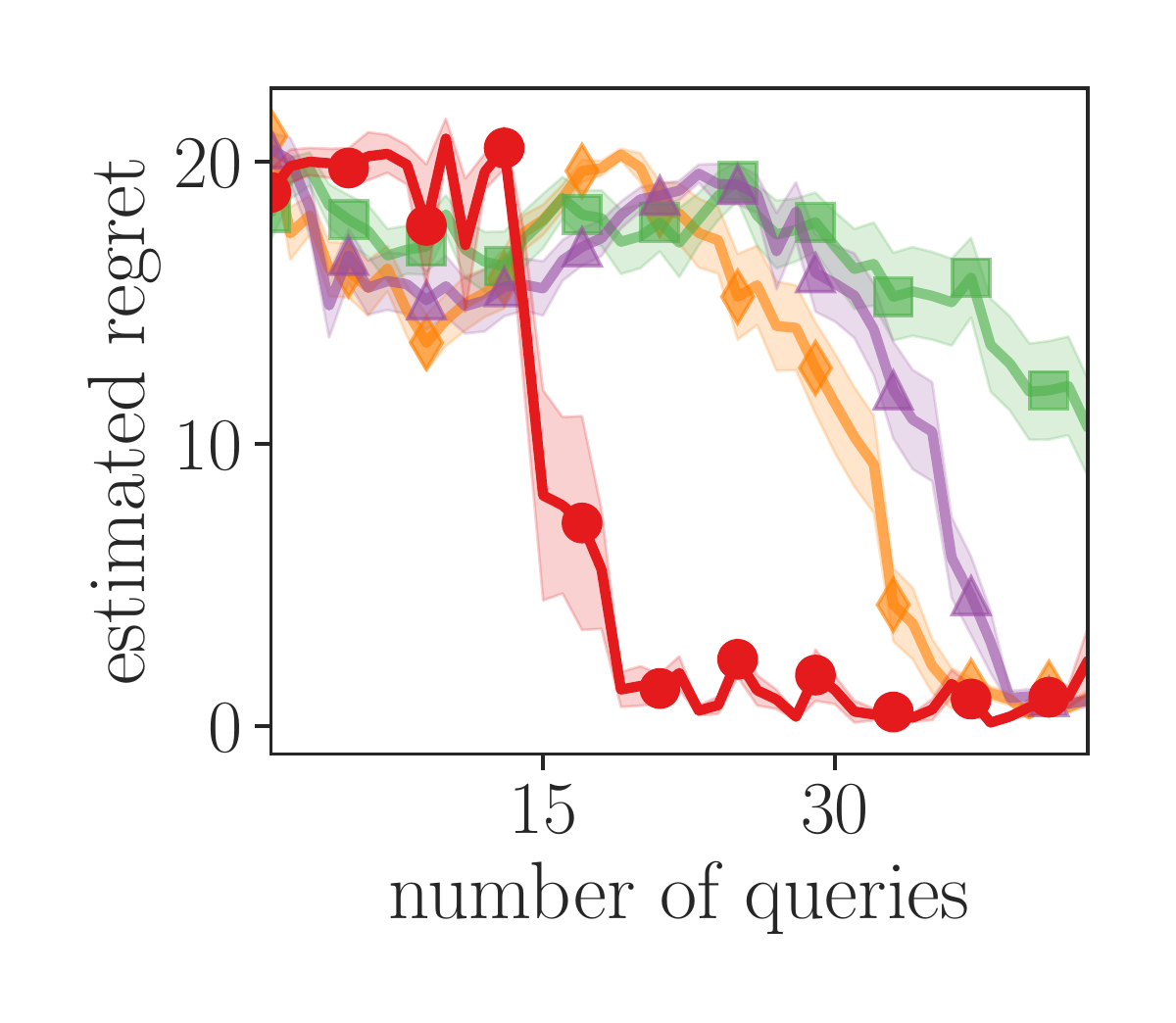}
\includegraphics[width=0.4\linewidth]{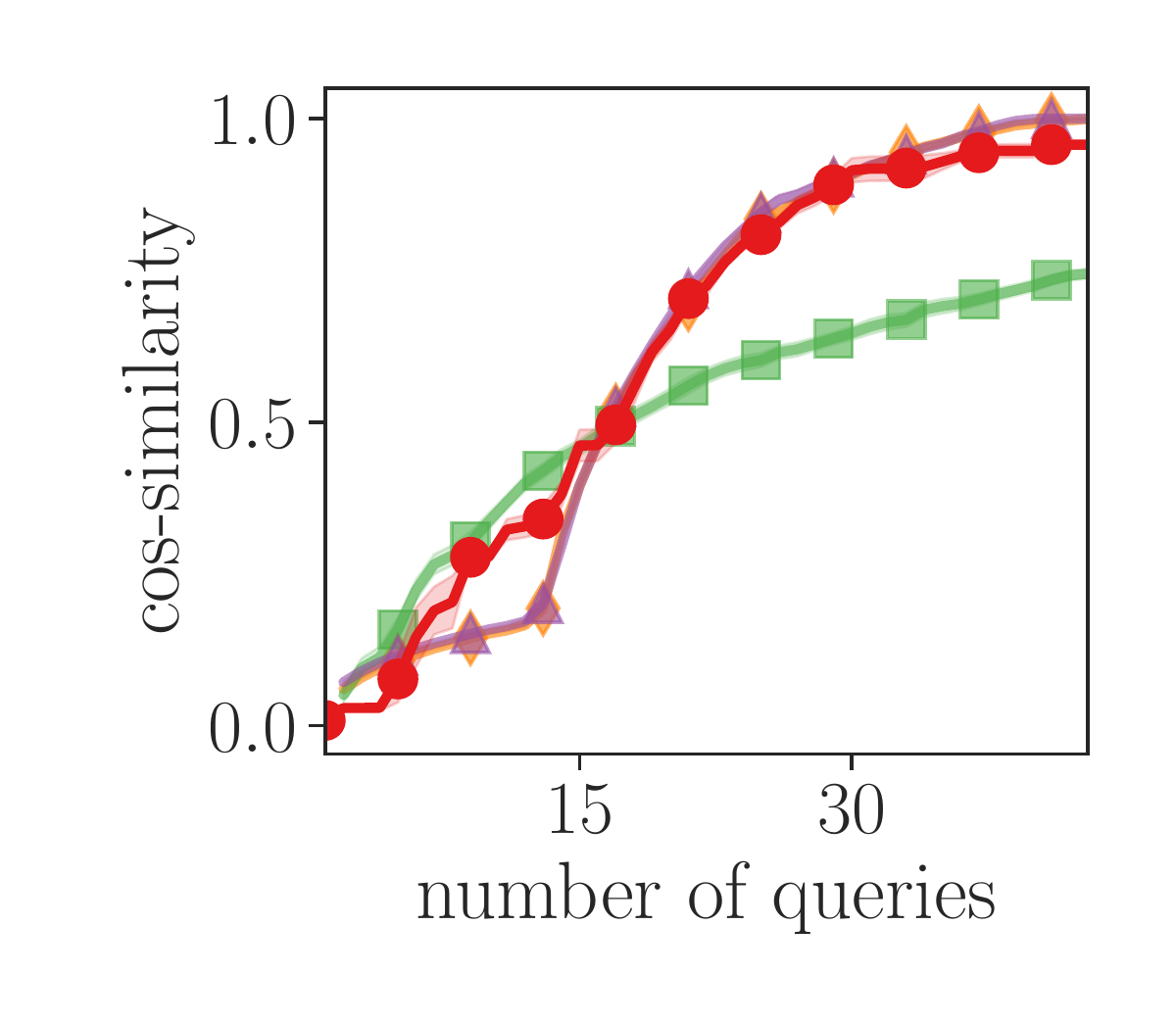}

\caption{
    Regret and cosine similarity in the \emph{Ant-Corridor} environment, comparing IDRL (\legendIDRL) to IGR (\legendIGGP), EI (\legendEI), and uniform sampling (\legendUniform). The experimental setup is exactly the same to the results shown in \Cref{fig:result_mujoco} in the main paper. The \emph{Ant-Corridor} environment is the same as the \emph{Swimmer-Corridor}, only with a different robot.
}
\label{fig:additional_results_ant_corridor}
\end{figure}

\subsection{Deep RL experiments in individual environments}

In \Cref{fig:additional_results_deep_rl} we provide learning curves for the individual MuJoCo environments that were aggregated to create \Cref{fig:results_deep_rl} in the main paper. To aggregate the results we normalized the return in each environment:
\[
\Ret_{\text{norm}}(\p) = 100 \cdot \frac{\Ret(\p) - \Ret(\p_{\text{rand}})}{\Ret(\p^*) - \Ret(\p_{\text{rand}})},
\]
where $\p_{\text{rand}}$ is a policy that samples action uniformly at random, and $\pi^*$ is an expert policy trained using SAC on the true reward. This results in a score that is $0$ for a policy that performs as well as a random policy, and $100$ for a policy that matches an expert performance. In \Cref{fig:results_deep_rl} this score is averaged over all environments.

\begin{figure}[p]
{\centering
    \begin{adjustbox}{max width=\textwidth}
    \begin{tabular}{>{\centering\arraybackslash}m{3.75cm}>{\centering\arraybackslash}m{3.75cm}>{\centering\arraybackslash}m{3.75cm}>{\centering\arraybackslash}m{3.75cm}}
        {\hspace*{3em}\small Ant-v3} & {\hspace*{3em}\small HalfCheetah-v3} & {\hspace*{3em}\small Hopper-v3} & {\hspace*{3em}\small Walker2d-v3} \\
        \includegraphics[width=0.99\linewidth]{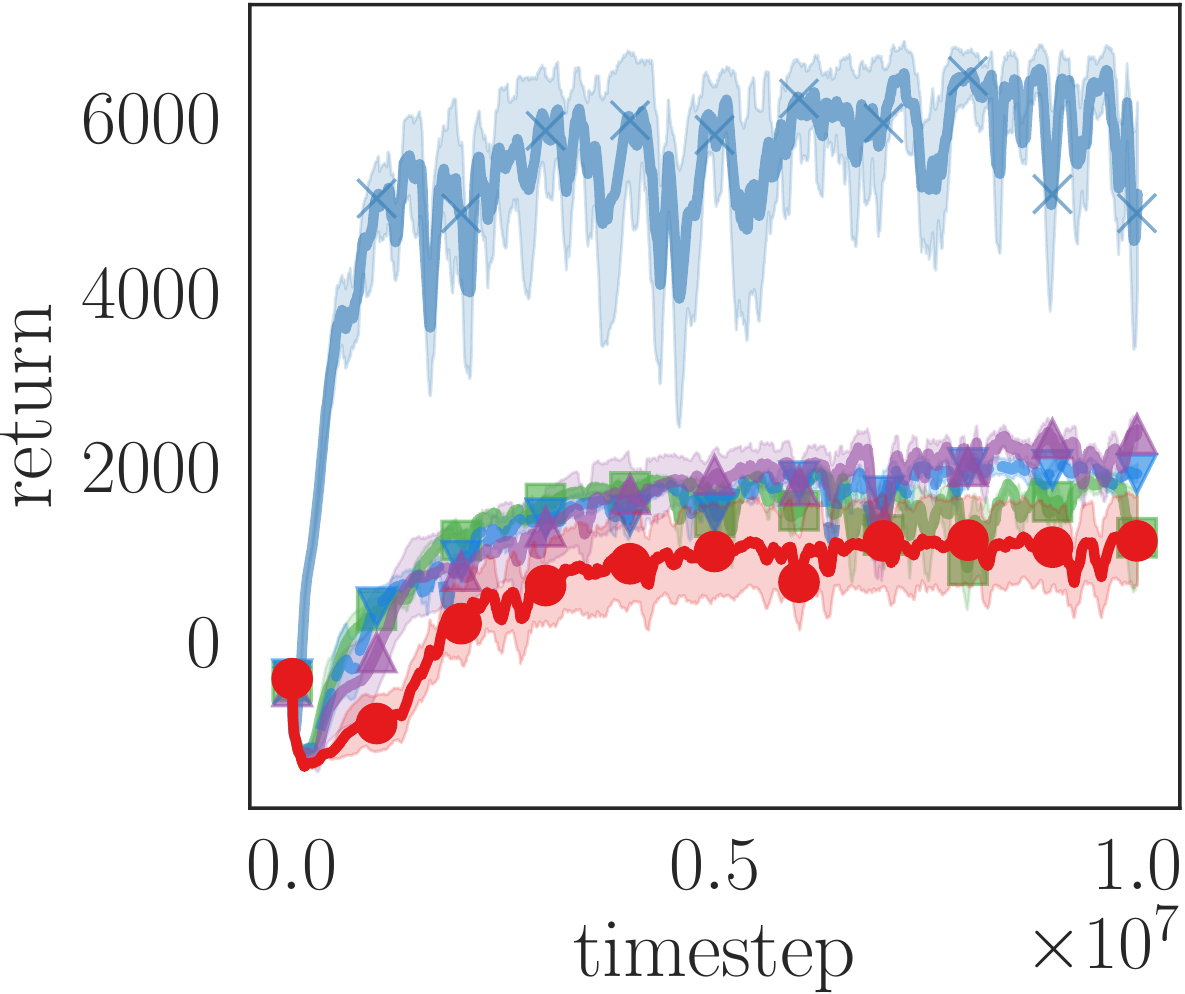} &
        \includegraphics[width=0.99\linewidth]{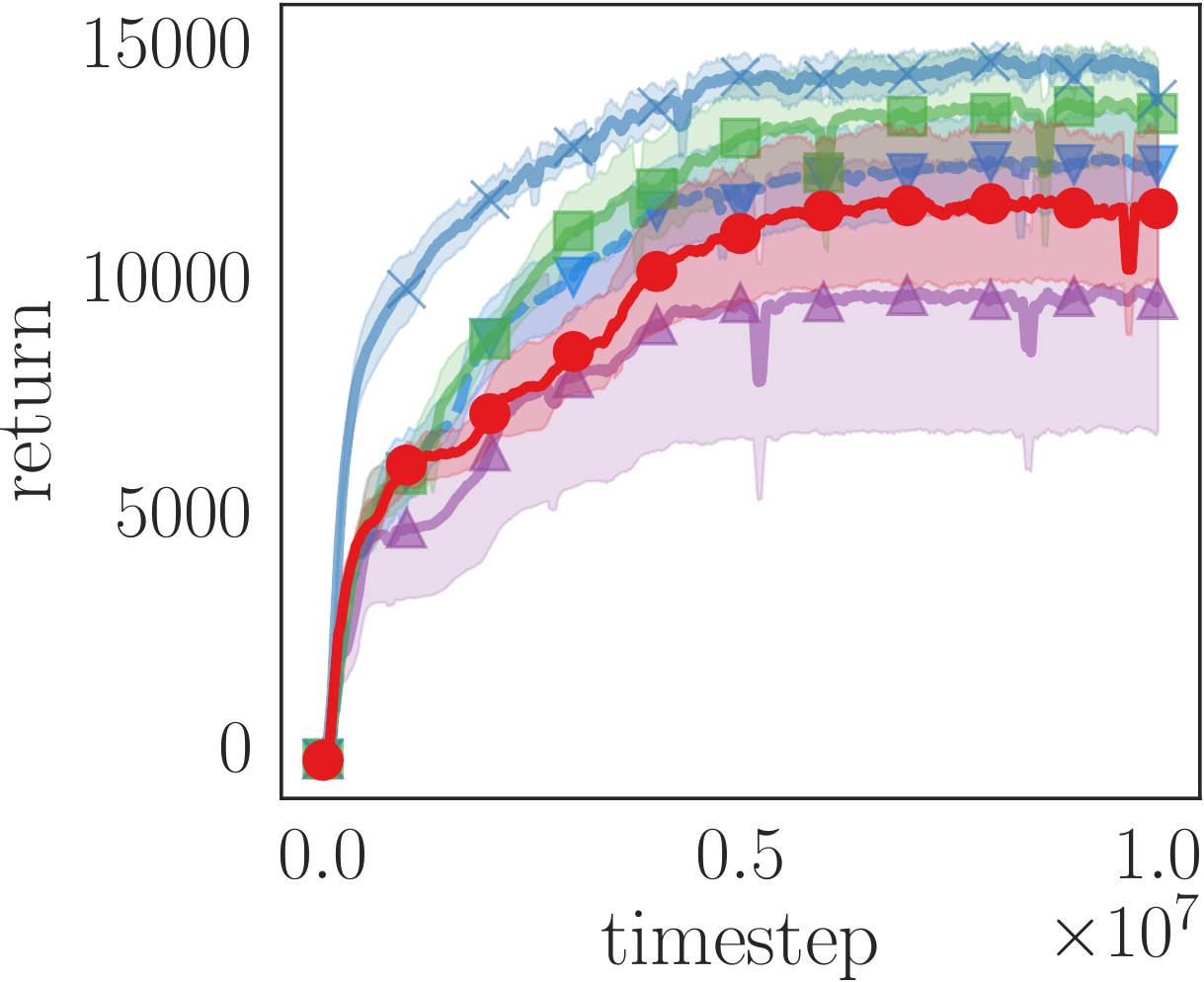} &
        \includegraphics[width=0.99\linewidth]{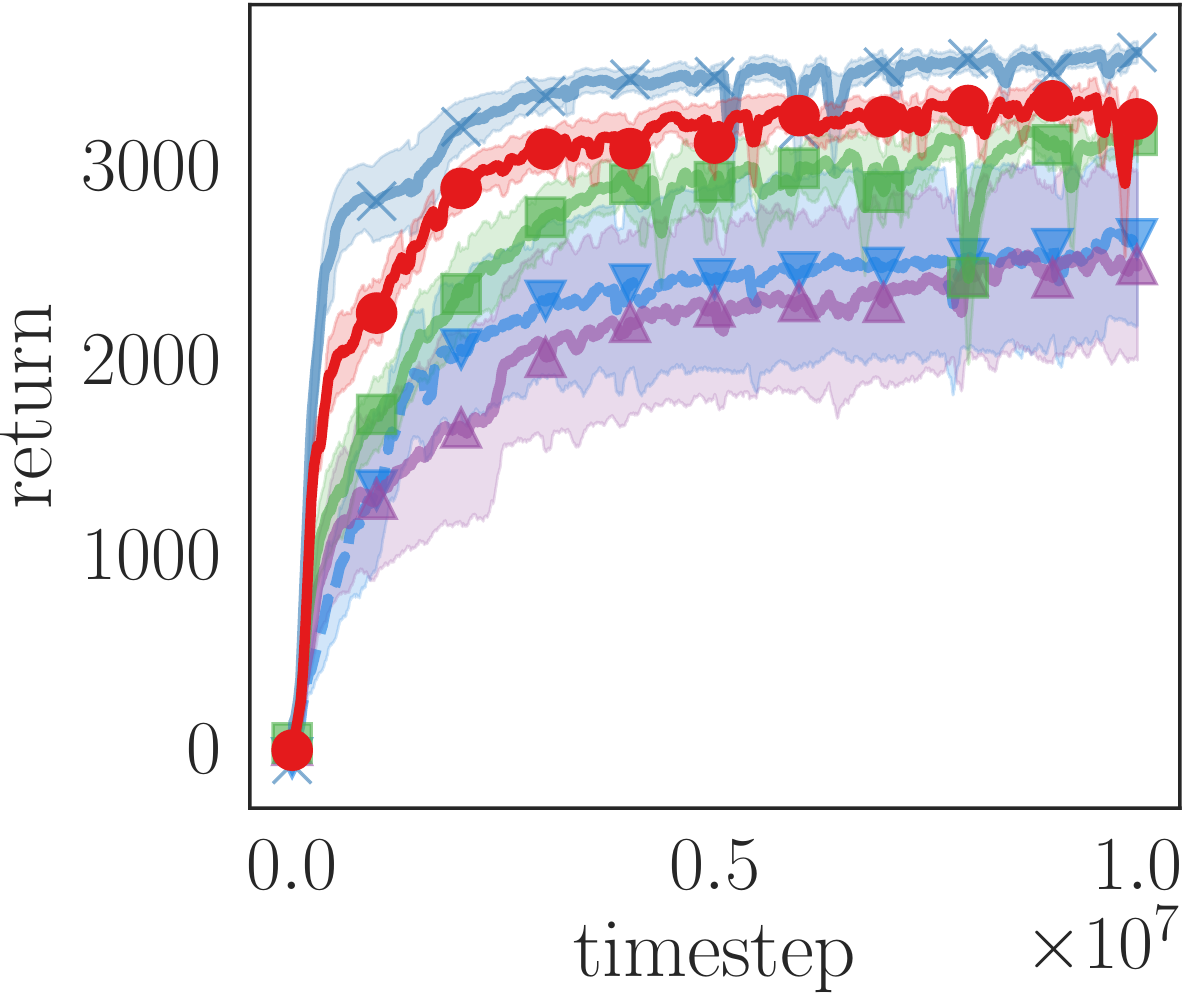} &
        \includegraphics[width=0.99\linewidth]{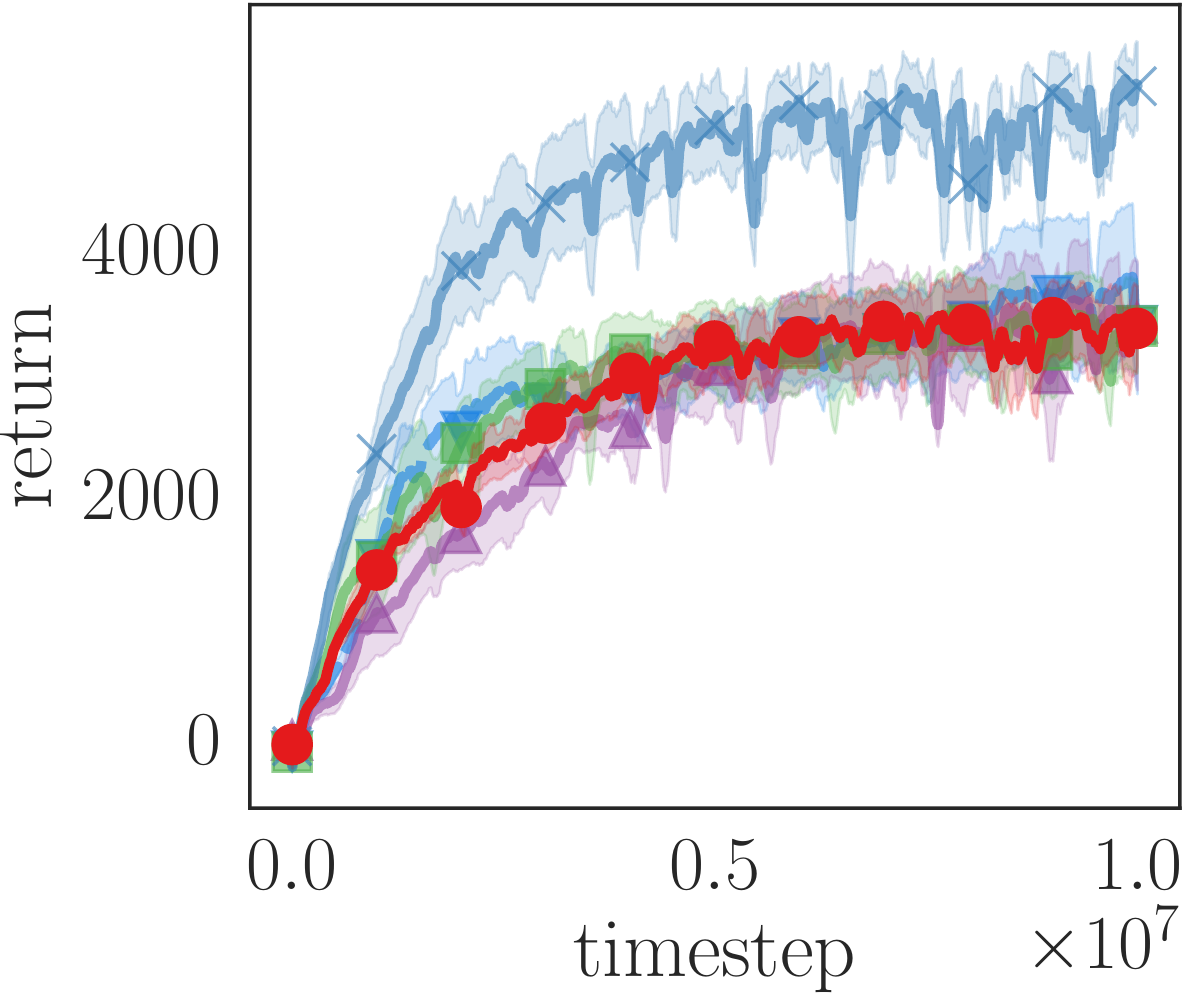} \\
        {\hspace*{3em}\small Swimmer-v3} & {\hspace*{3em}\small Reacher-v2} & {\hspace*{0em}\small InvertedDoublePendulum-v2} & {\hspace*{2em}\small InvertedPendulum-v2} \\
        \includegraphics[width=0.99\linewidth]{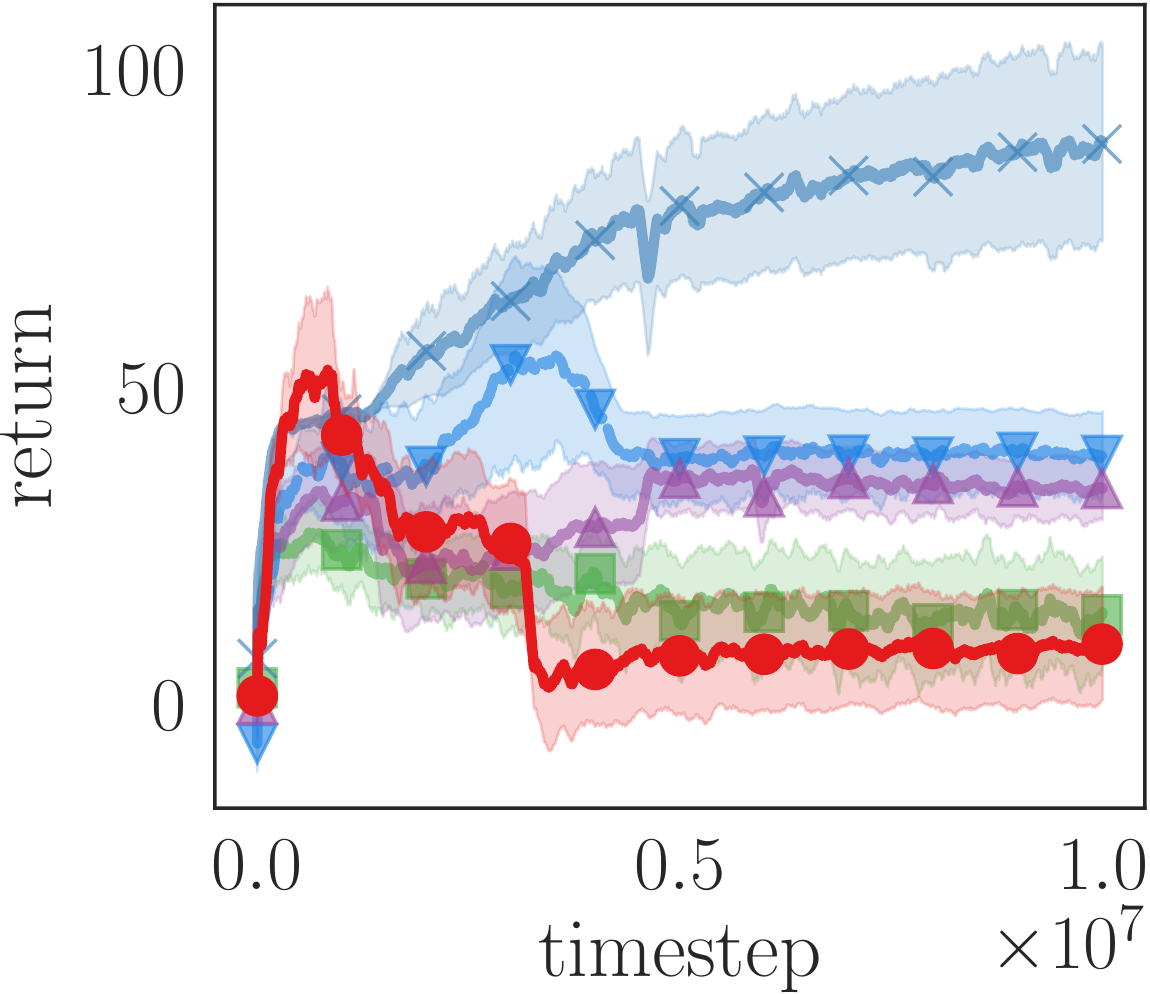} &
        \includegraphics[width=0.99\linewidth]{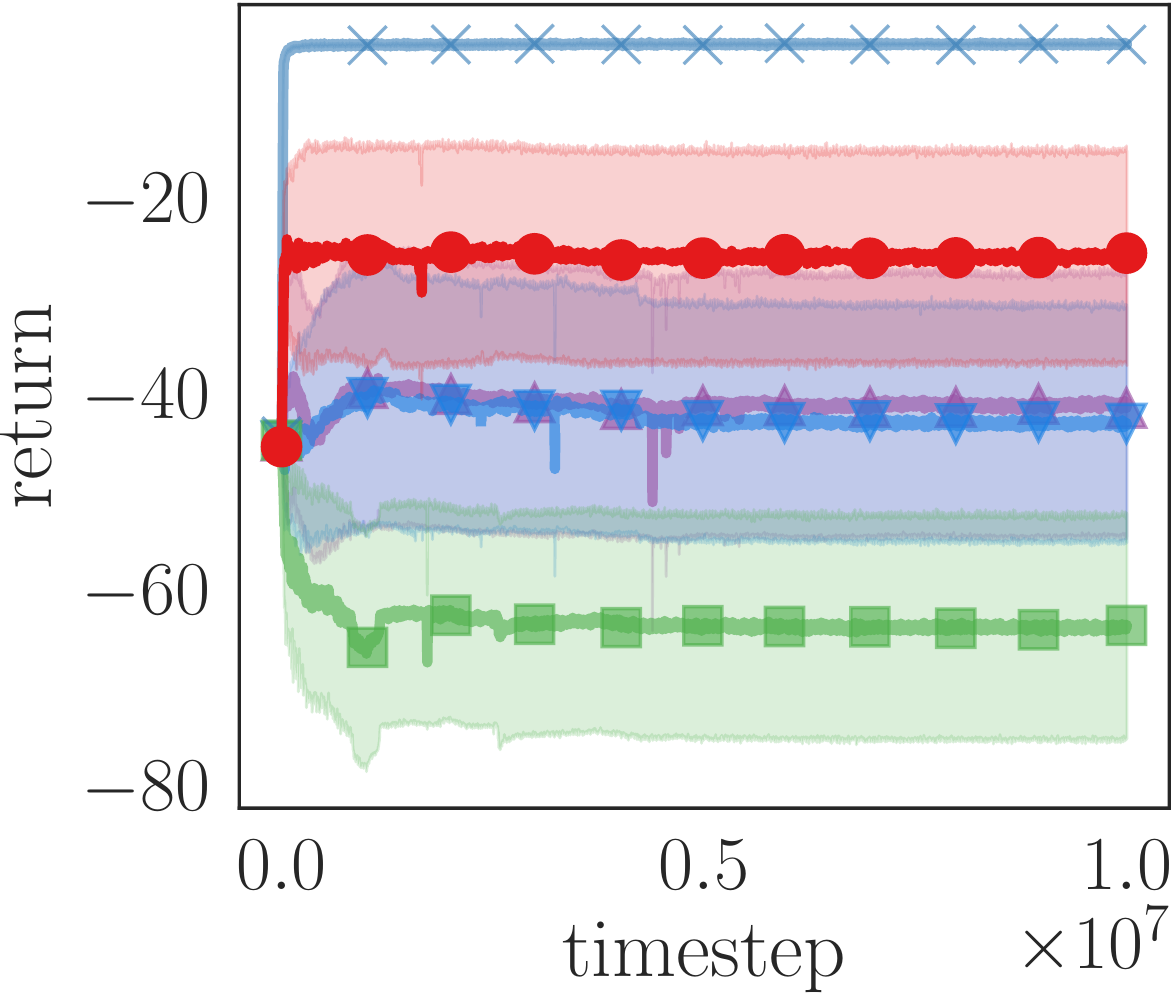} &
        \includegraphics[width=0.99\linewidth]{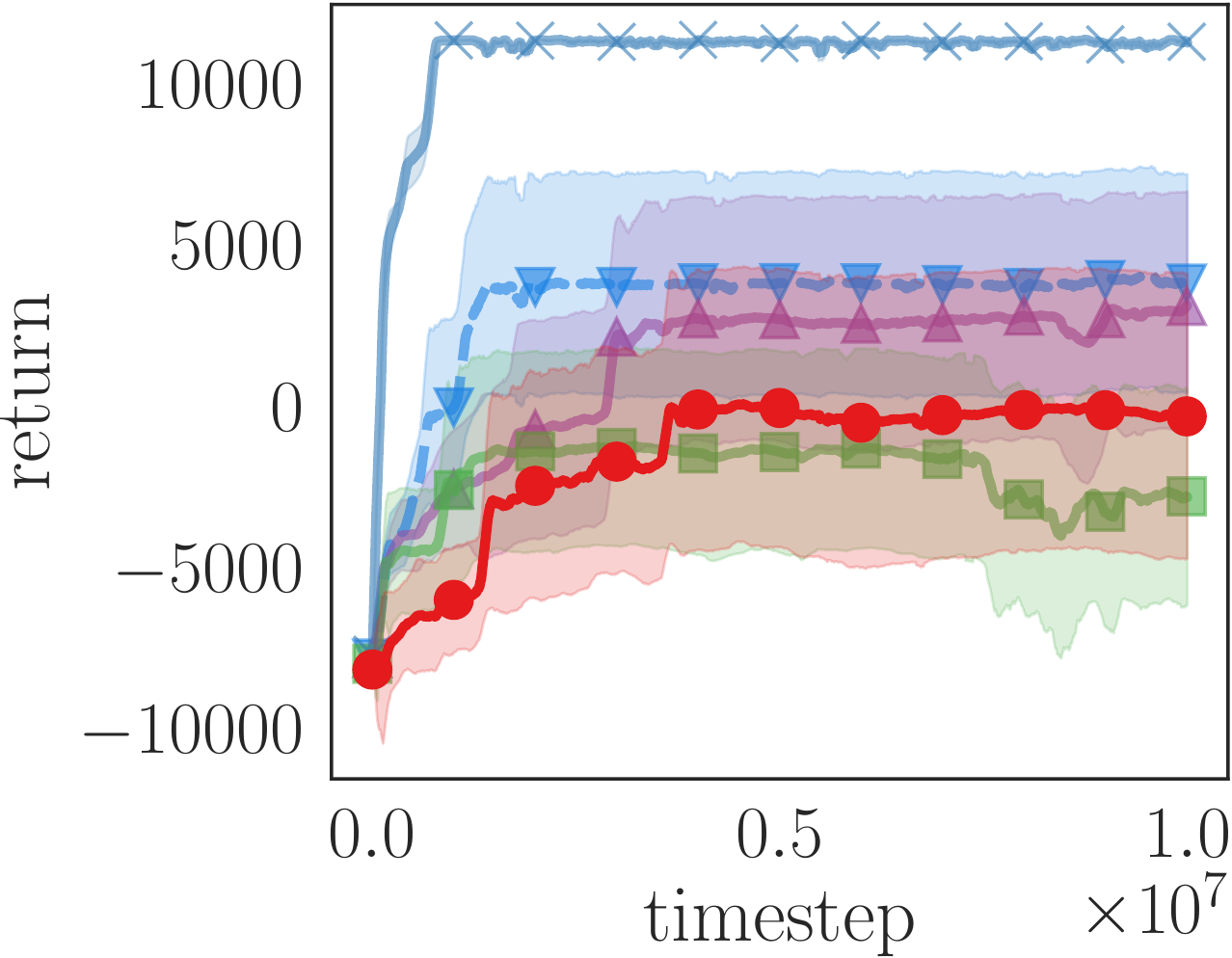} &
        \includegraphics[width=0.99\linewidth]{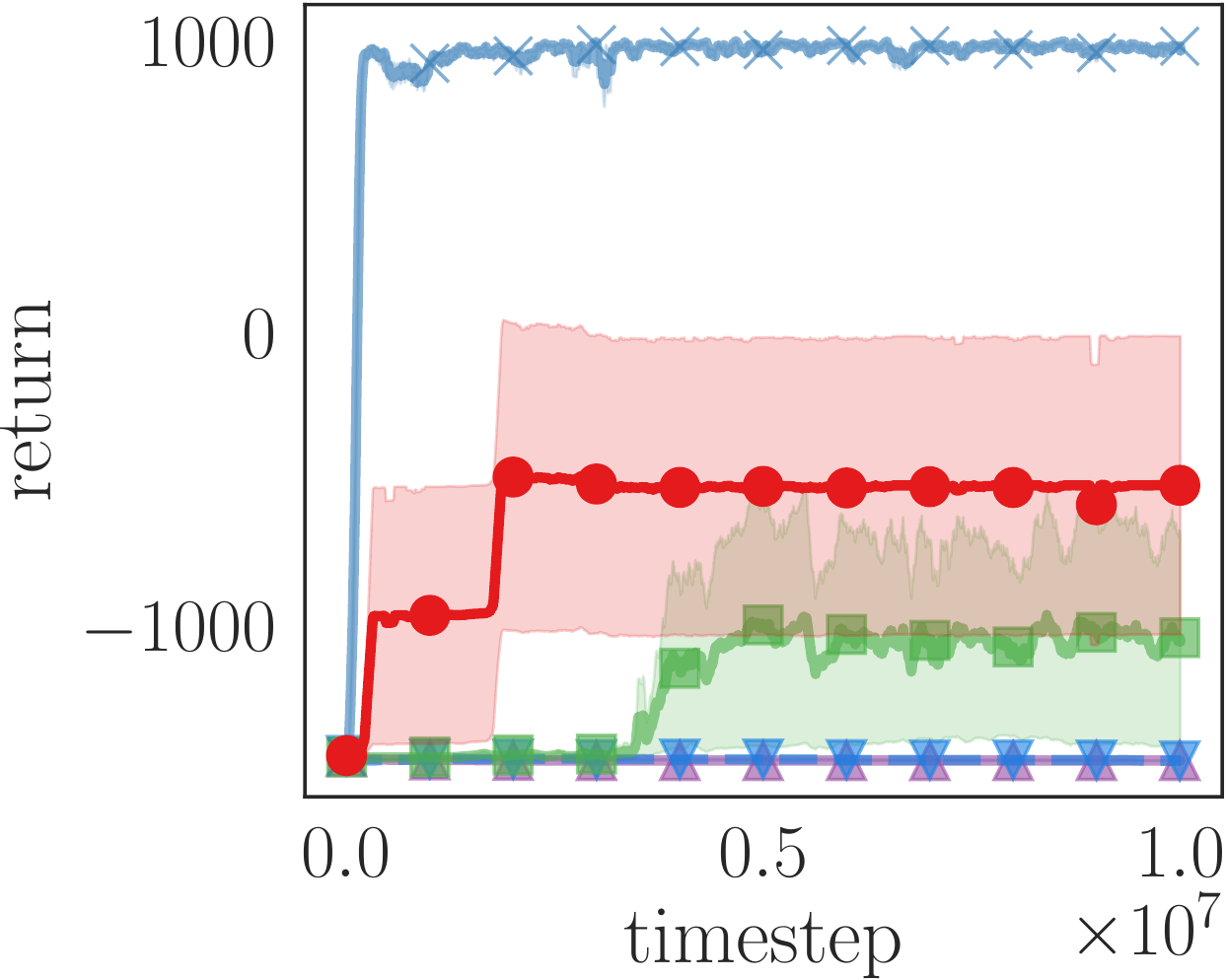}
    \end{tabular}
    \end{adjustbox}
}
{\small
    \hspace{0.1em}\textbf{Legend:} \vspace{1em} \\
    \begin{tabular}{llll}
    \legendUniform & Uniform Sampling &
    \legendIGGP & Information Gain on the Reward (IGR) \\
    \legendIDRL & \AlgoNameShort (ours) &
    \legendIDRLABLATION & \AlgoNameShort without candidate policy rollouts \\
    \legendEXPERT & SAC on true reward function & &
    \end{tabular}
}
\caption{
    Return of policies trained using a model trained from $1400$ comparison for each of the MuJoCo environments. \Cref{fig:results_deep_rl} in the main paper shows the normalized average over all environments.
}
\label{fig:additional_results_deep_rl}
\end{figure}

\end{document}